\documentclass[12pt]{colt2020} 
\usepackage[utf8]{inputenc}
\usepackage{times}

\usepackage{natbib}
\usepackage{graphicx}
\graphicspath{{./img/}}

\usepackage{amsmath, amssymb, amsfonts}
\usepackage{enumerate}

\usepackage{dsfont}
\usepackage{xargs}
\usepackage[textwidth=4cm,textsize=footnotesize]{todonotes}

\newtheorem{observation}{Observation}
\newtheorem{assum}{A\hspace{-2pt}}
\newtheorem{assumb}{B\hspace{-2pt}}

\def\Id{\operatorname{I}}
\def\rset{\mathbb{R}}
\def\nset{\mathbb{N}}

\newcommand{\vertiii}[1]{{\left\vert\kern-0.25ex\left\vert\kern-0.25ex\left\vert #1 
    \right\vert\kern-0.25ex\right\vert\kern-0.25ex\right\vert}}
\def\eqsp{\,}
\newcommandx{\norm}[2][2=]{\Vert#1 \Vert_{{#2}}}
\newcommandx{\normop}[2][2=]{\Vert{#1}\Vert_{{#2}}}

\def\ProdB{\Gamma}
\def\ProdBB{\widetilde \Gamma}

\newcommand{\beq}{\begin{equation}}
\newcommand{\eeq}{\end{equation}}
\newcommand{\beo}{\begin{array}{rl}}
\newcommand{\eeo}{\end{array}}

\DeclareMathOperator{\E}{\mathbb{E}}
\DeclareMathOperator{\PE}{\mathbb{E}}

\DeclareMathOperator{\Tr}{\operatorname{Tr}}
\def\F{{\mathcal F}}
\def\ttheta{{\tilde \theta}}

\def\tw{\tilde w}
\def\hw{\widehat{w}}

\def\mtw{\operatorname{M}^{\tw}}
\def\mth{\operatorname{M}^{\ttheta}}
\def\mthw{\operatorname{M}^{\ttheta, \tw}}

\def\topu{\tilde{\operatorname{U}}}
\def\opu{\operatorname{U}}

\def\mhr{\operatorname{M}^{\hw}}

\newcommandx{\CPE}[3][1=]{\ensuremath{\mathbb{E}^{#3}_{#1}\left[#2\right]}}
\newcommandx{\ps}[3][1=]{\ensuremath{\langle #2, #3 \rangle_{#1}}}
\def\mw{m_W}

\def\mv{m_V}

\def\ttmw{\tilde{m}_W}

\def\ttmv{\tilde{m}_V}
\def\ttmvt{\tilde{m}_{V}^{\theta}}

\def\ttmVW{\tilde{m}_{VW}}

\def\Linfty{\operatorname{L}_\infty}
\def\Cinfty{\operatorname{C}_\infty}
\def\CD{\operatorname{C}_D}
\newcommand{\EConst}[1]{{\rm C}_{#1}^{\sf{exp}}}
\def\EmVW{m_{VW}^{\sf{exp}}}
\def\EtmVW{\widetilde m_{VW}^{\sf{exp}}}

\def\Czero{\operatorname{C}_0}

\def\Czerow{\operatorname{C}^{\tw}_0}
\def\Conew{\operatorname{C}^{\tw}_1}
\def\Ctwow{\operatorname{C}^{\tw}_2}
\def\Czerot{\operatorname{C}^{\ttheta}_0}
\def\Conet{\operatorname{C}^{\ttheta}_1}
\def\Ctwot{\operatorname{C}^{\ttheta}_2}

\def\Czerotw{\operatorname{C}^{\ttheta,\tw}_0}
\def\Conetw{\operatorname{C}^{\ttheta,\tw}_1}
\def\Ctwotw{\operatorname{C}^{\ttheta,\tw}_2}

\newcommand{\CSEQ}[1]{\varrho^{#1}}

\def\Xset{\sf X}
\def\MK{\operatorname{P}}
\def\rmd{{\rm d}}

\def\dth{{d_{\theta}}}
\def\dw{{d_w}}
\def\pth{{p_\Delta}}
\def\pw{{p_{22}}}
\def\pwth{{p_{22,\Delta}}}

\def\markovterm{\psi}
\def\markovtermt{\widetilde{\psi}}
\def\markovonetb{\widetilde{\psi}^{b_1}}
\def\markovoneb{\psi^{b_1}}
\def\markovtwotb{\widetilde{\psi}^{b_2}}
\def\markovtwob{\psi^{b_2}}
\def\Markovterm{\Psi}
\def\markovoneA{\Psi^{A_{11}}}
\def\markovoneB{\Psi^{A_{12}}}
\def\markovtwoA{\Psi^{A_{21}}}
\def\markovtwoB{\Psi^{A_{22}}}
\def\martinterm{\xi}
\def\martinoneb{\xi^{b_1}}
\def\martintwob{\xi^{b_2}}
\def\Martinterm{\Xi}
\def\martinoneA{\Xi^{A_{11}}}
\def\martinoneB{\Xi^{A_{12}}}
\def\martintwoA{\Xi^{A_{21}}}
\def\martintwoB{\Xi^{A_{22}}}

\def\simplewtdiff{\Upsilon^{WV,\ttheta}}
\def\simplewtA{\widetilde{\Psi}^{WV,\ttheta}}
\def\simplewwdiff{\Upsilon^{WV,\tw}}
\def\simplewwA{\widetilde{\Psi}^{WV,\tw}}
\def\simplewtiter{\Phi^{WV,\ttheta}}
\def\simplewwiter{\Phi^{WV,\tw}}

\def\simplewtiter{\Phi^{WV,\ttheta}}
\def\simplewwiter{\Phi^{WV,\tw}}

\def\boundwsim{\operatorname{E}^{WV}_0}

\def\Czerowtilde{{\widetilde{\operatorname{C}}^{\tw}_0}}
\def\Conewtilde{{\widetilde{\operatorname{C}}^{\tw}_1}}
\def\Ctwowtilde{{\widetilde{\operatorname{C}}^{\tw}_2}}
\def\Cthreewtilde{{\widetilde{\operatorname{C}}^{\tw}_3}}

\def\Czerowtildep{{\widetilde{\operatorname{C}}^{\tw'}_0}}
\def\Conewtildep{{\widetilde{\operatorname{C}}^{\tw'}_1}}
\def\Ctwowtildep{{\widetilde{\operatorname{C}}^{\tw'}_2}}

\def\Czerowtildeh{{\widetilde{\operatorname{C}}^{\tw''}_0}}
\def\Conewtildeh{{\widetilde{\operatorname{C}}^{\tw''}_1}}
\def\Ctwowtildeh{{\widetilde{\operatorname{C}}^{\tw''}_2}}
\def\Cthreewtildeh{{\widetilde{\operatorname{C}}^{\tw''}_3}}

\def\Czerotwtilde{{\widetilde{\operatorname{C}}^{\ttheta,\tw}_0}}
\def\Conetwtilde{{\widetilde{\operatorname{C}}^{\ttheta,\tw}_1}}
\def\Ctwotwtilde{{\widetilde{\operatorname{C}}^{\ttheta,\tw}_2}}

\def\Czerottilde{{\widetilde{\operatorname{C}}^{\ttheta}_0}}
\def\Conettilde{{\widetilde{\operatorname{C}}^{\ttheta}_1}}
\def\Ctwottilde{{\widetilde{\operatorname{C}}^{\ttheta}_2}}

\def\boundtsim{\operatorname{E}^{V}_0}

\def\Boneinfty{{\operatorname{B}_{11,\infty}}}
\def\Btwoinfty{{\operatorname{B}_{22,\infty}}}

\def\ttmwzero{{\tilde{m}_W^{(0)}}}

\def\ttmvzero{{\tilde{m}_V^{(0)}}}

\def\ttmVWzero{{\tilde{m}_{VW}^{(0)}}}

\def\ttmwi{{\tilde{m}_W^{(i)}}}

\def\ttmvi{{\tilde{m}_V^{(i)}}}

\def\ttmVWi{{\tilde{m}_{VW}^{(i)}}}

\def\ttmtcom{{\tilde{m}_{\Delta \ttheta}}}
\def\ttmwcom{{\tilde{m}_{\Delta \tw}}}

\def\tC{\widetilde{\operatorname{C}}}
\def\hC{\widehat{\operatorname{C}}}

\title[Finite Time Analysis of Linear Two-timescale Stochastic Approximation]{Finite Time Analysis of Linear Two-timescale Stochastic Approximation with Markovian Noise}

\usepackage{autonum}

\coltauthor{\Name{Maxim Kaledin} \Email{mkaledin@hse.ru}\\
\addr National Research University Higher School of Economics, Moscow, Russia.\\
\Name{Eric Moulines} \Email{eric.moulines@polytechnique.edu}\\
\addr CMAP, \'{E}cole Polytechnique, Palaiseau, France.\\
\Name{Alexey Naumov} \Email{anaumov@hse.ru}\\
\addr National Research University Higher School of Economics, Moscow, Russia.\\
\Name{Vladislav Tadic} \Email{v.b.tadic@bristol.ac.uk}\\
\addr University of Bristol, Bristol, United Kingdom.\\
\Name{Hoi-To Wai} \Email{htwai@se.cuhk.edu.hk}\\
 \addr Department of SEEM, The Chinese University of Hong Kong, Hong Kong.\thanks{Authors listed in alphabetical order.}} 

\usepackage{verbatimbox}

\usepackage{enumitem}
\setlist{leftmargin=4mm}

\begin{document}

\maketitle

\begin{abstract}
    Linear two-timescale stochastic approximation (SA) scheme is an important class of algorithms which has become popular in reinforcement learning (RL), particularly for the policy evaluation problem. Recently, a number of works have been devoted to establishing the finite time analysis of the scheme, especially under the Markovian (non-i.i.d.) noise settings that are ubiquitous in practice.
    In this paper, we provide a finite-time  analysis for linear two timescale SA. Our bounds show that there is no discrepancy in the convergence rate between Markovian and martingale noise, only the constants are affected by the mixing time of the Markov chain. With an appropriate step size schedule, the transient term in the expected error bound is $o(1/k^c)$ and the steady-state term is ${\cal O}(1/k)$, where $c>1$ and $k$ is the iteration number. Furthermore, we present an asymptotic expansion of the expected error with a matching lower bound of $\Omega(1/k)$. A simple numerical experiment is presented to support our theory.
\end{abstract}

\begin{keywords}%
stochastic approximation, reinforcement learning, GTD learning, Markovian noise%
\end{keywords}


\section{Introduction}
Since its introduction close to 70 years ago, the stochastic approximation (SA) scheme \citep{robbins1951stochastic} has been a powerful tool for root finding when only noisy samples are available. 
During the past two decades, considerable progresses in the practical and theoretical research of SA have been made, see \citep{benaim:notes:1999,kushner2003stochastic,borkar:sa:2008} for an overview. 
Among others, linear SA schemes are popular in reinforcement learning (RL) as they lead to policy evaluation methods with linear function approximation, of particular importance is temporal difference (TD) learning \citep{sutton:td:1988} for which finite time analysis has been reported in \citep{srikant:1tsbounds:2019, lakshminarayanan2018linear, bhandari2018finite,dalal:td0:2017}. 

The TD learning scheme based on classical (linear) SA is known to be inadequate for the off-policy learning paradigms in RL, where data samples are drawn from a \emph{behavior policy} different from the  policy being evaluated  \citep{baird:resid:rl:funcappr:1995,tsitsiklis:td:1997}. 
To circumvent this problem, \citet{sutton:gtd:2009, sutton:gtd2:2009} have suggested to replace TD learning with the gradient TD (GTD) method or the TD with gradient correction (TDC) method. These methods fall within the scope of linear two-timescale SA scheme introduced by \citet{borkar1997stochastic}: 
\begin{align} \label{eq:tts-gen1}
\theta_{k+1} &= \theta_k + \beta_k \{ \widetilde{b}_1(X_{k+1}) - \widetilde{A}_{11}(X_{k+1}) \theta_k - \widetilde{A}_{12}(X_{k+1}) w_k \}, \\
w_{k+1} &= w_k + \gamma_k \{ \widetilde{b}_2(X_{k+1}) - \widetilde{A}_{21}(X_{k+1}) \theta_k - \widetilde{A}_{22}(X_{k+1})w_k \}. \label{eq:tts-gen2}
\end{align} 
The above recursion involves two iterates, $\theta_k \in \rset^\dth$, $w_k \in \rset^\dw$, whose updates are coupled with each other. 
In the above, $\widetilde{b}_i(x)$, $\widetilde{A}_{ij}(x)$ are measurable vector/matrix valued functions on $\Xset$ and the random sequence $(X_k)_{k \geq 0}, X_k \in \Xset$ forms an ergodic Markov chain. 
The scalars $\gamma_k, \beta_k > 0$ are step sizes. 
The above SA scheme is said to have two timescales as the step sizes satisfy $\lim_{k \rightarrow \infty} \beta_k / \gamma_k < 1$ such that $w_k$ is updated at a faster timescale. In fact, $w_k$ is a  `tracking' term which seeks solution to a linear system characterized by $\theta_k$.

The goal of this paper is to characterize the finite time expected error bound with improved convergence rate for the two timescale SA \eqref{eq:tts-gen1},\eqref{eq:tts-gen2}.
The almost sure convergence of two timescale SA have been established in \citep{borkar1997stochastic,tadic:asconvergence:tts:2004,tadic:asantdlearn:conststep:2006,borkar:sa:2008}, among others; the asymptotic convergence rates have been characterized in \citep{konda:tsitsiklis:2004, mokkadem2006convergence}. However, finite-time risk bounds for two timescale SA have not been analyzed until recently. With martingale samples, \citet{liu:prox:2015} provided the first finite time analysis of GTD method, \citet{dalal2018finite,dalal2019tale} provided improved finite time error bounds. Unlike our analysis, they analyzed modified two timescale SA with projection and their bounds hold with high probability. With Markovian noise, \citet{gupta2019finite} studied the finite time expected error bound with constant step sizes; \citet{xu2019two} and \citet{doan2019finite} provided similar analysis for general step sizes. 
It is important to notice that with homogeneous martingale noise, the asymptotic rate of \eqref{eq:tts-gen1}, \eqref{eq:tts-gen2} without a projection step, as shown in \citep[Theorem 2.6]{konda:tsitsiklis:2004}, is in the order  $\E[\| \theta_k - \theta^\star \|^2] = {\cal O}(\beta_k), \E[\| w_k - A_{22}^{-1} (b_2 - A_{21} \theta_k) \|^2] = {\cal O}(\gamma_k)$, where $\theta^\star$ is a stationary point of the SA scheme. However, the latter rate is not achieved in the finite-time error bounds analyzed by the above works except for \citep{dalal2019tale}. It remains an open problem whether this error bound holds for the Markovian noise setting and for linear two time-scale SA scheme without projection.

\paragraph{Contributions} This paper has the following contributions:
\begin{itemize}[itemsep=1pt]
    \item \emph{Improved Convergence Rate} -- We perform finite-time expected error bound analysis of the linear two timescale SA in both martingale and Markovian noise settings, in  Theorems~\ref{theo:preliminary-bound-martingale} \& \ref{theo:preliminary-bound-markov}. Our analysis allow for general step sizes schedules [cf.~A\ref{assum:stepsize}, B\ref{assb:step}], including constant, piecewise constant, and diminishing step sizes explored in the prior works \citep{gupta2019finite,dalal2019tale,xu2019two,doan2019finite}. We show that the error bound consists of a transient and a steady-state term, and the asymptotic rate is obtained from the latter. We show that this asymptotic rate matches those in \citep[Theorem 2.6]{konda:tsitsiklis:2004}, i.e., $\E[ \| \theta_k - \theta^\star \|^2 ] = {\cal O}(\beta_k), \E[\| w_k - A_{22}^{-1} (b_2 - A_{21} \theta_k) \|^2] = {\cal O}(\gamma_k)$. In particular, the fastest achievable rate for $\E[ \| \theta_k - \theta^\star \|^2 ]$ will be ${\cal O}(1/k)$ when we set $\beta_k = {\cal O}(1/k), \gamma_k = {\cal O}(1/k^\upsilon)$ with $\upsilon < 1$. 
    \item \emph{Novel Analysis without A-prori Stability Assumption} -- Unlike the prior works  \citep{liu:prox:2015,dalal2019tale,xu2019two}, our convergence results are obtained \emph{without} requiring a projection step throughout the SA iterations. In fact, \citet{dalal2019tale} have pointed out that the projection step is merely included to ensure \emph{a-priori} stability of the algorithm, and is often not used in practice. Our relaxation and the ability to achieve the optimal convergence rate are obtained through a tight analysis of the recursive inequalities of the (cross-)variances of $\theta_k$, $w_k$, see   Section~\ref{sec:upperbd}. 
    \item \emph{Asymptotic Expansion} -- With an additional assumption on the step size, we compute an exact asymptotic expansion of the expected error $\E[ \| \theta_k - \theta^\star \|^2 ]$, see Theorem \ref{th: expansion}. With an appropriate diminishing step sizes schedule, we show that the expected error cannot be smaller than $\Omega(\beta_k)$, which matches our upper bound results in Theorem~\ref{theo:preliminary-bound-martingale} \& \ref{theo:preliminary-bound-markov}. 
\end{itemize}
The rest of this paper is organized as follows. In Section~\ref{sec:sa}, we present the detailed conditions for two timescale linear SA, and the main results on finite-time performance bounds. In Section~\ref{sec:upperbd}, we provide an outline of the proof, illustrating the insights behind the main steps. In Section~\ref{sec:expansion}, we show that the finite-time error bounds are tight by quantifying an exact expansion of the covariance of iterates. In Section~\ref{sec:num}, we illustrate the theoretical findings using numerical experiments. 

\paragraph{Notations}
Let $n \in \nset$ and $Q$ be a symmetric definite $n \times n$ matrix.  For $x \in \rset^n$, we denote $\norm{x}[Q]= \{x^\top Q x\}^{1/2}$. For brevity, we set $\norm{x}= \norm{x}[\Id]$. Let $m \in \nset$, $P$ be a symmetric definite $m \times m$ matrix, $A$ be an $n \times m$ matrix. 
A matrix $A$ is said to be Hurwitz if the real parts of its eigenvalues are strictly negative. We denote $\normop{A}[P,Q]= \max_{\norm{x}[P]=1} \norm{Ax}[Q]$. If $A$  is a $n \times n$ matrix, we denote  $\normop{A}[Q]=\normop{A}[Q,Q]$.  Lastly,
we give a number of auxiliary lemmas in Appendix~\ref{app:aux} that are instrumental to our analysis.


\section{Linear Two Time-scale Stochastic Approximation (SA) Scheme} \label{sec:sa}
We investigate the linear two timescale SA given by the following equivalent form of \eqref{eq:tts-gen1}, \eqref{eq:tts-gen2}:
\begin{align}
    &\theta_{k+1} = \theta_k + \beta_k (b_1 - A_{11}\theta_k - A_{12}w_k + V_{k+1}), \label{eq:2ts1} \\
    &w_{k+1} = w_k + \gamma_k (b_2 - A_{21}\theta_k - A_{22} w_k + W_{k+1}), \label{eq:2ts2}
\end{align}
where the mean fields are defined as
$b_i := \lim_{k \rightarrow \infty} \E[ \widetilde{b}_i( X_k ) ]$, $A_{ij} := \lim_{k \rightarrow \infty} \E[\widetilde{A}_{ij}(X_k)]$ (these limits exist as we recall that $(X_k)_{k \geq 0}$ is an ergodic Markov chain). The noise terms $V_{k+1}, W_{k+1}$ are given by:
\beq \label{eq:noise_term}
\begin{split}
V_{k+1} & := \widetilde{b}_1(X_{k+1}) - b_1 - (\widetilde{A}_{11}(X_{k+1}) - A_{11}) \theta_k - (\widetilde{A}_{12}(X_{k+1}) - A_{12}) w_k , \\
W_{k+1} & := \widetilde{b}_2(X_{k+1}) - b_2 - (\widetilde{A}_{21}(X_{k+1}) - A_{21}) \theta_k - (\widetilde{A}_{22}(X_{k+1}) - A_{22}) w_k .
\end{split}
\eeq 
The goal of the recursion \eqref{eq:2ts1}, \eqref{eq:2ts2} is to find a stationary solution pair $(\theta^\star, w^\star)$ that solves the system of linear equations:
\begin{align} \label{eq:linear_sys}
&A_{11} \theta + A_{12} w = b_1, \quad A_{21} \theta + A_{22} w = b_2 .
\end{align}
We are interested in the scenario when the solution pair $(\theta^\star, w^\star)$ is unique and is given by
\beq \label{eq:opt_sol}
\theta^\star = \Delta^{-1} (b_1 - A_{12} A_{22}^{-1} b_2),\quad w^\star = A_{22}^{-1} (b_2 - A_{21} \theta^\star).
\eeq
where $\Delta := A_{11} - A_{12} A_{22}^{-1} A_{21}$.
To analyze the convergence of $(\theta_k, w_k)_{k \geq 0}$ in  \eqref{eq:2ts1}, \eqref{eq:2ts2} to $(\theta^\star, w^\star)$, we require the following assumptions:
\begin{assum}
\label{assum:hurwitz}
Matrices $-A_{22}$ and $ - \Delta =-\left(A_{11}-A_{12}A_{22}^{-1}A_{21}\right)$ are \textit{Hurwitz}. 
\end{assum}
The above assumption is common for linear two time-scale SA, see \citep{konda:tsitsiklis:2004}. 
As a consequence, using the Lyapunov lemma (stated in Lemma~\ref{lem:lyapunov} in the appendix for completeness), there exist positive definite matrices $Q_{22}^\top = Q_{22} \succ 0, Q_\Delta^\top = Q_\Delta \succ 0$ satisfying 
\begin{equation}
\label{eq:definition-Q-22}
A_{22}^\top Q_{22} + Q_{22} A_{22}= \Id, \quad
Q_\Delta \Delta + \Delta^\top Q_\Delta = \Id.
\end{equation}
This ensures the contraction (see Lemma~\ref{lem:stability} in the appendix): 
\beq \label{eq:contraction_p}
\normop{\Id - \gamma_k A_{22}}[Q_{22}] \leq 1 - a_{22} \gamma_k, \quad
\normop{\Id - \beta_k \Delta}[Q_\Delta] \leq 1 - a_\Delta \beta_k,
\eeq
where $a_{22} := 1 / (2 \| Q_{22}\|^2)$, $a_\Delta := 1 / (2 \| Q_\Delta \|^2)$. 
We consider the following conditions on the step sizes: 
\begin{assum}
\label{assum:stepsize}
$(\gamma_k)_{k \geq 0}$, $(\beta_k)_{k \geq 0}$ are nonincreasing sequences of positive numbers that satisfy the following.\vspace{-.1cm}
\begin{enumerate}[itemsep=1pt]
\item 
There exist constants $\kappa$ such that for all $k \in \nset$, we have $\beta_k / \gamma_k \leq \kappa$. 
\item 
For all $k \in \nset$, it holds \vspace{-.1cm}
\beq
\gamma_k / \gamma_{k+1} \leq 1 + (a_{22}/ 8) \gamma_{k+1},~~\beta_k / \beta_{k+1} \leq 1 + (a_\Delta/16) \beta_{k+1},~~\gamma_k / \gamma_{k+1} \leq 1 + (a_\Delta/16) \beta_{k+1}.\vspace{-.1cm}
\eeq
\end{enumerate}
\end{assum} 
As a consequence, we can define $\varsigma := 1 + \{ \gamma_0 a_{22}/8 \vee \beta_0 a_\Delta / 16\}$ such that $\gamma_{k}/\gamma_{k+1}\le \varsigma$, $\beta_{k}/\beta_{k+1}\le \varsigma$.
Our conditions on step sizes are similar to \citep[Assumption 2.3, 2.5]{konda:tsitsiklis:2004}. 
These conditions encompass diminishing, piecewise constant and constant step sizes schedules which are common in the literature. For instance, a popular choice of diminishing step sizes satisfying A\ref{assum:stepsize} is  
\beq \label{eq:stepsize_choice}
\beta_k = c^\beta / (k+ k_0^\beta), \quad \gamma_k = c^\gamma / (k+k_0^\gamma)^{2/3}
\eeq
with some constants $c^\beta$, $c^\gamma$, $k_0^\gamma, k_0^\beta$, e.g., as suggested in \citep[Remark 9]{dalal2018finite}; or a constant step size of $\beta_k = \beta, \gamma_k = \gamma$; or a piecewise constant step size, e.g., \citep{gupta2019finite}. 

We present new results on the convergence rate of \eqref{eq:2ts1}, \eqref{eq:2ts2} depending on the types of noise with $V_{k+1}, W_{k+1}$. To discuss these cases, let us define the $\sigma$-field generated by the two timescale SA scheme and the initial error made by the SA scheme, respectively as:
\begin{equation}
\label{eq:definition-filtration}
\F_k := \sigma \big\{ \theta_0, w_0, X_1, X_2,..., X_k \big\} ,\quad
{\rm V}_0 := \E [ \| \theta^0 - \theta^\star \|^2 + \| w^0 - w^\star \|^2 ].
\end{equation}
Our main results are presented as follows.

\paragraph{Martingale Noise} We consider a simple setting where the random elements $X_k$ are drawn i.i.d.~from the stationary distribution such that $b_i, A_{ij}$ are the expected values of $\widetilde{b}_i(X_k), \widetilde{A}_{ij}(X_k)$. Furthermore, the random variables $\widetilde{b}_i(X_k), \widetilde{A}_{ij}(X_k)$ have bounded second order moment. 
Note that this implies $\CPE{ V_{k+1} }{\F_k} = \CPE{W_{k+1}}{\F_k} = 0$, i.e., the sequences $(V_{k+1})_{k \in \nset}, (W_{k+1})_{k \in \nset}$ are martingale difference sequences.
Formally, we describe this setting as the following conditions on $V_{k+1}, W_{k+1}$:
\begin{assum}
 \label{assum:zero-mean} 
The noise terms are zero-mean conditioned on $\F_k$, i.e., $\CPE{ V_{k+1} }{\F_k} = \CPE{W_{k+1}}{\F_k} = 0$.
\vspace{-.2cm}
 \end{assum}
\begin{assum}
 \label{assum:bound-conditional-variance} 
 There exist constants $\mw, \mv$ such that 
    \begin{align}
        &\normop{\E[ V_{k+1} V_{k+1}^\top]} \le \mv ( 1 +  \normop{\E[\theta_{k} \theta_k^\top]} + \normop{ \E[w_{k} w_k^\top]} ), \label{eq:noise:constants:plain} \\
        &\normop{\E[ W_{k+1} W_{k+1}^\top]} \le \mw (1 + \normop{\E[\theta_{k} \theta_k^\top]} + \normop{ \E[w_{k} w_k^\top]} ) \notag \eqsp.
    \end{align}
\end{assum}
\begin{theorem}
\label{theo:preliminary-bound-martingale}
Assume A\ref{assum:hurwitz}--\ref{assum:bound-conditional-variance} and for all $k \in \nset$, we have $\gamma_k \in [0, \gamma_\infty^{\sf mtg}]$, $\beta_k \in [0, \beta_\infty^{\sf mtg}]$ and $\kappa \in [0,\kappa_\infty]$, where $\gamma_\infty^{\sf mtg}, \beta_\infty^{\sf mtg}, \kappa_\infty$ are constants defined in \eqref{eq:gamma_martingale}, \eqref{eq:condition-kappa}.
Then
\begin{align}
\label{eq:convergence-slow}
&\E[\norm{\theta_k - \theta^*}^2] \le \dth \Bigg\{ {\rm C}_0^{\ttheta,\sf mtg} \prod_{\ell=0}^{k-1} \Big(1-\beta_\ell \frac{a_\Delta}{4} \Big) {\rm V}_0  + {\rm C}_1^{\ttheta,\sf mtg}  \beta_{k} \Bigg\} \, \\
\label{eq:tracking-fast-component}
& \E[\norm{w_k - A_{22}^{-1} (b_2 - A_{21} \theta_k)}^2]  \leq \dw \Big\{ {\rm C}_0^{\hw,\sf mtg} \prod_{\ell=0}^{k-1} \Big(1-\beta_\ell \frac{a_\Delta}{4} \Big) {\rm V}_0  + {\rm C}_1^{\hw,\sf mtg} \gamma_k \Big\}
\end{align}
The exact constants are provided in the appendix, see \eqref{eq:martingale_finconst}, \eqref{eq:mhw_const}.
\end{theorem}

\paragraph{Markovian Noise}
Consider the sequence $(X_k)_{k \geq 0}$ to be samples from an exogenous Markov chain on $\Xset$ with the transition kernel $\MK: \Xset \times \Xset \rightarrow \rset_+$. For any measurable function $f$, we have \[
\CPE{f(X_{k+1})}{\F_k} = \MK f(X_{k}) = \int_{\Xset} f(x) \MK( X_{k}, \rmd x )
\]
We state the following assumptions:
\begin{assumb} \label{assb:mc1}
The Markov kernel $\MK$ has a unique invariant distribution $\mu: \Xset \rightarrow \rset_+$. Moreover, it is irreducible and aperiodic. 
\end{assumb}
Observe that
\[
b_i = \int_{\Xset} \widetilde{b}_i(x) \!~\mu(\rmd x),\quad A_{ij} = \int_{\Xset} \widetilde{A}_{ij}(x) \!~\mu(\rmd x),~i,j=1,2.
\]
We show that the linear two time-scale SA \eqref{eq:tts-gen1}, \eqref{eq:tts-gen2} converges to a unique fixed point defined by the above mean field vectors/matrices, see \eqref{eq:opt_sol}. 
An important condition that enables our analysis is the existence of a solution to the following Poisson equation:
\begin{assumb}\label{assb:poisson}
For any $i,j=1,2$, consider $\widetilde{b}_i(x), \widetilde{A}_{ij}(x)$, there exists vector/matrix valued measurable functions $\widehat{b}_i(x), \widehat{A}_{ij}(x)$ which satisfy
\beq
\widetilde{b}_i(x) - b_i = \widehat{b}_i(x) - \MK \widehat{b}_i(x),~~\widetilde{A}_{ij}(x) - A_{ij} = \widehat{A}_{ij}(x) - \MK \widehat{A}_{ij}(x)
\eeq
for any $x \in \Xset$ and $b_i, A_{ij}$ are the mean fields of $\widetilde{b}_i(x), \widetilde{A}_{ij}(x)$ with the stationary distribution $\mu$. 
\end{assumb}
The above assumption can be guaranteed under B\ref{assb:mc1} together with some regularity conditions, 
see \citep[Section 21.2]{douc2018markov}.
Moreover, 
\begin{assumb}\label{assb:bdd}
Under B\ref{assb:poisson}, the vector/matrix valued functions $\widehat{b}_i(x), \widehat{A}_{ij}(x)$ are uniformly bounded: for any $i,j=1,2$, $x \in \Xset$,
\beq
\| \widehat{b}_i(x)\| \leq \overline{\rm b},~\|\widehat{A}_{ij}(x) \|\leq \overline{\rm A}.
\eeq
\end{assumb}
\begin{assumb} \label{assb:step}
There exists constant $\rho_0$ such that for any $k \geq 1$, we have $\gamma_{k-1}^2 \leq \rho_0 \beta_k$.
\end{assumb}
To satisfy B\ref{assb:bdd}, we observe that the bounds $\overline{\rm b}, \overline{\rm A}$ depend on the mixing time of the chain $(X_k)_{k \geq 0}$ and a uniform bound on $\widetilde{b}_i(\cdot), \widetilde{A}_{ij}(\cdot)$. In the context of reinforcement learning, the latter can be satisfied when the feature vectors and reward are bounded. Note that B\ref{assb:bdd} implies A\ref{assum:bound-conditional-variance}, see Section~\ref{sec:bd_markov}. Meanwhile, B\ref{assb:step} imposes further restriction on the step size. The latter can also be satisfied by \eqref{eq:stepsize_choice}.  

The challenges of analysis with Markovian noise lie in the biasedness of the noise term as $\CPE{V_{k+1}}{\F_k} \neq 0$, $\CPE{W_{k+1}}{\F_k} \neq 0$. With a careful analysis, we obtain:
\begin{theorem}
\label{theo:preliminary-bound-markov}
Assume A\ref{assum:hurwitz}--\ref{assum:stepsize}, B\ref{assb:mc1}--\ref{assb:step} hold and for all $k \in \nset$, we have $\beta_k \in (0,\beta_\infty^{\sf mark}]$, $\gamma_k \in (0, \gamma_\infty^{\sf mark}]$, $\kappa \leq \kappa_\infty$, where $\beta_\infty^{\sf mark}$, $\gamma_\infty^{\sf mark}$, $\kappa_\infty$ are defined in \eqref{eq:gamma_markov}, \eqref{eq:condition-kappa}. 
Then 
\begin{align}
\label{eq:convergence-slow-markov}
&\E[\| \theta_k - \theta^\star \|^2] \le \dth \Bigg\{ {\rm C}_0^{\ttheta,\sf mark} \prod_{\ell=0}^{k-1} \Big(1- \beta_\ell \frac{a_\Delta}{8} \Big) (1 + {\rm V}_0)  + {\rm C}_1^{\ttheta,\sf mark} \, \beta_{k} \Bigg\}, \\
\label{eq:tracking-fast-component-markov}
& \E[\| w_k - A_{22}^{-1} (b_2 - A_{21} \theta_k)\|^2]  \leq \dw \Bigg\{ {\rm C}_0^{\hw,\sf mark} \prod_{\ell=0}^{k-1} \Big(1-\beta_\ell \frac{a_\Delta}{8} \Big) (1 + {\rm V}_0) + {\rm C}_1^{\hw,\sf mark} \gamma_k \Bigg\}.
\end{align}
The exact constants are given in the appendix, see \eqref{eq:ctheta_markov}, \eqref{eq:mthwbd_markov_const}.
\end{theorem}
While Theorem \ref{theo:preliminary-bound-markov} relaxes the martingale difference assumption A\ref{assum:bound-conditional-variance} in Theorem \ref{theo:preliminary-bound-martingale}, we remark that the results here do not generalize that in Theorem \ref{theo:preliminary-bound-martingale} due to the additional B\ref{assb:bdd}, B\ref{assb:step}. Particularly, with martingale noise, the convergence of linear two timescale SA only requires the noise to have bounded \emph{second order moment}, yet the Markovian noise needs to be uniformly bounded.

\paragraph{Convergence Rate of Linear Two Timescale SA}
The upper bounds in
Theorem~\ref{theo:preliminary-bound-martingale} and \ref{theo:preliminary-bound-markov} consist of two terms -- the first term is a `transient' error with product such as $\prod_{i=0}^{k-1} (1- \beta_i a_\Delta/8)$ decays to zero  at the rate $o(1/k^c)$ for some $c>1$ under an appropriate choice of step sizes such as \eqref{eq:stepsize_choice}; the second term is a `steady-state' error. 
We observe that the `steady-state' error of the iterates $\theta_k, w_k$ exhibit different behaviors. Taking the step size choices in \eqref{eq:stepsize_choice} as an example, the steady-state error of the slow-update iterates $\theta_k$ is ${\cal O}(1/k)$ while the error of fast-update iterates $w_k$ is ${\cal O}(1/k^{\frac{2}{3}})$. Furthermore, similar bounds hold for \emph{both} martingale and Markovian noise. 
{In Section~\ref{sec:expansion} we show that the obtained rates are also tight.}

\paragraph{Comparison to Related Works}
Our results improve the convergence rate analysis of linear two timescale  SA in a number of recent works. In the martingale noise setting (Theorem~\ref{theo:preliminary-bound-martingale}), the closest work to ours is \citep{dalal2019tale} which analyzed the linear two timescale  SA with martingale samples and diminishing step sizes. The  authors improved on \citep{dalal2018finite} and obtained the same convergence rate (in high probability) as our Theorem~\ref{theo:preliminary-bound-martingale}, furthermore it is demonstrated that the obtained rates are tight. Their bounds also exhibit a sublinear dependence on the dimensions $\dth, \dw$. However, their algorithm involves a sparsely executed projection step and the error bound holds only for a  sufficiently large $k$. These restrictions are lifted in our analysis. 

In the Markovian noise setting (Theorem~\ref{theo:preliminary-bound-markov}), the closest works to ours are \citep{doan2019finite, gupta2019finite, xu2019two}. In particular, \citet{gupta2019finite} analyzed the linear two timescale  SA with constant step sizes and showed that the steady-state error for both $\theta_k, w_k$ is ${\cal O}(\gamma^2/ \beta)$. \citet{xu2019two} analyzed the TDC algorithm with a projection step and showed that the steady-state error for $\theta_k$ is ${\cal O}(1/k^{\frac{2}{3}})$ if the step sizes in \eqref{eq:stepsize_choice} is used. \citet{doan2019finite} analyzed the linear two timescale  SA with diminishing step size and showed that the steady state error for both $\theta_k,w_k$ is ${\cal O}(1/k^{\frac{2}{3}})$.
Interestingly, the above works do not obtain the fast rate in Theorem~\ref{theo:preliminary-bound-markov}, i.e., $\E[ \| \theta_k - \theta^\star \|^2 ] = {\cal O}(1/k)$. One of the reasons for the sub-optimality in their rates is that their analysis are based on building a single Lyapunov function that controls both errors in $\theta_k$ and $w_k$. In contrast, our analysis relies on a set of coupled inequalities to obtain tight bounds for each of the iterates $\theta_k$, $w_k$.
 
\section{Convergence Analysis} \label{sec:upperbd}
While much of the technical details and the complete constants of non-asymptotic bounds will be postponed to the appendix, this section offers insights into our main theoretical results through sketching the major steps involved in proving Theorem~\ref{theo:preliminary-bound-martingale} \& \ref{theo:preliminary-bound-markov}.
Throughout, we shall consider the following bounds on the step sizes and step size ratio:
\begin{align}
\label{eq:condition-beta-gamma-k}
\begin{split}
& \beta_\infty^{(0)} := \frac{1}{2 \normop{Q_\Delta}[]^{2} \normop{\Delta}[Q_\Delta]^{2}} \wedge \frac{1}{ 2 \normop{ \Delta }[Q_\Delta] + a_\Delta},~~ \gamma_\infty^{(0)} := \frac{1}{2 \normop{Q_{22}}[]^{2} \normop{A_{22}}[Q_{22}]^{2}} 
\eqsp, 
\end{split}
\end{align}
\beq \label{eq:condition-kappa}
\begin{split}
\kappa_\infty & := \Bigg( \frac{a_{22}/2}{ \normop{A_{12}}[Q_{22},Q_\Delta] \normop{ A_{22}^{-1} A_{21} }[Q_\Delta, Q_{22}] + \frac{a_\Delta}{2}} \Big\{ 1 \wedge  \frac{a_\Delta / 2}{ \normop{\Delta}[Q_\Delta] + \frac{a_\Delta}{2}} \Big\} \Bigg) \wedge \frac{ a_{22} }{4 a_\Delta}.
\end{split}
\eeq
To begin with, let us present the reformulation of the two time-scale SA scheme \eqref{eq:2ts1}, \eqref{eq:2ts2} that is borrowed from \citep{konda:tsitsiklis:2004}.
Define:
\[
L_{k+1} := \big( L_k - \gamma_k A_{22} L_k + \beta_k A_{22}^{-1} A_{21} (\Delta - A_{12}L_k) \big) \big( \Id - \beta_k (\Delta - A_{12} L_k) \big)^{-1}, \quad L_0 := 0,
\]
and $\Linfty := a_\Delta/(2 \normop{A_{12}}[Q_{22},Q_\Delta])$.
As shown in Lemma~\ref{lem:LkBound} of the appendix, with the step sizes $\gamma_k \leq \gamma_\infty^{(0)}$, $\beta_k \leq \beta_\infty^{(0)}$, $\kappa \leq \kappa_\infty$, the above recursion on $L_k$ is well defined where it holds that $\normop{L_k}[Q_\Delta,Q_{22}] \leq \Linfty$ for any $k \geq 0$.
In addition, define the matrices:
\[
B_{11}^k := \Delta - A_{12} L_k,\quad B_{22}^k := \frac{\beta_k}{\gamma_k} \big( L_{k+1} + A_{22}^{-1} A_{21} \big) A_{12} + A_{22},\quad C_k := L_{k+1} + A_{22}^{-1} A_{21}.
\]
We obtain a simplified two timescale SA recursions (proof in Appendix~\ref{app:obs}):
\begin{observation} \label{obs:transform}
Consider the following change-of-variables:
\beq \label{eq:tilde_def}
\ttheta_k := \theta_k - \theta^\star,\quad \tw_k = w_k - w^\star + C_{k-1} \ttheta_k.
\eeq
The two time-scale SA \eqref{eq:2ts1}, \eqref{eq:2ts2} is equivalent to the following iterations:
\beq \label{eq:2ts1-1}
\ttheta_{k+1} = ( \Id - \beta_k B_{11}^k ) \ttheta_k - \beta_k A_{12} \tw_k - \beta_k V_{k+1}. \vspace{-.1cm}
\eeq
\beq \label{eq:2ts2-1}
\tw_{k+1} = (\Id - \gamma_k B_{22}^k ) \tw_k - \beta_k C_k V_{k+1} - \gamma_k W_{k+1}.
\eeq
\end{observation}
Observe that $\ttheta_k = 0, \tw_k = 0$ is equivalent to having $\theta_k = \theta^\star, w_k = w^\star$, i.e., the two timescale SA solves the linear system of equations \eqref{eq:linear_sys}. 
The simplified recursion \eqref{eq:2ts1-1}, \eqref{eq:2ts2-1} \emph{decouples} the update of $\tw_k$ from $\ttheta_k$. This allows one to treat the $\tw_k$ update as a one timescale linear SA, and therefore provides a shortcut to perform a tight analysis.
We focus on estimating the following operator norms of  covariances: 
\beq
\label{eq:definition-mr}
\mtw_k := \normop{\E[\tw_k \tw_k^\top]}[] ,\quad \mth_k := \normop{\E[\ttheta_k \ttheta_k^\top]}[], \quad \mthw_k:= \normop{\E[\ttheta_{k} \tw_{k}^\top]}{} ,
\eeq
which are respectively the covariance for $w_k$, $\theta_k$ and the cross-variance between $w_k$, $\theta_k$.

\subsection{Proof Outline of Theorem~\ref{theo:preliminary-bound-martingale}} \label{sec:bd_martin}
For this theorem, we assume the step sizes and their ratio are chosen such that
\beq \label{eq:gamma_martingale}
\begin{split}
&  \gamma_k \leq \gamma_\infty^{\sf mtg} := \gamma_\infty^{(0)} \wedge \frac{1}{ \frac{a_{22}}{2} + \frac{2}{a_{22}} \pw ( \ttmv + \kappa^2 \ttmw) } \wedge \frac{a_\Delta}{4 \Ctwot},~~ \beta_k \leq \beta_\infty^{\sf mtg} := \beta_\infty^{(0)},
\end{split}
\eeq
where $\pw = \lambda_{\sf min}^{-1}( Q_{22} ) \lambda_{\sf max}( Q_{22} )$ and $\Ctwot$ is defined in \eqref{eq:martingale_finconst} in the appendix. 

While the property which the noise terms satisfy $\CPE{ V_{k+1} }{\F_k} = 0$, $\CPE{ W_{k+1} }{\F_k} = 0$ has greatly simplified the analysis, 
the challenge with our analysis lies in the coupling between  slow and fast updating iterates whose convergence rates must be carefully characterized in order to obtain the desired rate in Theorem~\ref{theo:preliminary-bound-martingale}.
To summarize, our proof consists of three steps in order: {\sf(i)} we bound $\mtw_k$ with an inequality that is coupled with $\mth_k$; then {\sf (ii)} we bound the cross term $\mthw_k$ using an inequality coupled with $\mth_k$; lastly, {\sf (iii)} these bounds are combined to bound $\mth_k$.

\paragraph{Step 1: Bounding $\mtw_k$} Upon applying the variable transformation in Observation~\ref{obs:transform}, \eqref{eq:2ts2-1} can be treated as a one-timescale SA which updates $\tw_k$ independently, and the contributions from $\ttheta_k$ are only found in the noise term, as seen from \eqref{eq:noise:constants}. This leads to:
\begin{proposition} \label{prop:mw}
Assume A\ref{assum:hurwitz}--\ref{assum:bound-conditional-variance} and the step sizes satisfy \eqref{eq:gamma_martingale}. For any $k \in \nset$, it holds
\beq
\label{eq:mwbound} \textstyle
\mtw_{k+1} \leq 
\prod_{\ell=0}^k \big( 1 - \frac{ \gamma_\ell a_{22} }{2} \big) \frac{\lambda_{\sf max}( Q_{22} )}{\lambda_{\sf min}(Q_{22})} \mtw_0 + \Conew \gamma_{k+1} + \Ctwow \sum_{j=0}^k \gamma_j^2 \prod_{\ell=j+1}^k \big( 1 - \frac{ \gamma_\ell a_{22} }{2} \big) \mth_j, 
\eeq
where the constants $\Conew, \Ctwow$ can be found in \eqref{eq:mwbound_const} in the appendix.
\end{proposition}
The right hand side of \eqref{eq:mwbound} consists of three components: (i) a fast decaying term relying on the product $\prod_{\ell=0}^k (1-\gamma_\ell a_{22}/2)$, (ii) an ${\cal O}(\gamma_k)$ term, and (iii) a convolutive term between $\mth_k$ and the fast decaying term depending on the step size sequence $(\gamma_k)_{k \geq 0}$. In the above, the second term can be viewed as a `steady-state' term.

\paragraph{Step 2: Bounding $\mthw_k$} Observe that $\mthw_k$ refers to the cross variance between $\tw_k$ and $\ttheta_k$. We show that utilizing \eqref{eq:2ts1-1}, \eqref{eq:2ts2-1}, \eqref{eq:mwbound} allows us to derive:
\begin{proposition} \label{prop:mthw}
Assume A\ref{assum:hurwitz}--\ref{assum:bound-conditional-variance} and the step sizes satisfy \eqref{eq:gamma_martingale}. For any $k \in \nset$, it holds
\beq \textstyle
\label{eq:mtwbound}
\mthw_{k+1} \leq \Czerotw \prod_{\ell=0}^k \big( 1 - \frac{ \gamma_\ell a_{22} }{2} \big) + \Conetw \beta_{k+1} + \Ctwotw \sum_{j=0}^k \gamma_j^2 \prod_{\ell=j+1}^k \big( 1 - \frac{ \gamma_\ell a_{22} }{2} \big) \mth_j,
\eeq
where the constants $\Czerotw, \Conetw, \Ctwotw$ can be found in \eqref{eq:mthwbound_const} in the appendix.
\end{proposition}
The above bound is a crucial step in obtaining the ${\cal O}(\beta_k)$ rate for $\mth_k$.
To better appreciate it, note that as $\mthw_k \leq (\sqrt{\dth \dw}/2) \{ \mth_k + \mtw_k \}$ (see Lemma~\ref{lem:key-inequality} in the appendix),
one can derive a similar result to \eqref{eq:mtwbound} by merely applying Proposition~\ref{prop:mw}. However, doing so results in an overestimated `steady-state' error of ${\cal O}(\gamma_{k})$ which is worse than the ${\cal O}(\beta_{k})$ error in \eqref{eq:mtwbound}. On the other hand, we take care of the two timescale nature of the algorithm to obtain \eqref{eq:mtwbound} with the fast rate. 

\paragraph{Step 3: Bounding $\mth_k$} Having equipped ourselves with Proposition~\ref{prop:mw} and \ref{prop:mthw}, we can analyze $\mth_k$ using \eqref{eq:2ts1-1} and the derived bounds on $\mtw_k, \mthw_k$, this leads to
\begin{proposition} \label{prop:mth}
Assume A\ref{assum:hurwitz}--\ref{assum:bound-conditional-variance} and the step sizes satisfy \eqref{eq:gamma_martingale}. For any $k \in \nset$, it holds
\beq \textstyle
\label{eq:mthbound}
\mth_{k+1} \leq \Czerot \prod_{\ell=0}^k \big( 1 - \frac{ \beta_\ell a_{\Delta} }{2} \big) + \Conet \beta_{k+1} 
+ \Ctwot \sum_{j=0}^{k} \gamma_j \beta_j \prod_{\ell=j+1}^k \big( 1 - \frac{ \beta_\ell a_{\Delta} }{2} \big) \mth_j,
\eeq
where the constants $\Czerot, \Conet, \Ctwot$ are given in \eqref{eq:martingale_finconst} in the appendix.
\end{proposition}
Besides that the middle term is now ${\cal O}(\beta_k)$, we also observe that the convolution term with $(\mth_j)_{j \geq 0}$ depends on the \emph{product of step sizes} $\beta_j \gamma_j$. This bound is obtained using Proposition~\ref{prop:mthw} and the fact that the cross variance $\mthw_k$ has a steady-state error of ${\cal O}(\beta_k)$.

Eq.~\eqref{eq:mthbound} is a recursive inequality as $\mth_k$ are found on both sides. In the appendix, we show that there exists a sequence $(\opu_k)_{k \geq 0}$ satisfying $\mth_k \leq \opu_k$ and
\beq \label{eq:recursion_main_paper}
\opu_{k+1} \leq (1 - \beta_k a_\Delta / 4) \opu_k + \Conet (a_\Delta/2) \beta_k^2
\eeq
for some constant $\Conet$. This immediately leads to \eqref{eq:convergence-slow}, followed by \eqref{eq:tracking-fast-component} similarly.

\subsection{Proof Outline of Theorem~\ref{theo:preliminary-bound-markov}} \label{sec:bd_markov}
While our proof has largely followed the same strategy as in the martingale noise case, now that the main challenge in handling the Markovian noise case is that the noise terms $V_{k+1}, W_{k+1}$ are no longer (conditionally) zero-mean. To circumvent this difficulty, we recall B\ref{assb:poisson} and define the following using the solution of the Poisson equation: for any $i,j=1,2$, 
\beq
\begin{split}
& \markovterm^{b_i}_k := \MK \widehat{b}_i (X_k), \quad \Markovterm^{A_{ij}}_k := \MK \widehat{A}_{ij}(X_k), \\
& \martinterm^{b_i}_k := \widehat{b}_i(X_{k+1}) - \MK \widehat{b}_i(X_k), \quad \Martinterm^{A_{ij}}_k := \widehat{A}_{ij}(X_{k+1}) - \MK \widehat{A}_{ij}(X_k),
\end{split}
\eeq
where $\martinterm^{b_i}_k, \Martinterm^{A_{ij}}_k$ are zero mean when conditioned on $\F_k$. The noise terms \eqref{eq:noise_term} can be rewritten as
\beq \label{eq:vkwk_rewrite} \begin{split}
V_{k+1} & = \underbrace{\martinoneb_k + \martinoneA_k \theta_k + \martinoneB_k w_k}_{=: {V}_{k+1}^{(0)}} + \underbrace{  (\markovoneb_k - \markovoneb_{k+1}) + (\markovoneA_k - \markovoneA_{k+1}) \theta_k + (\markovoneB_k - \markovoneB_{k+1}) w_k }_{=: V_{k+1}^{(1)}} \\
W_{k+1} & = \underbrace{\martintwob_k + \martintwoA_k \theta_k + \martintwoB_k w_k}_{=: {W}_{k+1}^{(0)}}   + \underbrace{ (\markovtwob_k - \markovtwob_{k+1}) + (\markovtwoA_k - \markovtwoA_{k+1}) \theta_k + ( \markovtwoB_k - \markovtwoB_{k+1} ) w_k  }_{=: W_{k+1}^{(1)}}.
\end{split}
\eeq
We observe that $\CPE{{V}_{k+1}^{(0)}}{\F_k} = 0, \CPE{{W}_{k+1}^{(0)}}{\F_k} = 0$ and therefore \eqref{eq:vkwk_rewrite} separates the noise terms into their martingale ($V_{k}^{(0)}, W_{k}^{(0)}$) and Markovian ($V_{k}^{(1)}, W_{k}^{(1)}$) components. Under B\ref{assb:bdd}, the second order moment of these noise components satisfy  A\ref{assum:bound-conditional-variance}. Accordingly, we define $\ttheta^{(0)}_0 = \ttheta_0, \ttheta_0^{(1)} = 0$, and $\tw^{(0)}_0 = \tw_0, \tw^{(1)}_0 = 0$ and the recursions:
\beq \label{eq:recur_wt}
\begin{split}
& \ttheta^{(i)}_{k+1} = (\Id - \beta_k B_{11}^k) \ttheta^{(i)}_k - \beta_k A_{12} \tw^{(i)}_k - \beta_k {V}_{k+1}^{(i)} ,~i=0,1,\\
& \tw^{(i)}_{k+1} = (\Id - \gamma_k B_{22}^k) \tw^{(i)}_k - \gamma_k ({W}_{k+1}^{(i)} + C_k V_{k+1}^{(i)}) ,~i=0,1,
\end{split}
\eeq
where it holds that $\ttheta_k = \ttheta_k^{(0)} + \ttheta_k^{(1)}$, $\tw_k = \tw_k^{(0)} + \tw_k^{(1)}$. Clearly, $\ttheta_k^{(0)}, \tw_k^{(0)}$ (resp.~$\ttheta_k^{(1)}, \tw_k^{(1)}$) are iterates of the two timescale SA driven by martingale (resp.~Markovian) noise. The two sets of recursions are independent except the second order moments of noise are bounded by $\mth_k, \mtw_k$, containing the contributions from $\ttheta_k^{(0)}, \tw_k^{(0)}$ and $\ttheta_k^{(1)}, \tw_k^{(1)}$. 

In the sequel, we show the martingale noise driven terms $\| \E[ \tw_{k}^{(0)} (\tw_k^{(0)})^\top] \|$, $\| \E[ \tw_{k}^{(0)} (\ttheta_k^{(0)})^\top] \|$, $\| \E[ \ttheta_{k}^{(0)} (\ttheta_k^{(0)})^\top] \|$ can be estimated using similar procedures as in Proposition~\ref{prop:mw}--\ref{prop:mth} from the previous subsection. Meanwhile the Markovian noise driven terms $\| \E[ \tw_{k}^{(1)} (\tw_k^{(1)})^\top] \|$ vanish at a faster rate than the former. Throughout this subsection, we set the step sizes to satisfy:
\beq
\label{eq:gamma_markov}
\begin{split}
&  \gamma_k \leq \gamma_\infty^{\sf mark} := \gamma_\infty^{(0)} \wedge \frac{1/\sqrt{ \dth \vee \dw }}{6 \pw \boundwsim} \wedge \frac{a_{22}/4}{\tC_0 + \tC_3},~~ \beta_k \leq \beta_\infty^{\sf mark} := \beta_\infty^{(0)} \wedge {\frac{1}{\sqrt{6 \tC_3^{(1,1)}}}} \wedge \frac{a_\Delta}{8 \Ctwottilde},\vspace{-.2cm}
\end{split}
\eeq
where $\pw = \lambda_{\sf min}^{-1}( Q_{22} ) \lambda_{\sf max}( Q_{22} )$, $\tC_0$, $\tC_3$, $\boundwsim$ are defined in \eqref{eq:tC0}, \eqref{eq:tC3}, \eqref{eq:boundwsim}, respectively, and $\tC_3^{(1,1)}$, $\Ctwottilde$ are defined in \eqref{eq:c11const}, \eqref{eq:ctheta_markov}, respectively, in the appendix.

\paragraph{Step 1: Bounding $\mtw_k$} We first show that the martingale and Markov noise driven iterates converge with different rates as follows:
\begin{lemma} \label{lem:m1interbd}
Assume A\ref{assum:hurwitz}--\ref{assum:stepsize}, B\ref{assb:mc1}--\ref{assb:step} and the step sizes satisfy \eqref{eq:gamma_markov}. For any $k \in \nset$, it holds 
\beq \label{eq:mtw0bd} 
\begin{split}
\| \E[ \tw^{(0)}_{k+1} ( \tw^{(0)}_{k+1} )^\top ] \| & \textstyle \leq  \prod_{\ell=0}^k \big( 1 - \frac{\gamma_\ell a_{22}}{2} \big)^2 \frac{\lambda_{\sf max}( Q_{22} )}{\lambda_{\sf min}(Q_{22})} \mtw_0 \\
& \quad \textstyle + \tC_0 \sum_{j=0}^k \gamma_j^2 \prod_{\ell=j+1}^k \big( 1 - \frac{\gamma_\ell a_{22}}{2} \big)^2 (1 + \mtw_j + \mth_j ) ,
\end{split}
\eeq
\beq \label{eq:mtw1bd}
\begin{split}
\| \E[ \tw^{(1)}_{k+1} ( \tw^{(1)}_{k+1} )^\top ] \| & \textstyle \leq \tC_1 \prod_{\ell=0}^k \big( 1 - \frac{\gamma_\ell a_{22}}{2} \big)^2  + \tC_2 \gamma_k^2 (\mth_{k+1} + \mtw_{k+1} ) + \tC_4 \gamma_{k}^2 \\
& \textstyle \quad + \tC_3 \gamma_{k+1} \sum_{j=0}^k \gamma_j^2 \prod_{\ell=j+1}^k \big( 1 - \frac{\gamma_\ell a_{22}}{2} \big)^2 ( \mth_j + \mtw_j ),
\end{split}
\eeq
where $\tC_0, \tC_1, \tC_2, \tC_3, \tC_4$ are constants defined in \eqref{eq:tC0}, \eqref{eq:tC3} in the appendix.
\end{lemma}
Let us compare the `steady-state' error on the right hand side of both inequalities: second term of \eqref{eq:mtw0bd} and the second to fourth term of \eqref{eq:mtw1bd}. We observe  those in the Markovian noise driven iterates $\tw_k^{(1)}$ are ${\cal O}(\gamma_k)$ times smaller than the martingale noise driven counterparts, indicating a faster convergence. This is roughly due to the special structure of the Markovian noise in $V_k^{(1)}, W_k^{(1)}$, where each term can be written as successive differences of a bounded sequence, e.g., $V_k^{(1)} \approx \xi_k - \xi_{k+1}$. When the linear SA \eqref{eq:recur_wt} is run over a long time horizon, the noise terms from consecutive iterations (roughly) cancels each other, leading to a significantly a smaller `steady-state' error.  

Using $\tw_k = \tw_k^{(0)} + \tw_k^{(1)}$ together with the above lemma give the following estimate for $\mtw_k$:
\begin{proposition} \label{prop:mw_markov}
Assume A\ref{assum:hurwitz}--\ref{assum:stepsize}, B\ref{assb:mc1}--\ref{assb:step} and the step sizes satisfy \eqref{eq:gamma_markov}. For any $k \in \nset$, it holds 
\beq \label{eq:mwmarkovbd} 
\mtw_{k+1} \leq \prod_{\ell=0}^k \big( 1 - \frac{\gamma_\ell a_{22}}{4} \big) \Czerowtilde + \Conewtilde \gamma_{k+1} + \Ctwowtilde \sum_{j=0}^k \gamma_j^2 \prod_{\ell=j+1}^k \big( 1 - \frac{\gamma_\ell a_{22}}{4} \big) \mth_j +  \Cthreewtilde \gamma_k^2  \mth_{k+1},
\eeq
where $\Czerowtilde, \Conewtilde, \Ctwowtilde, \Cthreewtilde$ are defined in \eqref{eq:mwmarkovbd_const} in the appendix.
\end{proposition}
We note in passing that by considering a special case with $\mth_k = 0$ for all $k$, the above proposition generalizes \citep[Theorem 7]{srikant:1tsbounds:2019} for linear one timescale  SA with Markovian noise. 

In a similar vein to the proof of Theorem~\ref{theo:preliminary-bound-martingale}, we bound the cross term $\| \E[ \ttheta_k^{(0)} (\tw_k^{(0)})^\top ] \|$ as:
\begin{lemma} \label{lem:mtw0crossbd}
Assume A\ref{assum:hurwitz}--\ref{assum:stepsize}, B\ref{assb:mc1}--\ref{assb:step} and the step sizes satisfy \eqref{eq:gamma_markov}. For any $k \in \nset$, it holds 
\beq \notag 
\| \E [ \ttheta^{(0)}_{k+1} (\tw^{(0)}_{k+1})^\top ] \| \leq \Czerotwtilde \prod_{\ell=0}^k \big( 1 - \frac{\gamma_\ell a_{22}}{4} \big) + \Conetwtilde \beta_{k+1} + \Ctwotwtilde \sum_{j=0}^k \gamma_j^2 \prod_{\ell=j+1}^k \big( 1 - \frac{\gamma_\ell a_{22}}{4} \big) \mth_j,
\eeq
where the constants $\Czerotwtilde, \Conetwtilde, \Ctwotwtilde$ are defined in \eqref{eq:mtwmarkovbd_const} in the appendix.
\end{lemma}
However, we observe that it is unnecessary to derive a similar (tight) bound for $\|\E [ \ttheta^{(1)}_{k} (\tw^{(1)}_{k})^\top ] \|$ as in the above lemma. The reason is that as observed in Lemma~\ref{lem:m1interbd}, the Markovian noise driven terms are anticipated to be sufficiently small compared to the martingale noise driven terms. In particular, a crude bound suffices to obtain the desirable convergence rate of $\mth_k$, as we observe next.

\paragraph{Step 2: Bounding $\mth_k$} 
Again we consider the bounds on $\| \E [ \ttheta^{(0)}_{k+1} ( \ttheta^{(0)}_{k+1} )^\top ] \|$ and $\E[\| \ttheta_{k+1}^{(1)} \|^2]$ separately. As we show in the appendix, both bounds are comparable as the Markovian noise term admits a successive difference structure. 
Using the decomposition $\ttheta_k = \ttheta_k^{(0)} + \ttheta_k^{(1)}$, we obtain:
\begin{proposition} \label{prop:mth_markov}
Assume A\ref{assum:hurwitz}--\ref{assum:stepsize}, B\ref{assb:mc1}--\ref{assb:step} and the step sizes satisfy \eqref{eq:gamma_markov}. For any $k \in \nset$, it holds
\beq \label{eq:mthmarkovbd} \textstyle
\mth_{k+1} \leq \Czerottilde \prod_{\ell=0}^k \big( 1 - \frac{\beta_\ell a_\Delta}{4} \big) + \Conettilde \beta_{k+1} + \Ctwottilde \sum_{i=0}^k \beta_i^2 \prod_{\ell=j+1}^k \big( 1 - \frac{\beta_\ell a_\Delta}{4} \big) \mth_i,
\eeq
where the constants $\Czerottilde, \Conettilde, \Ctwottilde$ are defined in \eqref{eq:ctheta_markov} in the appendix.
\end{proposition}
Equipped with Proposition~\ref{prop:mth_markov}, we can repeat the same steps as in \eqref{eq:recursion_main_paper} to derive an upper bound for $\mth_k$ through solving the recursive inequality \eqref{eq:mthmarkovbd}. Similar steps also apply for yielding \eqref{eq:tracking-fast-component-markov}.


\section{Tightness of the Finite-time Error Bounds} \label{sec:expansion}
This section examines the tightness of our finite time error bounds in Theorem~\ref{theo:preliminary-bound-martingale}, \ref{theo:preliminary-bound-markov} through characterizing the squared error $\E[ \| \theta_k - \theta^\star \|^2 ]$ with expansion. We consider the assumption: 
\begin{assum}
\label{assum:covariance matrices}
There exist matrices $\Sigma^{11}, \Sigma^{12}, \Sigma^{22}$, and a constant $\EmVW \geq 0$ such that for all $j \in \nset$, it holds
\begin{align} 
\normop{\E[V_jV_j^\top] - \Sigma^{11}} \vee \normop{ \E[W_j W_j^\top] - \Sigma^{22}} \vee \normop{\E[V_j W_j^\top] - \Sigma^{12}} \le \EmVW (\normop{\E[\theta_k \theta_k^\top]} + \normop{[w_k w_k^\top]}). \nonumber
\end{align}
\end{assum}
Note that A\ref{assum:covariance matrices} implies A\ref{assum:bound-conditional-variance} and therefore poses a stronger assumption. We have
\begin{theorem} \label{th: expansion}
Assume A\ref{assum:hurwitz}--\ref{assum:zero-mean}, A\ref{assum:covariance matrices} and for all $k \in \nset$, we have $\gamma_k \in [0, \gamma_\infty^{\sf mtg}]$, $\beta_k \in [0, \beta_\infty^{\sf exp}]$ and $\kappa \in [0,\kappa_\infty^{\sf exp}]$, where $\gamma_\infty^{\sf mtg}, \beta_\infty^{\sf exp}, \kappa_\infty^{\sf exp}$ are constants defined in \eqref{eq:gamma_martingale}, \eqref{eq:condition-beta expansion}, \eqref{eq:condition-kappa expansion}  in the appendix. Then for any $k \geq k_0^{\sf exp}:= \min\{\ell: \sum_{j=0}^{\ell-1} \beta_j \geq \log(2)/(2 \normop{\Delta})\}$, the following expansion holds
\beq \label{eq:expansion_form}
\E \big[ \|\theta_k - \theta^\star \|^2 \big] = I_k + J_k.
\eeq
The leading term $I_k$ is given by the following explicit  formula
$$ \textstyle
I_k: = \sum_{j=0}^{k} \beta_j^2 \Tr\left(\prod_{\ell=j+1}^k (\Id -  \beta_\ell \Delta) \, \Sigma \, \left\{\prod_{\ell=j+1}^k (\Id -  \beta_\ell \Delta)\right\}^\top \right),
$$
where $\Sigma := \Sigma^{11} + A_{12} A_{22}^{-1} \Sigma^{22}  A_{22}^{-\top} A_{12}^\top + \Sigma^{12}  A_{22}^{-\top} A_{12}^\top + A_{12} A_{22}^{-1} \Sigma^{21}$. Meanwhile, the following two-sided inequality holds
\beq
\label{eq:leading term of exp}
    \EConst{3} \Tr(\Sigma) \le \frac{I_k}{\beta_k} \le \EConst{4} \Tr(\Sigma),
\eeq 
and $J_k$ is bounded by
\beq 
\label{eq:remainder term of exp}
    |J_{k}| \le \EConst{0} \prod_{\ell=0}^{k-1} \left(1 - \frac{a_\Delta}{4} \beta_\ell \right) {\rm V}_0 + \EConst{1} \beta_{k} \left( \gamma_{k} + \frac{\beta_k}{\gamma_k} \right),
\eeq
where ${\rm V}_0$ was defined in \eqref{eq:definition-filtration}.
All constants $\EConst{0}, \EConst{1}, \EConst{3}$, $\EConst{4}$ are given in~\eqref{eq: exp const 0 1}, \eqref{eq: exp const 3} and \eqref{eq: exp const 4} in the appendix, respectively, and they are independent of $\beta_k,\gamma_k$. 
\end{theorem}
The proof is skipped in the interest of space, and it can be found in Appendix~\ref{app:ext}. 
Observe that from \eqref{eq:remainder term of exp}, the dominant term for $J_k$ is given by ${\cal O}(\beta_k \gamma_k + \frac{\beta_k^2}{\gamma_k})$. As such, using \eqref{eq:leading term of exp}, we observe that 
\[
{|J_k|} / {I_k} = {\cal O} \left( \gamma_k + {\beta_k} / {\gamma_k} \right)
\]
If $\lim_{k \rightarrow \infty} \beta_k / \gamma_k = 0$, we have $\lim_{k \rightarrow \infty} |J_k| / I_k = 0$. Combining \eqref{eq:expansion_form}, \eqref{eq:leading term of exp} shows that the expected error $\E[ \| \theta_k - \theta^\star \|^2 ]$ is lower bounded by $\Omega( \beta_k )$. 

We note that the assumptions A\ref{assum:hurwitz}--\ref{assum:zero-mean}, A\ref{assum:covariance matrices} imposed by the theorem imply A\ref{assum:hurwitz}--A\ref{assum:bound-conditional-variance} required by Theorem~\ref{theo:preliminary-bound-martingale}. Hence, together with \eqref{eq:convergence-slow} in Theorem~\ref{theo:preliminary-bound-martingale}, the above observations constitute a \emph{matching} lower bound on the convergence rate of linear two timescale SA with martingale noise. 
For the Markovian noise setting, we observe that if we impose the assumption that the random elements $(X_k)_{k \geq 0}$ are i.i.d., and $\widetilde{b}_i(x), \widetilde{A}_{ij}(x)$ are bounded above for any $i,j=1,2$ and $x \in {\sf X}$, then A\ref{assum:covariance matrices}, B\ref{assb:poisson}--B\ref{assb:bdd} can be satisfied. Therefore, the lower bound on the convergence rate also holds.

\vspace{-.1cm}
\section{Numerical Experiments, Conclusions} \label{sec:num}

We present numerical experiments to support our theoretical claims. We consider {\sf (a)} a toy example with a randomly generated problem parameters $b_i, A_{ij}$ and i.i.d.~samples $(X_k)_{k \in \nset}$ such that $\E[ \widetilde{b}_i(X_k) ] = b_i$, $\E[ \widetilde{A}_{ij}(X_k)] = A_{ij}$, {\sf (b)} the Garnet problem \citep{geist:offpolicy:2014} with the GTD algorithm \citep{sutton:gtd:2009} using $X_k$ from a simulated Markov chain. For example {\sf (a)}, we compute the stationary point $\theta^\star, w^\star$ exactly using \eqref{eq:opt_sol}; for example {\sf (b)}, while it is known that $w^\star = 0$, the solution $\theta^\star$ is computed using Monte Carlo simulation of the matrices $\widetilde{b}_i(X_k), \widetilde{A}_{ij}(X_k)$ with $2 \cdot 10^9$ iterations. The step sizes are chosen as $\beta_k = c^\beta/(k_0^\beta + k), \gamma_k = c^\gamma/(k_0^\beta + k)^{\sigma}$ with $\sigma \in \{ 0.5, 0.67, 0.75\}$. In the toy example {\sf (a)}, we have $\dth = \dw = 10,k_0^{\beta}= 10^4,k_0^{\gamma}=10^7, c^{\beta}=140,c^{\gamma}=300$; while for the Garnet problem {\sf (b)}, we have $k_0^{\beta}=8 \cdot 10^5,k_0^{\gamma}=2 \cdot 10^5, c^{\beta}=2300,c^{\gamma}=120$. Garnet problem is generated from family $n_S=30,n_A=2,b=2,p=8$, see \citep{geist:offpolicy:2014}. 
Further details about both experiments are described in Appendix~\ref{app:num}. 

\newcommand\parboxc[3]{%
    \settowidth{\mytemplength}{#3}%
    \parbox[#1][#2]{\mytemplength}{\centering #3}%
}

\begin{figure}[h!]
\centering
    \begin{tabular}{c c r c r c r c}
    \addvbuffer[0cm 2.6cm]{\sf (a)}\hspace{-.45cm} & \includegraphics[width=.2\linewidth]{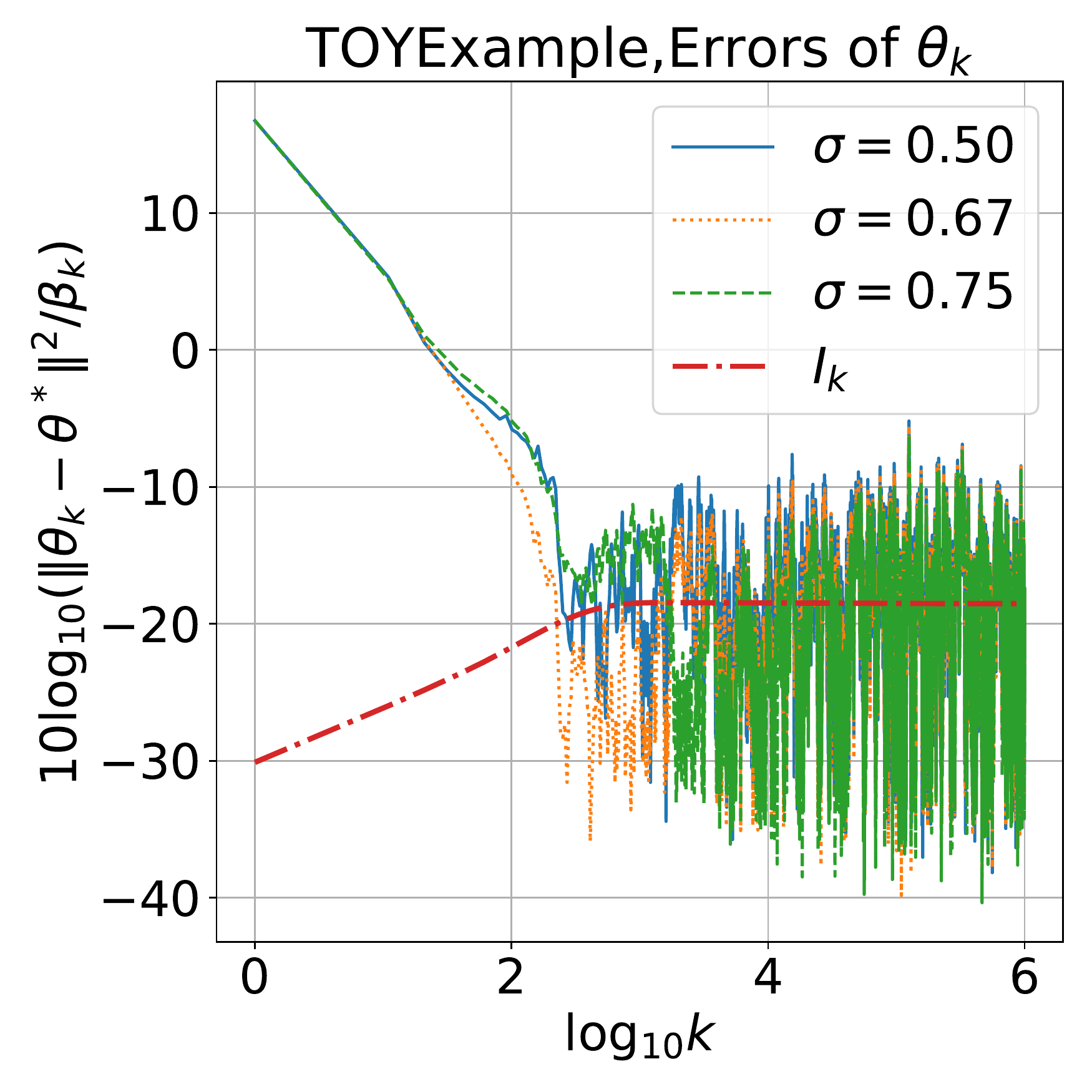} & \addvbuffer[0cm 2.6cm]{\sf (b)}\hspace{-.45cm} & \includegraphics[width=.2\linewidth]{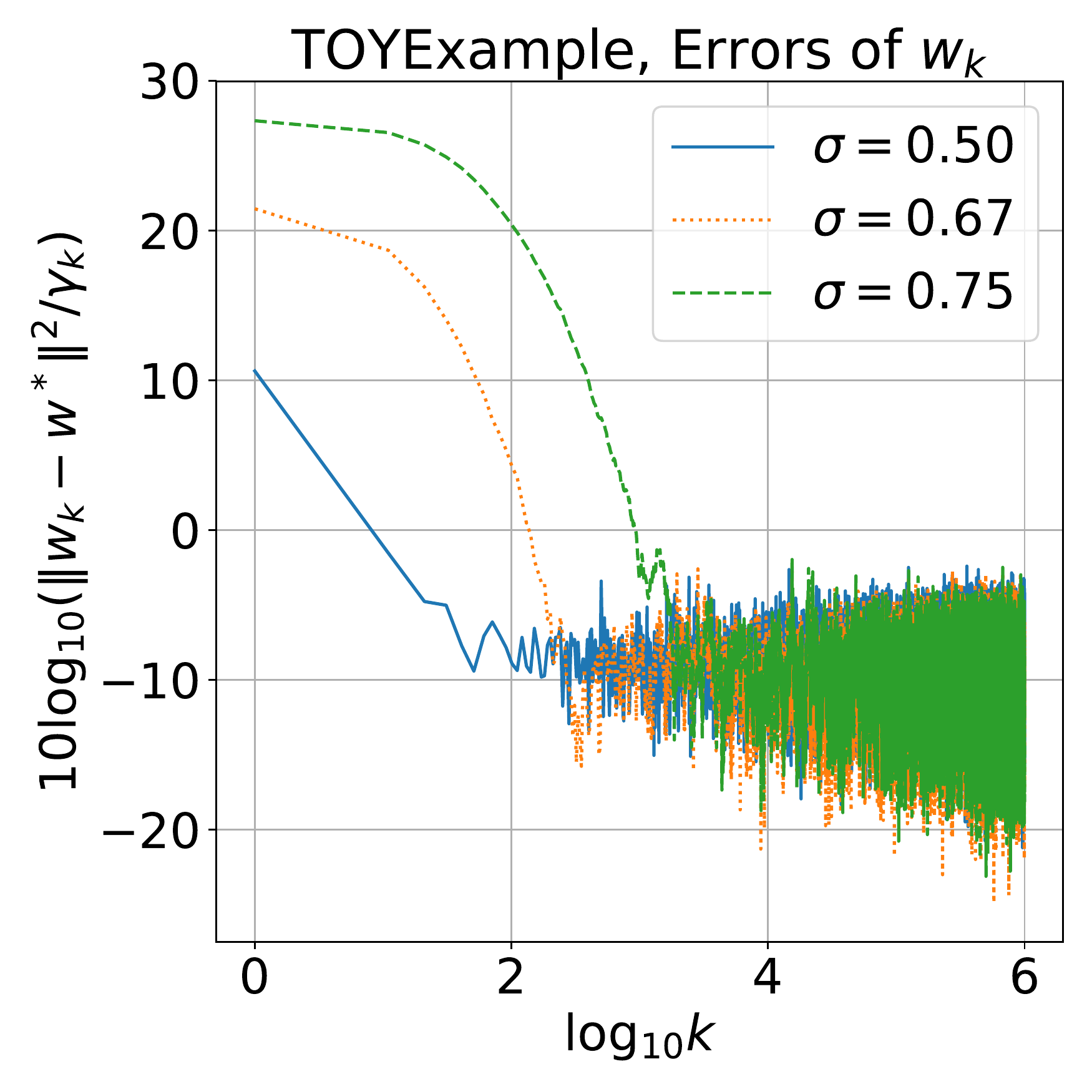} & \addvbuffer[0cm 2.6cm]{\sf (c)}\hspace{-.45cm} & \includegraphics[width=.2\linewidth]{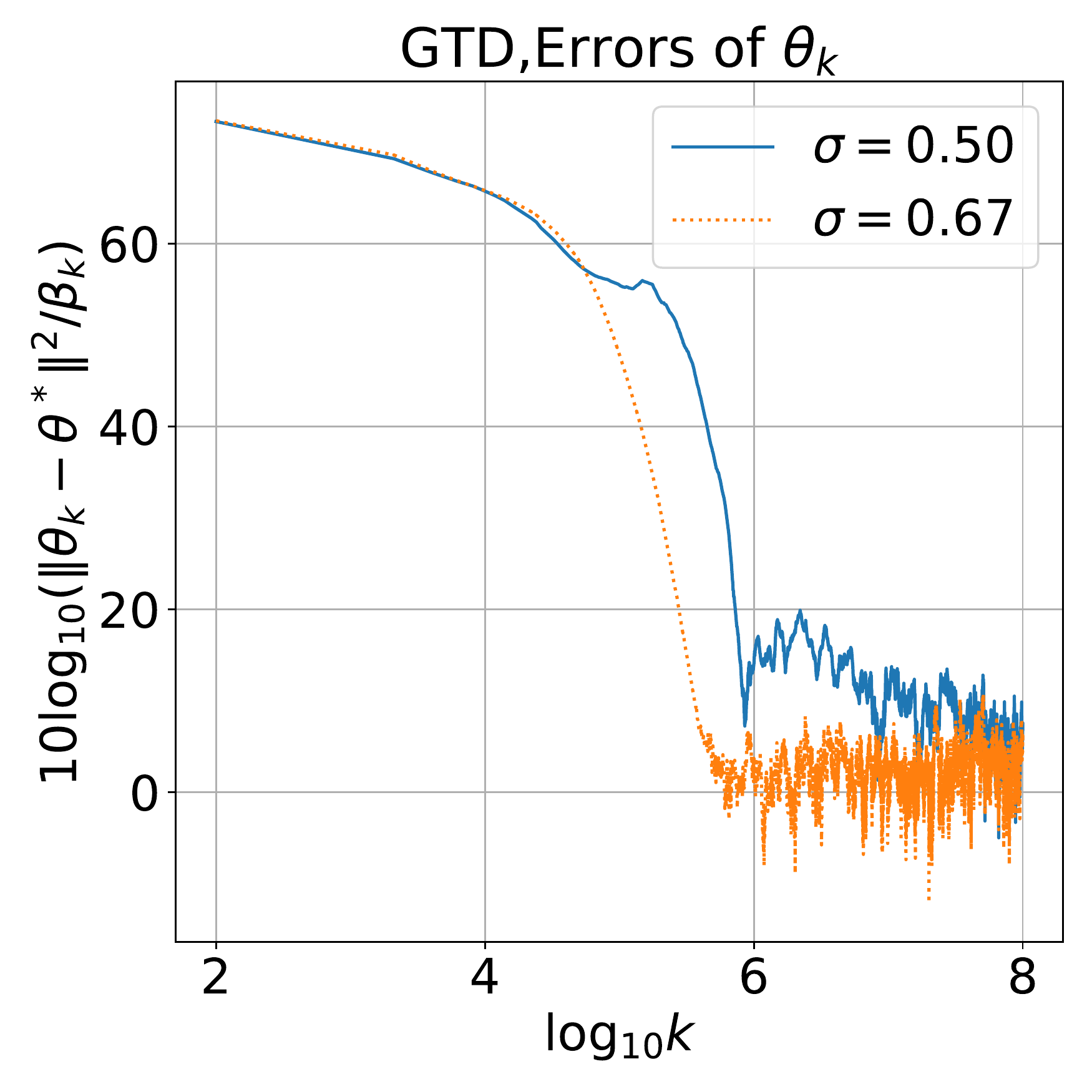} & \addvbuffer[0cm 2.6cm]{\sf (d)}\hspace{-.45cm} & \includegraphics[width=.2\linewidth]{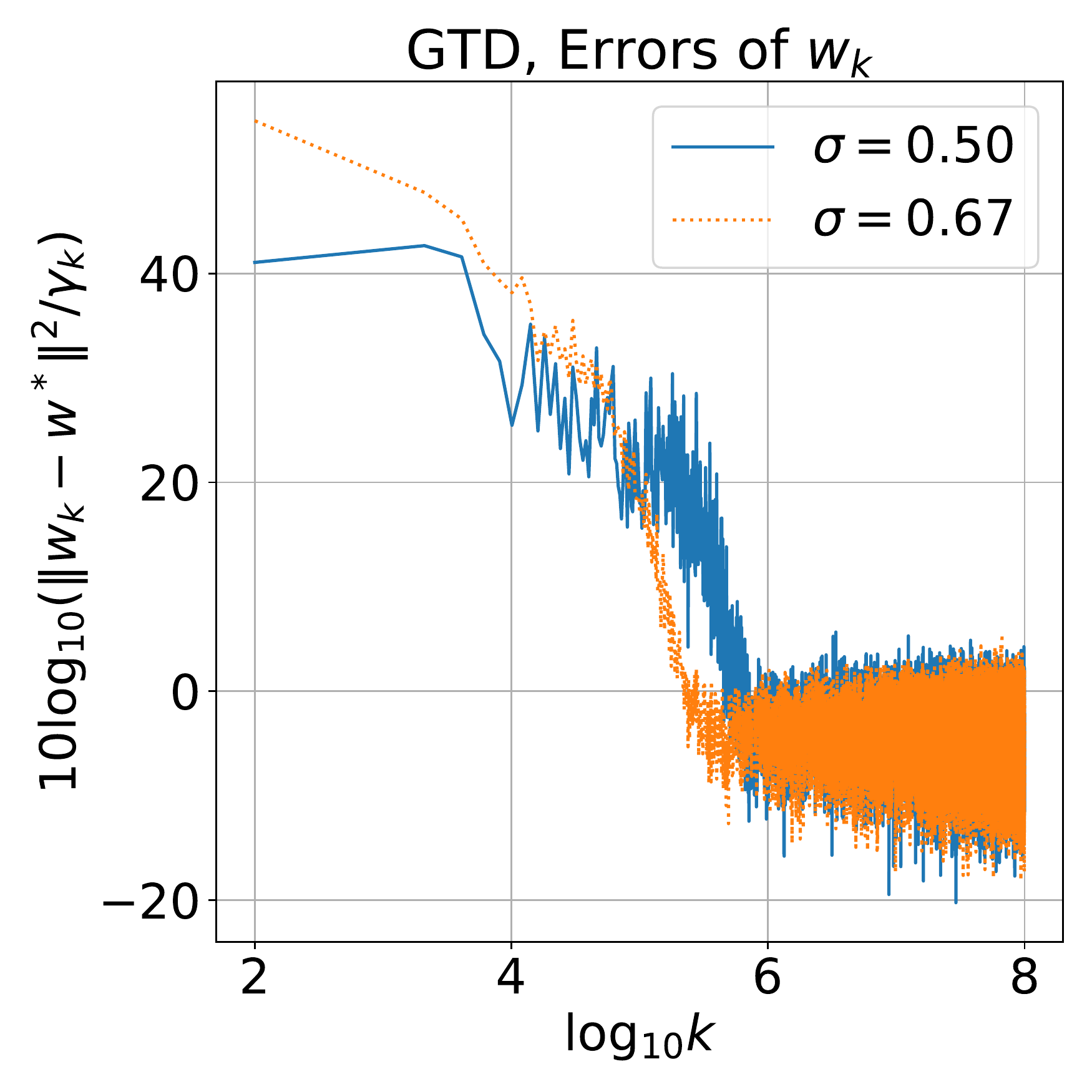}\vspace{-.4cm}
    \end{tabular}
    \caption{Deviations from stationary point $(\theta^\star, w^\star)$ \emph{normalized} by step sizes $\beta_k, \gamma_k$: {\sf (a,b)} the toy example, note we also show $I_k$ using the exact formula in Theorem~\ref{th: expansion} (unnormalized plot also available in the Appendix);  {\sf (c,d)} the Garnet problem.\vspace{-.5cm}}
    \label{fig:ratesMartingaleMarkov}
\end{figure}

We illustrate the convergence rates of the linear two timescale SA on the two problems in Figure~\ref{fig:ratesMartingaleMarkov}. Note that the plots show the (normalized) steady state errors are $\E[ \| \theta_k - \theta^\star \|^2 ] ={\cal O}(\beta_k)$, $\E[ \| w_k - w^\star \|^2 ] = {\cal O}(\gamma_k)$, which hold for both examples on martingale and Markovian noise. In addition, they are independent of the choice of $\sigma$. These observations agree with our main results. 

\paragraph{Conclusions} We have provided an improved finite time convergence analysis of the linear two timescale SA on both martingale and Markovian noises with relaxed conditions. Our analysis show that a tight analysis is possible through deriving and solving a sequence of recursive error bounds. Future works include the finite time analysis of nonlinear two timescale SA.

\newpage
\bibliographystyle{plain}
\bibliography{references}

\newpage
\appendix

\section{Proof of Observation~\ref{obs:transform}}\label{app:obs}
The following derivation is largely borrowed from \citep{konda:tsitsiklis:2004} and is repeated here for completeness. We begin by substituting $\ttheta_k$ into \eqref{eq:2ts1} to obtain
\[
\begin{split}
\ttheta_{k+1} & = (\Id - \beta_k A_{11}) \theta_k - \beta_k A_{12} w_k - \theta^\star + \beta_k b_1 - \beta_k V_{k+1} \\
& = (\Id - \beta_k A_{11}) \ttheta_k - \beta_k A_{11} \theta^\star - \beta_k A_{12} (\tw_k + w^\star - C_{k-1} \ttheta_k) + \beta_k b_1 - \beta_k V_{k+1} \\
& = (\Id - \beta_k (A_{11} - A_{12}A_{22}^{-1}A_{21} - A_{12} L_k))\ttheta_k - \beta_k A_{12} \tw_k - \beta_k ( A_{12} w^\star + A_{11} \theta^\star - b_1 ) - \beta_k V_{k+1}.
\end{split}
\]
Notice that
\[
\begin{split}
A_{12} w^\star + A_{11} \theta^\star - b_1 &  = (A_{11} - A_{12} A_{22}^{-1} A_{21}) \theta^\star + A_{12} A_{22}^{-1} b_2 - b_1 = 0.
\end{split}
\]
The above yields
\beq 
\ttheta_{k+1} = ( \Id - \beta_k B_{11}^k ) \ttheta_k - \beta_k A_{12} \tw_k - \beta_k V_{k+1}. 
\eeq
Next, we observe that
\[
\begin{split}
w_{k+1} - w^\star & = (\Id - \gamma_k A_{22}) w_{k} - \gamma_k A_{21} \theta_k - w^\star  + \gamma_k b_2 - \gamma_k W_{k+1} \\
& = (\Id - \gamma_k A_{22}) (w_{k} - w^\star) - \gamma_k A_{22} w^\star - \gamma_k A_{21} \theta_k + \gamma_k b_2 - \gamma_k W_{k+1} \\
& = (\Id - \gamma_k A_{22}) (w_{k} - w^\star) - \gamma_k A_{21} (  \theta_k - \theta^\star ) - \gamma_k W_{k+1}
\end{split}
\]
Substitute $\tw_k$ into \eqref{eq:2ts2} and using \eqref{eq:2ts1-1} yield:
\[
\begin{split}
& \tw_{k+1} = (\Id - \gamma_k A_{22}) (w_{k} - w^\star) - \gamma_k A_{21} \ttheta_k  + C_{k} \ttheta_{k+1} - \gamma_k W_{k+1} \\
& = (\Id - \gamma_k A_{22}) \tw_k - \big( (\Id - \gamma_k A_{22}) C_{k-1} + \gamma_k A_{21} \big) \ttheta_k + C_{k} \big( (\Id - \beta_k B_{11}^k) \ttheta_k - \beta_k A_{12} \tw_k \big) \\
& \quad - \beta_k C_k V_{k+1} - \gamma_k W_{k+1} \\
& = (\Id - \gamma_k B_{22}^k ) \tw_k - \Big( C_{k-1} - \gamma_k (A_{22} C_{k-1} - A_{21} ) - C_k ( \Id - \beta_k B_{11}^k) \Big) \ttheta_k - \beta_k C_k V_{k+1} - \gamma_k W_{k+1}
\end{split}
\]
We observe that
\[
\begin{split}
& C_{k-1} - \gamma_k (A_{22} C_{k-1} - A_{21} ) - C_k ( \Id - \beta_k B_{11}^k) \\
& = L_k + A_{22}^{-1} A_{21} - (L_{k+1} + A_{22}^{-1} A_{21}) ( \Id - \beta_k B_{11}^k) - \gamma_k (A_{22} C_{k-1} - A_{21} ) \\
& = L_k - ( L_k - \gamma_k A_{22} L_k + \beta_k A_{22}^{-1} A_{21} B_{11}^k ) - \beta_k A_{22}^{-1} A_{21} B_{11}^k - \gamma_k (A_{22} C_{k-1} - A_{21} ) \\
& = \gamma_k A_{22} L_k - \gamma_k (A_{22} ( L_k + A_{22}^{-1} A_{21})  - A_{21} ) = 0.
\end{split}
\]
The above yields
\beq 
\tw_{k+1} = (\Id - \gamma_k B_{22}^k ) \tw_k - \beta_k C_k V_{k+1} - \gamma_k W_{k+1}.
\eeq

\section{Detailed Proofs for Section~\ref{sec:upperbd}}

Before we proceed to proving the main results of Section~\ref{sec:upperbd}, we first study a few properties of the two timescale linear SA scheme.

To facilitate our discussions next, we define the constant:
\begin{align}
\label{eq:definition-CC} 
\Cinfty & := \sqrt{\lambda_{\sf min}(Q_\Delta)^{-1} \lambda_{\sf max}(Q_{22})} \Linfty+ \normop{A_{22}^{-1} A_{21}}[] ,
\end{align}
where $\| C_k \| \leq \Cinfty$ for any $k \geq 0$. Then, as we have $\theta_k \theta_k^\top \preceq 2 \ttheta_k \ttheta_k^\top + 2 \theta^\star (\theta^\star)^\top$, it holds
\[
\normop{\E[\theta_k \theta_k^\top]}[] \le 2 \big\{ \mth_k  +  \normop{\theta^\star (\theta^\star)^\top}[] \big\}, \quad
\normop{\E[w_k w_k^\top]}[] \le 3 \big\{\mtw_k + \mth_k \Cinfty^2 +  \normop{w^\star (w^\star)^\top}[] \big\}
\]
The noise terms $V_k,W_k$ can then be estimated in terms of the transformed variables $\ttheta_k, \tw_k$ and their variances $\mth_k, \mtw_k$. In particular, combining with A\ref{assum:bound-conditional-variance} yields
\begin{align} \label{eq:noise:constants}
\normop{\E[ V_{k+1} V_{k+1}^\top]}[] &\leq \ttmv (1 + \mth_k + \mtw_k),~~\normop{\E[ W_{k+1} W_{k+1}^\top]}[]  \leq \ttmw (1 + \mth_k + \mtw_k) \\
\label{eq:noise:constantsVW}
\normop{\E[ V_{k+1}W_{k+1}^\top]}[] & \leq \ttmVW (1 + \mth_k + \mtw_k )
\end{align}
where 
\beq \label{eq:ttm_def}
\begin{split}
& \frac{\ttmv}{\mv} = \frac{\ttmw}{\mw} = ( 1 + 2 \normop{\theta^\star (\theta^\star)^\top}[] + 3 \normop{w^\star (w^\star)^\top}[]) \vee (2 + 3 \Cinfty^2) \vee 3 \\
& \ttmVW  = \frac{\sqrt{\dth \dw}}{2} \left( \ttmw + \ttmv\right)
\end{split}
\eeq
We also define a few constants related to the matrices $Q_\Delta, Q_{22}$ associated with the Hurwitz matrices $\Delta, A_{22}$ in \eqref{eq:definition-Q-22}. Set $\pth := \lambda_{\sf min}^{-1}( Q_\Delta ) \lambda_{\sf max}( Q_\Delta )$, $\pw := \lambda_{\sf min}^{-1}( Q_{22} ) \lambda_{\sf max}( Q_{22} )$, $\pwth := \sqrt{ \pw \pth }$. Moreover, for any $a > 0$, we set 
\beq \label{eq:cseq}
\CSEQ{a} := \frac{2}{a} \varsigma \max\{1, a_{22}/(4a_\Delta)\} \vee \frac{4}{a}  \, (\varsigma)^3.
\eeq 
Next, we study the contraction properties of $\Id - \beta_k B_{11}^k$ and $\Id - \gamma_k B_{22}^k$ that appear in the transformed two timescale SA \eqref{eq:2ts1-1},\eqref{eq:2ts2-1}.
Using \eqref{eq:contraction_p}, we observe that 
\begin{align} 
    \normop{\Id - \beta_k B_{11}^k}[Q_\Delta]  = \normop{\Id - \beta_k \Delta + \beta_k A_{12} L_k}[Q_\Delta] 
    &\leq \normop{\Id - \beta_k \Delta}[Q_\Delta] + \beta_k \normop{A_{12}}[Q_{22},Q_{\Delta}] \normop{L_k}[Q_\Delta,Q_{22}] \nonumber \\
    &\leq (1-\beta_k a_\Delta) + \beta_k \normop{A_{12}}[Q_{22},Q_{\Delta}] \normop{L_k}[Q_\Delta,Q_{22}] \eqsp. \nonumber
\end{align}
Recalling that $\normop{L_k}[Q_\Delta,Q_{22}] \leq \Linfty$, the above inequality yields
\begin{equation}
\label{eq:contraction-B11}
\normop{\Id - \beta_k B_{11}^k}[Q_\Delta] \leq 1 - (1/2) \beta_k a_\Delta \eqsp.
\end{equation}
Since $\normop{\Id - \gamma_k B_{22}^k}[Q_{22}] \leq \normop{\Id - \gamma_k A_{22}}[Q_{22}]
+ \beta_k \normop{C_k A_{12}}[Q_{22}]$, we obtain the contraction:
\beq
\label{eq:contraction-B22}
\begin{split} 
\normop{\Id - \gamma_k B_{22}^k}[Q_{22}] & \leq 1 - \gamma_k a_{22} + \beta_k (\Linfty + \normop{ A_{22}^{-1} A_{21} }[Q_\Delta, Q_{22}]) \normop{A_{12}}[Q_{22},Q_\Delta] \\
& \leq 1 - (1/2) \gamma_k a_{22}.
\end{split}
\eeq
The last inequality is due to $\kappa  \leq (a_{22}/2) \{ (\Linfty + \normop{ A_{22}^{-1} A_{21} }[Q_\Delta, Q_{22}]) \normop{A_{12}}[Q_{22}, Q_\Delta]\}^{-1}$.
Lastly,
the following quantities will be used throughout the analysis: 
\begin{align}
&\ProdB_{m:n}^{(1)} := \prod_{i=m}^n (\Id - \beta_i B_{11}^i ), \quad \ProdB_{m:n}^{(2)} := \prod_{i=m}^n (\Id - \gamma_i B_{22}^i ), \nonumber \\
& G_{m:n}^{(1)} := \prod_{i=m}^n \Big(1 - (1/2) \beta_i a_{\Delta} \Big), \quad G_{m:n}^{(2)} := \prod_{i=m}^n \Big(1 - (1/2) \gamma_i a_{22} \Big). \nonumber
\end{align}
As a convention, we define $\ProdB_{m:n}^{(1)} = \ProdB_{m:n}^{(2)} = \Id$ if $m > n$.
In particular, for any $n,m \geq 0$,
we observe the following bound on the operator norm of $\ProdB_{m:n}^{(1)}$, 
\beq
\begin{split}
\| \ProdB_{m:n}^{(1)} \| & = \sqrt{\pth}
\| \ProdB_{m:n}^{(1)} \|_{Q_\Delta} \leq \sqrt{\pth} \prod_{i=m}^n \| \Id - \beta_i B_{11}^i \|_{Q_\Delta} 
\leq \sqrt{\pth} G_{m:n}^{(1)}
\end{split}
\eeq
Similarly, we have $\| \ProdB_{m:n}^{(1)} \| \leq \sqrt{\pw} G_{m:n}^{(2)}$. Lastly, we define
\[
\Sigma_k := \E \big[ \tw_{k} \tw_{k}^\top \big], \quad \Omega_k := \E \big[ \ttheta_k \tw_k^\top], \quad \Theta_k := \E \big[ \ttheta_k \ttheta_k^\top \big],
\]
whose operator norms correspond to $\mtw_k, \mthw_k, \mth_k$, respectively.

\subsection{Detailed Proof of Theorem \ref{theo:preliminary-bound-martingale}}

This subsection provides proofs to the propositions stated in Section~\ref{sec:bd_martin}, as well as providing detailed steps in establishing Theorem~\ref{theo:preliminary-bound-martingale}.


\paragraph{Bounding $\mtw_k$ (Proof of Proposition~\ref{prop:mw})} Using~\eqref{eq:2ts2-1}, as the noise terms are martingale, we get 
\begin{multline} \label{eq:w_recursion}
\CPE{\tw_{k+1} \tw_{k+1}^\top}{\F_k} = (\Id- \gamma_k B_{22}^k) \tw_k \tw_k^\top (\Id- \gamma_k B_{22}^k)^\top + \gamma_k^2 \CPE{W_{k+1}W_{k+1}^\top}{\F_k} \\ 
+ \beta_k^2 C_k \CPE{V_{k+1} V_{k+1}^\top}{\F_k}C_k^\top + \beta_k \gamma_k \left(\CPE{W_{k+1} V_{k+1}^\top}{\F_k} C_k^\top + C_k \CPE{V_{k+1} W_{k+1}^\top}{\F_k} \right).
\end{multline}
Repeatedly applying \eqref{eq:w_recursion} and taking the total expectation on both sides show
\beq \label{eq:sigma_recur}
\Sigma_{k+1} = \ProdB_{0:k}^{(2)} \Sigma_0 (\ProdB_{0:k}^{(2)})^\top + \sum_{j=0}^{k} \ProdB_{j+1:k}^{(2)} D_{j+1} (\ProdB_{j+1:k}^{(2)} )^\top,
\eeq
where  
\[
\begin{split}
& D_{k+1} = {\textstyle \gamma_k^2 \E [W_{k+1}W_{k+1}^\top] + \beta_k^2 C_k \E[ V_{k+1} V_{k+1}^\top] C_k^\top}  {\textstyle + \beta_k \gamma_k \big( \E[ {W_{k+1} V_{k+1}^\top}] C_k^\top + C_k \E[{V_{k+1} W_{k+1}^\top}] \big)}. \end{split}
\]
Using Lemma~\ref{lem:key-inequality}, we observe that 
\beq
\label{eq:crossbound}
\gamma_k \beta_k \| \E [W_{k+1} V_{k+1}^\top] C_k^\top \| \leq \frac{\sqrt{\dth \dw}}{2} C_\infty \big( \gamma_k^2 \| \E [W_{k+1} W_{k+1}^\top ] \| + \beta_k^2 \| \E [V_{k+1} V_{k+1}^\top ] \| \big),
\eeq
Let $K_C := \max\{ C_\infty^2,1 \} + \sqrt{\dth \dw} C_\infty$, we have 
\beq \notag
\begin{split}
\| D_{k+1} \| & 
\leq \gamma_k^2 \Big(1 + C_\infty \sqrt{\dth \dw}  \Big) \| \E [W_{k+1} W_{k+1}^\top ] \| + \beta_k^2 C_\infty \Big( C_\infty + \sqrt{\dth \dw} \Big) \| \E [V_{k+1} V_{k+1}^\top ] \| \\
& \leq K_C \Big( \gamma_k^2 \big\{ \ttmv + \ttmv \mth_k + \ttmv \mtw_k \big\}  + \beta_k^2 \big\{ \ttmw + \ttmw \mth_k + \ttmw \mtw_k \big\} \Big)
\end{split}
\eeq
where the last inequality is due to \eqref{eq:noise:constants}. Taking the operator norm on both sides of \eqref{eq:sigma_recur} yields
\beq \notag \mtw_{k+1} \leq \pw \Big\{ \big( G_{0:k}^{(2)} \big)^2 \mtw_{0} + K_C \sum_{j=0}^{k} 
 \big(  G_{j+1:k}^{(2)} \big)^2 ( \gamma_j^2 \ttmv + \beta_j^2 \ttmw ) \big\{ 1 + \mth_j + \mtw_j \big\} \Big\} .
\eeq
Using that $\beta_k \leq \kappa \gamma_k$ one writes
\beq \label{eq:mtw_intermediate}
\mtw_{k+1} \leq {\Czerow}' \big( G_{0:k}^{(2)} \big)^2 + c_0 {\Conew}' \sum_{j=0}^{k} \gamma_j^2  \big( G_{j+1:k}^{(2)} \big)^2 (1 + \mth_j)  + {\Ctwow}' \sum_{j=0}^{k} \gamma_j^2 \big( G_{j+1:k}^{(2)} \big)^2 \mtw_j
\eeq
where $c_0 = \ttmv+ \kappa^2\ttmw$, ${\Czerow}' = \pw \mtw_0$, ${\Conew}' =  \pw K_C$, and ${\Ctwow}' =  \pw c_0$.
Define:
\[
\topu_{k} = {\Czerow}' \big( G_{0:k-1}^{(2)} \big)^2 + c_0 {\Conew}' \sum_{j=0}^{k-1} \gamma_j^2  \big( G_{j+1:k-1}^{(2)} \big)^2 (1 + \mth_j)   + {\Ctwow}' \sum_{j=0}^{k-1} \gamma_j^2 \big( G_{j+1:k-1}^{(2)} \big)^2 \topu_j,
\]
It is easily seen that the sequence $(\topu_k)_{k \geq 0}$ is given by the following recursion
\begin{align} \nonumber
\topu_{k+1} =  (1-a_{22}\gamma_k/2)^2 \topu_k + c_0 {\Conew}' \gamma_k^2 (1 + \mth_k)   + {\Ctwow}' \gamma_k^2  \topu_k,~~\topu_0 = {\Czerow}'.
\end{align}
Since the step size was chosen such that $\gamma_k ( {\Ctwow}' + (a_{22}^2/4) ) \leq \frac{a_{22}}{2}$ [cf.~\eqref{eq:gamma_martingale}], we have 
\[
\topu_{k+1} \leq (1-a_{22}\gamma_k/2)\topu_k  + {\Conew}' \gamma_k^2 (c_0 + c_1 \mth_k) 
\]
which implies 
\beq
\topu_{k+1} \leq \Czero G_{0:k}^{(2)} + c_0 {\Conew}' \sum_{j=0}^{k} \gamma_j^2 (1 + \mth_j) G_{j+1:k}^{(2)}.
\eeq
Observe that $\mtw_k \leq \topu_k$.
Applying Corollary~\ref{cor:rate-of-convergence} shows that $\sum_{j=0}^k \gamma_j^2 G_{j+1:k}^{(2)} \leq \CSEQ{a_{22}/2} \gamma_{k+1}$, we get
\beq
\label{eq:mwbound_app}
\boxed{\mtw_{k+1} \leq 
\Czerow G_{0:k}^{(2)} + \Conew \gamma_{k+1} + \Ctwow \sum_{j=0}^k \gamma_j^2 G_{j+1:k}^{(2)} \mth_j, }
\eeq
where we recall $K_C := \max\{ C_\infty^2,1 \} + \sqrt{\dth \dw} C_\infty$, and
\beq \label{eq:mwbound_const}
\Czerow :=  \pw \mtw_0, \quad \Conew := \pw (\ttmv + \kappa^2 \ttmw) K_C  \CSEQ{a_{22}/2}, \quad \Ctwow := \pw K_C (\ttmv + \kappa^2 \ttmw) .
\eeq
This concludes the proof for Proposition~\ref{prop:mw}.

\paragraph{Bounding $\mthw_k$ (Proof of Proposition~\ref{prop:mthw})} We proceed by observing the following recursion of $\Omega_k$:
\beq
\begin{split}
\Omega_{k+1} &= (\Id - \beta_kB_{11}^k) \Omega_{k} (\Id - \gamma_kB_{22}^k)^\top - \beta_k A_{12} \Sigma_k (\Id - \gamma_kB_{22}^k)^\top \\
& +\beta_k\gamma_k \E[V_{k+1}W_{k+1}^\top] + \beta_k^2 \E[V_{k+1}V_{k+1}^\top]C_k^\top.
\end{split}
\eeq
Repeatedly applying the recursion gives
\begin{align}
\Omega_{k+1} &= \ProdB_{0:k}^{(1)} \Omega_0 \left( \ProdB_{0:k}^{(2)} \right)^\top - \sum_{j=0}^{k}\beta_j \ProdB_{j+1:k}^{(1)} A_{12} \Sigma_j \left( \ProdB_{j:k}^{(2)}\right)^\top \\
& + \sum_{j=0}^{k} \beta_j \gamma_j \ProdB_{j+1:k}^{(1)} \E\left[ V_{j+1} W_{j+1}^\top\right] \left(\ProdB_{j+1:k}^{(2)} \right)^\top + \sum_{j=0}^{k}\beta_j^2 \ProdB_{j+1:k}^{(1)} \E\left[ V_{j+1} V_{j+1}^\top \right] C_{j}^\top \left(\ProdB_{j+1:k}^{(2)} \right)^\top. \notag
\end{align}
The contraction properties \eqref{eq:contraction-B22}, \eqref{eq:contraction-B11}  result in
\begin{align}
\mthw_{k+1} & \leq \pwth \Big\{ G_{0:k}^{(1)} G_{0:k}^{(2)} \mthw_0 +  \normop{A_{12}}[] \sum_{j=0}^{k}\beta_j G_{j+1:k}^{(1)} G_{j:k}^{(2)}  \mtw_j \Big\} \label{eq:mtwbound_immediate} \\
& + \pwth \Big\{ \sum_{j=0}^{k} \beta_j \gamma_j G_{j+1:k}^{(1)}G_{j+1:k}^{(2)} \normop{\E [ V_{j+1} W_{j+1}^\top]} + \Cinfty \sum_{j=0}^{k}\beta_j^2 G_{j+1:k}^{(1)} G_{j+1:k}^{(2)} \normop{\E[ V_{j+1} V_{j+1}^\top]}[] \Big\} \notag
\end{align}
Applying \eqref{eq:mwbound_app}, we bound the third last term of \eqref{eq:mtwbound_immediate} as
\beq \label{eq:mtwbd_template}
\begin{split}
& \sum_{j=0}^{k}\beta_j G_{j+1:k}^{(1)} G_{j:k}^{(2)} \mtw_j
\leq \sum_{j=0}^{k}\beta_j G_{j+1:k}^{(1)} G_{j:k}^{(2)} \big(
\Czerow G_{0:j-1}^{(2)} + \Conew \gamma_{j} + \Ctwow \sum_{i=0}^{j-1} \gamma_i^2 G_{i+1:j-1}^{(2)} \mth_i
\big) \\
& \leq \frac{2 \Czerow G_{0:k}^{(2)}}{a_\Delta} + \Conew \sum_{j=0}^k \beta_j G_{j:k}^{(2)} \gamma_j + \Ctwow \sum_{j=0}^k \beta_j G_{j+1:k}^{(1)} G_{j:k}^{(2)} \sum_{i=0}^{j-1} \gamma_i^2 G_{i+1:j-1}^{(2)} \mth_i
\end{split}
\eeq
where we have used Lemma~\ref{lem:bsum} and  $G_{j+1:k}^{(1)} \leq 1$ in the last inequality.
Applying Corollary~\ref{cor:rate-of-convergence} and  Lemma~\ref{lem:summation-lemma}, A\ref{assum:stepsize} to the second and the last term on the right hand side, respectively, we obtain the following upper bound:
\beq \label{eq:w_term}
\begin{split}
\sum_{j=0}^{k}\beta_j G_{j+1:k}^{(1)} G_{j:k}^{(2)} \mtw_j &
\leq \frac{2 \Czerow G_{0:k}^{(2)}}{a_\Delta} + \Conew \CSEQ{a_{22}/2} \beta_{k+1} + \frac{2 \Ctwow}{a_\Delta} \sum_{i=0}^k \gamma_i^2 G_{i+1:k}^{(2)} \mth_i.
\end{split}
\eeq
Applying \eqref{eq:noise:constantsVW}, we bound the second last term of \eqref{eq:mtwbound_immediate} as 
\beq
\label{eq:mtwbound2} \begin{split}
& \sum_{j=0}^{k} \beta_j \gamma_j G_{j+1:k}^{(1)}G_{j+1:k}^{(2)} \normop{\E [ V_{j+1} W_{j+1}^\top]}
\leq 
\ttmVW \sum_{j=0}^{k} \beta_j \gamma_j G_{j+1:k}^{(1)}G_{j+1:k}^{(2)} \big( 1 + \mth_j + \mtw_j \big) \\
& \leq \ttmVW \Big\{ \sum_{j=0}^k \beta_j \gamma_j G_{j+1:k}^{(2)} + \sum_{j=0}^k \beta_j \gamma_j G_{j+1:k}^{(2)} \mth_j + \sum_{j=0}^k \beta_j \gamma_j G_{j+1:k}^{(1)} G_{j+1:k}^{(2)} \mtw_j \Big\} \\
& \leq \ttmVW \Big\{ \CSEQ{a_{22}/2} \beta_{k+1} + \kappa \sum_{j=0}^k \gamma_j^2  G_{j+1:k}^{(2)} \mth_j + \sum_{j=0}^k \beta_j \gamma_j G_{j+1:k}^{(1)} G_{j+1:k}^{(2)} \mtw_j \Big\}
\end{split}
\eeq
where the last inequality applied Corollary~\ref{cor:rate-of-convergence} again. We observe
\beq \label{eq:mtwbound3}
\begin{split}
\sum_{j=0}^{k} \beta_j \gamma_j G_{j+1:k}^{(1)} G_{j+1:k}^{(2)} \mtw_j &
\leq \frac{ \gamma_0 }{1 - \gamma_0 a_{22}/2} \sum_{j=0}^{k}\beta_j G_{j+1:k-1}^{(1)} G_{j:k-1}^{(2)} \mtw_j
\end{split}
\eeq
Thirdly, we repeat the calculations above and exploit $\beta_k \leq \kappa \gamma_k$ to bound
\beq \label{eq:mtwbound4} \begin{split}
& \sum_{j=0}^{k} \beta_j^2 G_{j+1:k}^{(1)}G_{j+1:k}^{(2)} \normop{\E [ V_{j+1} V_{j+1}^\top]}
\leq 
\ttmv \sum_{j=0}^{k} \beta_j^2 G_{j+1:k}^{(1)}G_{j+1:k}^{(2)} \big( 1 + \mth_j + \mtw_j \big) \\
& \leq \ttmv \Big\{ \sum_{j=0}^k \beta_j^2 G_{j+1:k}^{(1)} + \sum_{j=0}^k \beta_j^2 G_{j+1:k}^{(2)} \mth_j + \sum_{j=0}^k \beta_j^2 G_{j+1:k}^{(1)} G_{j+1:k}^{(2)} \mtw_j \Big\} \\
& \leq \ttmv \Big\{ \kappa \CSEQ{a_{22}/2} \beta_{k+1} + \kappa^2 \sum_{j=0}^k \gamma_j^2  G_{j+1:k}^{(2)} \mth_j + \kappa \sum_{j=0}^k \beta_j \gamma_j G_{j+1:k}^{(1)} G_{j+1:k}^{(2)} \mtw_j \Big\}.
\end{split}
\eeq
Combining \eqref{eq:w_term}, \eqref{eq:mtwbound2}, \eqref{eq:mtwbound3}, \eqref{eq:mtwbound4}, we conclude that
\beq \label{eq:mtwbound_app}
\boxed{ \mthw_{k+1} \leq \Czerotw G_{0:k}^{(2)} + \Conetw \beta_{k+1} + \Ctwotw \sum_{j=0}^k \gamma_j^2 G_{j+1:k}^{(2)} \mth_j }
\eeq
where
\beq \label{eq:mthwbound_const}
\begin{split}
& \Czerotw := \pwth  \Big( \mthw_0 + \normop{A_{12}}[] \frac{2 \Czerow}{a_\Delta} + (\ttmVW + \kappa C_\infty \ttmv) \frac{ 2 \gamma_0 \Czerow }{a_\Delta(1-\gamma_0 a_{22}/2)} \Big), \\
& \Conetw := \pwth  \CSEQ{a_{22}/2} \Big( \Conew \big( \normop{A_{12}}[] + \frac{\gamma_0}{1-\gamma_0 a_{22}/2} ( \ttmVW + C_\infty \kappa \ttmv ) \big)  + \ttmVW + C_\infty \kappa \ttmv \Big), \\
& \Ctwotw := \pwth \Big( \frac{2 \Ctwow}{a_\Delta} \big( \normop{A_{12}}[] + \frac{ \gamma_0 }{1 - \gamma_0 a_{22}/2 } (\ttmVW + C_\infty \kappa \ttmv ) \big) + \kappa \big( \ttmVW + C_\infty \kappa \ttmv \big) \Big).
\end{split}
\eeq
This concludes the proof of Proposition~\ref{prop:mthw}.

\paragraph{Bounding $\mth_k$ (Proof of Proposition~\ref{prop:mth})} We observe the following recursion:
\[
\begin{split}
& \CPE{ \ttheta_{k+1} \ttheta_{k+1}^\top }{\F_k}  = ( \Id - \beta_k B_{11}^k ) \CPE{ \ttheta_{k} \ttheta_{k}^\top }{\F_k} ( \Id - \beta_k B_{11}^k )^\top
+ \beta_k^2 A_{12} \CPE{ \tw_k \tw_k^\top }{ \F_k } A_{12}^\top \\
& \quad + \beta_k^2 \CPE{ V_{k+1} V_{k+1}^\top }{ \F_k } - \beta_k \Big( (\Id - \beta_k B_{11}^k) \CPE{ \ttheta_k \tw_k^\top }{\F_k} A_{12}^\top + A_{12} \CPE{ \tw_k \ttheta_k^\top }{\F_k} (\Id - \beta_k B_{11}^k )^\top \Big)
\end{split}
\]
Taking total expectations and evaluating the recursion gives
\[
\begin{split}
\Theta_{k+1} & = \ProdB_{0:k}^{(1)} \Theta_0 ( \ProdB_{0:k}^{(1)} )^\top +  \sum_{j=0}^k \beta_j^2 \ProdB_{j+1:k}^{(1)} (A_{12} \Sigma_j A_{12}^\top + \E[V_{j+1}V_{j+1}^\top]) \big(\ProdB_{j+1:k}^{(1)}\big)^\top \\
& - \sum_{j=0}^k \beta_j       \ProdB_{j+1:k}^{(1)}  ( (\Id - \beta_j B_{11}^j) \Omega_j A_{12}^\top + A_{12} \Omega_j^\top (\Id - \beta_j B_{11}^j)^\top ) \big(\ProdB_{j+1:k}^{(1)}\big)^\top 
\end{split}
\]
The above implies
\beq \label{eq:mth_immediate}
\begin{split}
\mth_{k+1} & \le \pth \bigg\{ (G_{0:k}^{(1)})^2 \mth_0 + 2\normop{A_{12}}[] \sum_{j=0}^k \beta_j (G_{j+1:k}^{(1)})^2 (1 - \beta_j a_\Delta/2) \mthw_j  \bigg\} \\
& + \pth \sum_{j=0}^{k} \beta_j^2 (G_{j+1:k}^{(1)})^2 \big( \normop{A_{12}}[]^2 \mtw_j + \| \E [ V_{j+1}V_{j+1}^\top ] \| \big)
\end{split} 
\eeq
Applying \eqref{eq:noise:constants} and Corollary~\ref{cor:rate-of-convergence} yield
\beq \label{eq:mtbound1}
\begin{split}
\mth_{k+1} & \leq \pth \Big\{ (G_{0:k}^{(1)})^2 \mth_0 + \ttmv \CSEQ{a_\Delta/2} \beta_{k+1} + \ttmv \sum_{j=0}^k \beta_j^2 (G_{j+1:k}^{(1)})^2 \mth_j \Big\}
\\ 
& + 2 \pth \| A_{12} \| \sum_{j=0}^k \beta_j G_{j:k}^{(1)} G_{j+1:k}^{(1)} \mthw_j + \pth \big(\| A_{12} \|^2 + \ttmv \big) \sum_{j=0}^{k} \beta_j^2 (G_{j+1:k}^{(1)})^2  \mtw_j ,
\end{split} 
\eeq
Applying \eqref{eq:mtwbound_app}, we can bound the second last term in \eqref{eq:mtbound1} as
\[
\begin{split}
& \sum_{j=0}^k \beta_j G_{j:k}^{(1)} G_{j+1:k}^{(1)} \mthw_j
\leq \sum_{j=0}^k \beta_j G_{j:k}^{(1)} G_{j+1:k}^{(1)} \Big( \Czerotw G_{0:j-1}^{(2)} + \Conetw \beta_{j} + \Ctwotw \sum_{i=0}^{j-1} \gamma_i^2 G_{i+1:j-1}^{(2)} \mth_i \Big) \\
& \leq \frac{2 \Czerotw}{a_\Delta} G_{0:k}^{(1)} +  \Conetw \CSEQ{a_{22}/2} \beta_{k+1} + \Ctwotw \sum_{j=0}^k \beta_j G_{j:k}^{(1)} G_{j+1:k}^{(1)} \sum_{i=0}^{j-1} \gamma_i^2 G_{i+1:j-1}^{(2)} \mth_i
\end{split}
\]
where the second inequality is derived using Corollary~\ref{cor:rate-of-convergence}. To bound the last term in the above we start from the following observation. 
Indeed, taking into account definition of $\beta_\infty$ in \eqref{eq:gamma_martingale}, we get
$(1 - \beta_\ell a_\Delta/2)^{-1} \le 1 + \beta_\ell a_\Delta$.
This inequality and assumption A\ref{assum:stepsize}-2 yield that
\beq \label{eq:atildebound}
\begin{split}
& \frac{\gamma_{\ell-1}}{\gamma_{\ell}} \frac{1 - \gamma_\ell a_{22}/2}{1 - \beta_\ell a_\Delta/2} \leq (1 + \epsilon_\gamma \gamma_\ell)(1 - \gamma_\ell a_{22}/2) ( 1+ \beta_\ell a_\Delta) \\ 
&  \le 1 - \gamma_\ell \bigg\{a_{22}/2 - a_\Delta \kappa  - \epsilon  \bigg \} + \epsilon_\gamma \gamma_\ell^2 \bigg\{\kappa a_\Delta - a_{22}/2\bigg\} \le 1 - (1/8) a_{22} \gamma_\ell, 
\end{split}
\eeq
since $\kappa_\infty \le (1/4) a_{22}/a_\Delta$, see~\eqref{eq:condition-kappa}. 
We observe the following chain
\beq \label{eq:mthbound2}
\begin{split}
& \sum_{j=0}^k \beta_j G_{j:k}^{(1)} G_{j+1:k}^{(1)} \sum_{i=0}^{j-1} \gamma_i^2 G_{i+1:j-1}^{(2)} \mth_i = \sum_{i=0}^{k-1} \gamma_i^2  \mth_i \sum_{j=i+1}^k \beta_j G_{j:k}^{(1)} G_{j+1:k}^{(1)} G_{i+1:j-1}^{(2)} \\
& = \sum_{i=0}^{k-1} \gamma_i^2 G_{i+1:k}^{(1)}  \mth_i \sum_{j=i+1}^k \beta_j G_{j+1:k}^{(1)} \frac{G_{i+1:j-1}^{(2)}}{G_{i+1:j-1}^{(1)}} \overset{(a)}{\leq} \sum_{i=0}^{k-1} \beta_i  \gamma_i G_{i+1:k}^{(1)}  \mth_i \sum_{j=i+1}^k \gamma_{j-1} \prod_{\ell=i+1}^{j-1} \frac{\gamma_{\ell-1}}{\gamma_\ell} \frac{G_{i+1:j-1}^{(2)}}{G_{i+1:j-1}^{(1)}} \\
& \overset{(b)}{\leq} \sum_{i=0}^{k-1} \beta_i  \gamma_i G_{i+1:k}^{(1)}  \mth_i \sum_{j=i+1}^k \gamma_{j-1}  \prod_{\ell=i+1}^{j-1} (1 - (1/8)\gamma_\ell a_{22} ) \overset{(c)}{\leq} \frac{8 \varsigma}{a_{22}} \sum_{i=0}^{k-1} \beta_i  \gamma_i G_{i+1:k}^{(1)}  \mth_i.
\end{split}
\eeq
where (a) is due to $\beta_j \leq \beta_i$ and $G_{j+1:k}^{(1)} \leq 1$, (b) is due to \eqref{eq:atildebound}, (c) is due to A\ref{assum:stepsize}-1 and $\sum_{j=i+1}^k \gamma_j \prod_{\ell=i+1}^{j-1} (1 - \gamma_\ell \tilde{a} ) \leq (8/a_{22})$ for any $i,k$.

Moreover, applying \eqref{eq:mwbound_app}, we can bound the last term of \eqref{eq:mtbound1} as:
\[
\begin{split}
& \sum_{j=0}^{k} \beta_j^2 (G_{j+1:k}^{(1)})^2  \mtw_j \leq \sum_{j=0}^{k} \beta_j^2 (G_{j+1:k}^{(1)})^2 \Big( \Czerow G_{0:j-1}^{(2)} + \Conew \gamma_j + \Ctwow \sum_{i=0}^{j-1} \gamma_i^2 G_{i+1:j-1}^{(2)} \mth_i \Big) \\
& \leq \frac{\Czerow G_{0:k}^{(1)}}{1-\beta_0 a_\Delta/2} \sum_{j=0}^k \beta_j^2 G_{j+1:k}^{(1)}  + \gamma_0 \Conew \sum_{j=0}^k \beta_j^2 G_{j+1:k}^{(1)} + \Ctwow \sum_{j=0}^{k} \beta_j^2 (G_{j+1:k}^{(1)})^2 \sum_{i=0}^{j-1} \gamma_i^2 G_{i+1:j-1}^{(2)} \mth_i \\
& \leq \Big( \frac{\Czerow G_{0:k}^{(1)}}{1-\beta_0 a_\Delta/2} + \gamma_0 \Conew \Big) \CSEQ{a_{22}/2} \beta_{k+1} + \Ctwow \sum_{j=0}^{k} \beta_j^2 (G_{j+1:k}^{(1)})^2 \sum_{i=0}^{j-1} \gamma_i^2 G_{i+1:j-1}^{(2)} \mth_i,
\end{split}
\]
where the last inequality is due to Corollary~\ref{cor:rate-of-convergence}. In addition, similar to \eqref{eq:mthbound2}, we can derive the bound
\[
\begin{split}
& \sum_{j=0}^{k} \beta_j^2 (G_{j+1:k}^{(1)})^2 \sum_{i=0}^{j-1} \gamma_i^2 G_{i+1:j-1}^{(2)} \mth_i  \leq \frac{ (8 \varsigma)/a_{22} }{1 - \beta_0 a_\Delta/2} \sum_{i=0}^{k} \beta_i^2 \gamma_i G_{i+1:k}^{(1)} \mth_i
\end{split}
\]
Substituting the above inequalities into \eqref{eq:mtbound1} leads to
\beq \label{eq:mthbound_app}
\boxed{\mth_{k+1} \leq \Czerot G_{0:k}^{(1)} + \Conet \beta_{k+1} 
+ \Ctwot \sum_{j=0}^{k} \gamma_j \beta_j G_{j+1:k}^{(1)} \mth_j}
\eeq
where 
\beq \label{eq:martingale_finconst}
\begin{split}
& \Czerot := \pth \Big( \mth_0 + \frac{4 \| A_{12} \| \Czerotw}{a_\Delta} \Big), \\
& \Conet := \pth \Bigg\{ \ttmv \CSEQ{a_\Delta/2} + 2 \| A_{12} \| \Conetw \CSEQ{a_{22}/2} +  ( \|A_{12}\|^2 + \ttmv) \Big( \gamma_0 \Conew + \frac{ \Czerow }{1 - \beta_0 a_\Delta/2} \Big) \CSEQ{a_{22}/2} \Bigg\}, \\
& \Ctwot:= \pth \Bigg\{ \frac{16 \varsigma \| A_{12} \| \Ctwotw }{a_{22}} + \ttmv + (\| A_{12}\|^2 + \ttmv) \frac{8 \Ctwow \varsigma / a_{22} }{1 - \beta_0 a_\Delta/2} \Bigg\}.
\end{split}
\eeq
This completes the proof for Proposition~\ref{prop:mth}. 

\paragraph{Completing the Proof of Theorem~\ref{theo:preliminary-bound-martingale}} We complete the proof by analyzing the convergence rate of $\mth_k$ using \eqref{eq:mthbound_app}. Consider the following recursion which upper bounds $\mth_k$:
\[
\opu_{k+1} = \Czerot G_{0:k}^{(1)} + \Conet \beta_{k+1} 
+ \Ctwot \sum_{j=0}^{k} \gamma_j \beta_j G_{j+1:k}^{(1)} \opu_j,
\]
where we have set $\opu_0 = \Czerot$.
Observe that
\[
\begin{split}
& \opu_{k+1} - (1-\beta_k a_\Delta/2) \opu_k = \Conet ( \beta_{k+1} - (1-\beta_k a_\Delta/2) \beta_k ) + \Ctwot \gamma_k \beta_k \opu_k \\
\Longleftrightarrow & \opu_{k+1} = (1 - \beta_k (a_\Delta / 2 - \Ctwot \gamma_k) ) \opu_k + \Conet ( \beta_{k+1} - \beta_k + \beta_k^2 a_\Delta/2 )
\end{split}
\]
Since $\gamma_k \leq \gamma_0 \leq \frac{a_\Delta}{4 \Ctwot}$, we have
\[
\opu_{k+1} \leq (1 - \beta_k a_\Delta/4) \opu_k + \Conet \beta_k^2 a_\Delta/2
\]
Evaluating the recursion gives
\[
\opu_{k+1} \leq \prod_{\ell=0}^k (1-\beta_\ell a_\Delta/4) \opu_0 + \Conet (a_\Delta/2) \sum_{j=0}^{k} \beta_j^2 \prod_{\ell=j+1}^k (1- \beta_\ell a_\Delta/4)
\]
Applying Corollary~\ref{cor:rate-of-convergence} shows $\sum_{j=0}^{k} \beta_j^2 \prod_{\ell=j+1}^k (1- \beta_\ell a_\Delta/4) \leq \CSEQ{a_\Delta/4} \beta_{k+1}$. Lastly, observing that $\mth_k \leq \opu_k$ gives
\beq
\label{eq:mthbound final}
\boxed{\mth_{k+1} \leq \Czerot \prod_{\ell=0}^{k}\Big(1-\beta_\ell \frac{a_\Delta}{4} \Big)  + \Conet \CSEQ{a_\Delta/4} \frac{a_\Delta}{2} \beta_{k+1}.}
\eeq    
To finish the proof of~\eqref{eq:convergence-slow}, we observe (i) the constant $\Czerot \leq {\rm C}_0^{\ttheta,\sf mtg} {\rm V}_0$ for some constant ${\rm C}_0^{\ttheta,\sf mtg}$,  
(ii) the inequality that $\E[ \| \theta_k - \theta^\star \|^2 ] \leq \dth \mth_k$, and (iii) setting the constant ${\rm C}_1^{\ttheta,\sf mtg} := \Conet \CSEQ{a_\Delta/4} (a_\Delta/2)$.

Our last endeavor is to prove~\eqref{eq:tracking-fast-component}. Observe that the tracking error $\hw_k := w_k - A_{22}^{-1} ( b_2 - A_{21} \theta_k)$ may be represented as 
\[
\begin{split}
\hw_k & = w_k - w^\star + w^\star - A_{22}^{-1} ( b_2 - A_{21} \theta_k) \\
& = \tw_k - C_{k-1} \ttheta_k + A_{22}^{-1} \big( (b_2 - A_{21} \theta^\star) - (b_2 - A_{21} \theta_k ) \big) = \tw_k - L_k \ttheta_k
\end{split}
\]
using the definitions in \eqref{eq:tilde_def}. This leads to the following estimate of 
$\mhr_k: = \normop{\E[\hw_k \hw_k^\top]}$:
\beq 
\label{eq:haw via tilde w}
\mhr_{k+1} \le 2 \mtw_{k+1} + 2 \| L_{k+1} \|^2 \mth_{k+1} \leq 2 \mtw_{k+1} + 2  \Linfty^2  \frac{\lambda_{\sf max}( Q_{22} )}{\lambda_{\sf min}( Q_\Delta )} \mth_{k+1}
\eeq
In particular, substituting~\eqref{eq:mthbound final} into~\eqref{eq:mwbound_app}, we obtain:
\[
\begin{split}
\mtw_{k+1} &\leq 
\Czerow G_{0:k}^{(2)} + \Conew \gamma_{k+1} + \Ctwow \sum_{j=0}^k \gamma_j^2 G_{j+1:k}^{(2)} \bigg\{ \Czerot \prod_{\ell=0}^{j-1} \Big(1-\beta_\ell \frac{a_\Delta}{4} \Big) + \Conet \CSEQ{a_\Delta/4} \frac{a_\Delta}{2} \beta_{j}\bigg\} \\
&\le \Czerow G_{0:k}^{(2)} +  \gamma_{k+1} \Bigg\{ \Conew + \Conet \CSEQ{a_\Delta/4} \frac{a_\Delta}{2} \Ctwow \CSEQ{a_{22}/2}  + \frac{\Ctwow \Czerot \CSEQ{a_{22}/2}}{1 - \beta_0 a_\Delta/4} \prod_{\ell=0}^k \Big(1-\beta_\ell \frac{a_\Delta}{4} \Big) \Bigg\}
\end{split}
\]
where the last inequality is due to the observation $G_{j+1:k}^{(2)} \leq \prod_{i=j+1}^k (1-\gamma_i a_{22}/4)^2$ and the application of Corollary~\ref{cor:rate-of-convergence}. Furthermore using $G_{0:k}^{(2)} \leq \prod_{\ell=0}^k (1 - \beta_\ell a_\Delta/4)$ and applying \eqref{eq:haw via tilde w} gives 
\beq \label{eq:mhwbound final}
\boxed{\mhr_{k+1} \leq {\rm C}_0^{w} \prod_{\ell=0}^k (1-\beta_\ell a_\Delta/4) + {\rm C}_1^{\hw,\sf mtg} \gamma_{k+1}, }
\eeq
where
\beq \label{eq:mhw_const}
\begin{split}
& {\rm C}_0^w := 2 \Big\{ \Linfty^2  \frac{\lambda_{\sf max}( Q_{22} )}{\lambda_{\sf min}( Q_\Delta )} \Czerot + \frac{\CSEQ{a_{22}/2} \Ctwow \Czerot}{1 - \beta_0 a_\Delta/4} + \Czerow \Big\} \\
& {\rm C}_1^{\hw,\sf mtg} := 2 \Big\{ \kappa \Linfty^2  \frac{\lambda_{\sf max}( Q_{22} )}{\lambda_{\sf min}( Q_\Delta )} \CSEQ{a_\Delta/4} \frac{a_\Delta}{2} \Conet + \Conew +  \CSEQ{a_\Delta/4} \frac{a_\Delta}{2} \Ctwow \CSEQ{a_{22}/2} \Conet \Big\}
\end{split}
\eeq
We conclude the proof for Theorem~\ref{theo:preliminary-bound-martingale} by observing that ${\rm C}_0^w \leq {\rm C}_0^{\hw,\sf mtg} {\rm V}_0$ for some constant ${\rm C}_0^{\hw,\sf mtg}$.

\subsection{Detailed Proofs of Theorem~\ref{theo:preliminary-bound-markov}}
To facilitate our discussions next, define a few additional constants as:
\beq
\begin{split}
& \widetilde{G}_{m:n}^{(1)} := \prod_{i=m}^n (1 - \beta_i a_\Delta / 4), \quad \widetilde{G}_{m:n}^{(2)} := \prod_{i=m}^n (1 - \gamma_i a_{22} / 4) \\
& \Boneinfty := \| \Delta \| + \sqrt{\lambda_{\sf min}(Q_\Delta)^{-1} \lambda_{\sf max}(Q_{22})} \Linfty \| A_{12} \|,~~ \Btwoinfty := \kappa \Cinfty \| A_{12} \| + \| A_{22} \| .
\end{split}
\eeq
Before we begin the proof, notice by observing the form of \eqref{eq:vkwk_rewrite} that that A\ref{assum:bound-conditional-variance} is satisfied by the Markovian noise through setting
\[
\mv = \overline{\rm b} \vee ( 3 \overline{\rm A} ),~\mw = \overline{\rm b} \vee ( 3 \overline{\rm A} ),
\]
and furthermore \eqref{eq:noise:constants} is satisfied with $\ttmv, \ttmw, \ttmVW$ defined in \eqref{eq:ttm_def} and the above $\mv, \mw$. Moreover, for $i=0,1$, the second order moments of the decomposed noise satisfy:
\begin{align}
\label{eq:noise:constants:0}
& \normop{\E[ V_{k+1}^{(i)} (V_{k+1}^{(i)})^\top]}[] \leq \ttmvi ( 1 + \mth_k + \mtw_k),~\normop{\E[ W_{k+1}^{(i)} (W_{k+1}^{(i)})^\top]}[] \leq \ttmwi (1 +  \mth_k + \mtw_k), \\
\label{eq:noise:constantsVW:0}
& \normop{\E[ V_{k+1}^{(i)} (W_{k+1}^{(i)})^\top]}[] \leq \ttmVWi (1 + \mth_k + \mtw_k),
\end{align}
for some constants $\ttmvi$, $\ttmwi$, $\ttmVWi$, $i=1,2$. 
We proceed with the proof for Theorem~\ref{theo:preliminary-bound-markov} as follows.
\paragraph{Bounding $\mtw_k$ (Proof of Lemma~\ref{lem:m1interbd} and Proposition~\ref{prop:mw_markov})} 
Repeating the analysis that leaded to \eqref{eq:mtw_intermediate} and using the martingale property of $V_{k+1}^{(0)}, W_{k+1}^{(0)}$ shows that 
\beq \label{eq:mtw0bd_app}
\| \E[ \tw^{(0)}_{k+1} ( \tw^{(0)}_{k+1} )^\top ] \| \leq  (G_{0:k}^{(2)})^2 \pw \mtw_0 +    \tC_0 \sum_{j=0}^k \gamma_j^2 (G_{j+1:k}^{(2)})^2 (1 + \mtw_j + \mth_j ) ,
\eeq
where
\beq \label{eq:tC0}
\tC_0 = \pw \Big[\{ \Cinfty^2 \vee 1 + \sqrt{\dw \dth} \Cinfty \} \{ \ttmv+ \kappa^2\ttmwzero \vee \ttmvt+\kappa^2\ttmwzero\} \Big] \vee \big[ \ttmvzero + \kappa^2 \ttmwzero \big] .
\eeq
Our next endeavor is to bound $\E[ \| \tw^{(1)}_{k+1} \|^2 ]$. Evaluating the recursion in \eqref{eq:recur_wt} gives
\beq \label{eq:solved_wt}
\tw^{(1)}_{k+1} = \ProdB_{0:k}^{(2)} \tw^{(1)}_0 + \sum_{j=0}^k \gamma_j \ProdB_{j+1:k}^{(2)} ( W_{j+1}^{(1)} + C_j V_{j+1}^{(1)} )
\eeq
Set $\markovtermt_j^{b_i} := \markovterm_j^{b_i} + \Markovterm_j^{A_{i1}} \theta^\star + \Markovterm_j^{A_{i2}} w^\star$ for $i=1,2$.
Using the definitions, the combined noise has the following expression 
\beq \notag
\begin{split}
& W_{j+1}^{(1)} + C_j V_{j+1}^{(1)} = \big(\markovtwotb_j - \markovtwotb_{j+1}\big) + C_j \big( \markovonetb_j - \markovonetb_{j+1} \big) + \Big\{ \markovtwoB_j - \markovtwoB_{j+1} + C_j \big( \markovoneB_j - \markovoneB_{j+1} \big) \Big\} \tw_j \\
& + \Big\{ \markovtwoA_j - \markovtwoA_{j+1} - (\markovtwoB_j - \markovtwoB_{j+1} ) C_{j-1} + C_j ( \markovoneA_j - \markovoneA_{j+1}) - C_j ( \markovoneB_j - \markovoneB_{j+1}) C_{j-1} \Big\} \ttheta_j 
\end{split}
\eeq
Upon some algebra manipulations that are detailed in Appendix~\ref{subsec:aux_markov}, we deduce that the combined noise may be decomposed as:
\beq \label{eq:WVdecompose}
\begin{split}
W_{j+1}^{(1)} + C_j V_{j+1}^{(1)} & \equiv \markovterm_j^{WV} - \markovterm_{j+1}^{WV}   + \simplewtA_j \ttheta_j + \simplewwA_j \tw_j  \\
& \quad + \big(\simplewtdiff_j \ttheta_j - \simplewtdiff_{j+1} \ttheta_{j+1}\big)  + 
\big(\simplewwdiff_j \tw_j - \simplewwdiff_{j+1} \tw_{j+1}\big) \\
& \quad  + \simplewtiter \big( \ttheta_{j+1} - \ttheta_j \big) + \simplewwiter \big( \tw_{j+1} - \tw_j \big),
\end{split}
\eeq
where it holds that 
\beq \notag
\| \markovterm_j^{WV}\| \vee \| \simplewtdiff_j \| \vee \| \simplewwdiff_j \| \vee \| \simplewwiter \| \vee \|\simplewtiter \| \leq \boundwsim,~~
\| \simplewtA_j \| \vee \| \simplewwA_j \| \leq \boundwsim \gamma_j,
\eeq
with
\beq \label{eq:boundwsim}
\boundwsim := \max\{ \overline{\rm b} (1 + \Cinfty), \overline{\rm A}(1 + 2 \Cinfty + \Cinfty^2), \overline{\rm A} C_2^U \CSEQ{a_{22}/2} ( 1 + \Cinfty ) (1 + \varsigma ) \}.
\eeq
Let us bound the second term in \eqref{eq:solved_wt} one by one as follows.
Using Lemma~\ref{lem:addsubtract}, we obtain
\[
\begin{split}
& \sum_{j=0}^k \gamma_j \ProdB_{j+1:k}^{(2)} \Big( \markovterm_j^{WV} - \markovterm_{j+1}^{WV}   + \big(\simplewtdiff_j \ttheta_j - \simplewtdiff_{j+1} \ttheta_{j+1} \big)  + \big( \simplewwdiff_j \tw_j - \simplewwdiff_{j+1} \tw_{j+1}\big) \Big) \\
& = \gamma_0 \ProdB_{1:k}^{(2)} \big( \markovterm_0^{WV} + \simplewtdiff_0 \ttheta_0 + \simplewwdiff_0 \tw_0 \big)  - \gamma_k \big( \markovterm_{k+1}^{WV} + \simplewtdiff_{k+1} \ttheta_{k+1} + \simplewwdiff_{k+1} \tw_{k+1} \big) \\
& + \sum_{j=1}^k \big( \gamma_j^2 B_{22}^j \ProdB_{j+1:k}^{(2)} + (\gamma_j - \gamma_{j-1}) \ProdB_{j:k}^{(2)} \big) \big( \markovterm_j^{WV} + \simplewtdiff_j \ttheta_j + \simplewwdiff_j \tw_j \big),
\end{split}
\]
Secondly, 
\[ \begin{split}
& \sum_{j=0}^k \gamma_j \ProdB_{j+1:k}^{(2)} \simplewtiter \big( \ttheta_{j+1} - \ttheta_j \big) = -\sum_{j=0}^k \gamma_j \beta_j \ProdB_{j+1:k}^{(2)} \simplewtiter \big( A_{12} \tw_{j+1} + V_{j+1} \big)
\end{split}
\]
\[ \begin{split}
& \sum_{j=0}^k \gamma_j \ProdB_{j+1:k}^{(2)} \simplewwiter \big( \tw_{j+1} - \tw_j \big) = -\sum_{j=0}^k \gamma_j^2 \ProdB_{j+1:k}^{(2)} \simplewwiter \big( W_{j+1} + C_j V_{j+1} \big)
\end{split}
\]
As a consequence of \eqref{eq:noise:constants:0}--\eqref{eq:noise:constantsVW:0}, we have
\beq \notag
\E [\| A_{12} \tw_{j+1} + V_{j+1} \|^2] \leq \ttmtcom \big( 1 + \mtw_j + \mth_j \big),~~
\E [\| W_{j+1} + C_j V_{j+1} \|^2] \leq \ttmwcom \big( 1 + \mtw_j + \mth_j \big)
\eeq
where
\[
  \ttmtcom := 2 \big\{ \|A_{12}\|^2 + \ttmv \big\}, \quad
  \ttmwcom := 2 ( \ttmw + \Cinfty \ttmv ).
\]
Noting that $\tw_0^{(1)} = 0$, taking Euclidean norm on both sides of \eqref{eq:solved_wt} yields
\beq \label{eq:bd1stordermarkov}
\begin{split}
\| \tw_{k+1}^{(1)} \| & \leq \sqrt{\pw} \Big\{ \boundwsim \big[ G_{1:k}^{(2)} \gamma_0  (1 + \| \ttheta_0 \| + \| \tw_0 \| ) + \gamma_k (1 + \| \ttheta_{k+1} \| + \| \tw_{k+1} \| ) \big] \Big\} \\
& \hspace{-.9cm} + \sqrt{\pw}  \boundwsim \Big\{ \sum_{j=1}^k G_{j+1:k}^{(2)} (\gamma_j^2 + \| \gamma_j^2 B_{22}^j + (\gamma_j-\gamma_{j-1})(\Id - B_{22}^j) \| )  (1 + \| \tw_j\| + \| \ttheta_j \|)  \Big\} \\
& \hspace{-.9cm} + \sqrt{\pw} \boundwsim \sum_{j=0}^k \gamma_j^2 G_{j+1:k}^{(2)} ( \kappa \| A_{12} \tw_{j+1} + V_{j+1} \| + \| W_{j+1} + C_j V_{j+1} \|)
\end{split}
\eeq
Note that for any sequence $(b_j)_{j \geq 0}$, the following inequality holds:
\beq \label{eq:jensen}
\Big( \sum_{j=0}^k \gamma_j^2 G_{j+1:k}^{(2)} b_j \Big)^2 \leq \big( \sum_{i=0}^k \gamma_i^2 G_{i+1:k}^{(2)} \big) \sum_{j=0}^k \gamma_j^2 G_{j+1:k}^{(2)} b_j^2
\leq  \gamma_{k+1} \CSEQ{a_{22}/2}  \sum_{j=0}^k \gamma_j^2 G_{j+1:k}^{(2)} b_j^2,
\eeq
where the first inequality is due to Jensen's inequality and the second inequality is due to Corollary~\ref{cor:rate-of-convergence}.
Using $\| B_{22}^j \| \leq \Btwoinfty$, $|\gamma_j - \gamma_{j-1}| \leq \frac{a_{22}}{8} \gamma_j^2$ [cf.~it is a direct consequence of A\ref{assum:stepsize}-2 and the fact $\gamma_j \leq \gamma_{j-1}$] and applying the above inequality to \eqref{eq:bd1stordermarkov} yields
\beq \notag
\begin{split}
\| \tw_{k+1}^{(1)} \|^2 & \leq 9\pw (\boundwsim)^2 \Big\{ (G_{1:k}^{(2)})^2 \gamma_0^2 (1+\| \tw_0\| + \| \ttheta_0 \| )^2 + \gamma_k^2 (1 + \| \ttheta_{k+1} \|^2 + \| \tw_{k+1} \|^2) \Big\}
\\
& + 9\pw (\boundwsim)^2 (\Btwoinfty + \frac{a_{22}}{8} + 1)^2 \CSEQ{a_{22}/2} \gamma_{k+1} \Big\{ \sum_{j=0}^k \gamma_j^2 G_{j+1:k}^{(2)} (1 + \| \tw_j\|^2  + \| \ttheta_j \|^2 ) \Big\} \\
& + 9\pw (\boundwsim)^2  \CSEQ{a_{22}/2} \gamma_{k+1} \sum_{j=0}^k \gamma_j^2 G_{j+1:k}^{(2)} \big( \kappa \| A_{12} \tw_{j+1} + V_{j+1} \|^2 + \| W_{j+1} + C_j V_{j+1} \|^2 \big)
\end{split}
\eeq
Using the fact $\E[ \| \tw_k \|^2 ] \leq \dw \| \E[ \tw_k \tw_k^\top ] \|$, $\E[ \| \ttheta_k \|^2 ] \leq \dth \| \E[ \ttheta_k \ttheta_k^\top ] \|$ (cf.~Corollary~\ref{coro:matrix-trace-normop}), taking the expectation on both sides yields 
\[
\begin{split}
& \E[ \| \tw_{k+1}^{(1)} \|^2] \leq 9\pw (\boundwsim)^2 \Big\{ (G_{1:k}^{(2)})^2 \gamma_0^2 (1+\| \tw_0\| + \| \ttheta_0 \| )^2 + \gamma_k^2 (1 + \dth \mth_{k+1} + \dw \mtw_{k+1} ) \Big\}
\\
& + 9\pw (\boundwsim)^2 (\Btwoinfty + \frac{a_{22}}{8} + 1)^2 \CSEQ{a_{22}/2} \gamma_{k+1}  \Big\{ \sum_{j=0}^k \gamma_j^2 G_{j+1:k}^{(2)} (1 + \dth \mth_j + \dw \mtw_j ) \Big\} \\
& + 9\pw (\boundwsim)^2 ( \kappa \ttmtcom + \ttmwcom) \CSEQ{a_{22}/2} \gamma_{k+1}  \sum_{j=0}^k \gamma_j^2 G_{j+1:k}^{(2)} \big( 1 + \mth_j + \mtw_j \big)
\end{split}
\]
The above simplifies to
\beq \label{eq:mtw1bd_app}
\begin{split}
\E[ \| \tw_{k+1}^{(1)} \|^2] & \leq \tC_1 (G_{0:k}^{(2)})^2  + \tC_2 \gamma_k^2 (\mth_{k+1} + \mtw_{k+1} ) + \tC_3 \gamma_{k+1} \sum_{j=0}^k \gamma_j^2 G_{j+1:k}^{(2)} ( \mth_j + \mtw_j ) + \tC_4 \gamma_{k}^2
\end{split}
\eeq
where we have used $\gamma_{k+1} \leq \gamma_k$ and defined
\beq \label{eq:tC3}
\begin{split}
& \tC_1 = 9 \pw ( \boundwsim)^2 (1 + \| \tw_0 \| + \| \ttheta_0 \|)^2 ( \gamma_0 / (1-\gamma_0 a_{22}/2) )^2,\\
& \tC_2 = 9 \pw ( \boundwsim )^2 (\dth \vee \dw), \\
& \tC_3 = 9 \pw ( \boundwsim)^2 \CSEQ{a_{22}/2} \big[ (\dth \vee \dw) (\Btwoinfty + \frac{a_{22}}{8} + 1)^2 + \big( \kappa \ttmtcom + \ttmwcom\big)  ], \\
& \tC_4 = \tC_2 + \CSEQ{a_{22}/2} \tC_3.
\end{split}
\eeq 
Notice that the intermediate results \eqref{eq:mtw0bd_app}, \eqref{eq:mtw1bd_app} lead to Lemma~\ref{lem:m1interbd}.

Compared to \eqref{eq:mtw0bd_app}, an important feature of the bound \eqref{eq:mtw1bd_app} is that the latter contains an extra $\gamma_k$ factor. This indicates that the iterate $\tw_{k+1}^{(1)}$ driven by Markovian noise decays at a faster rate. As we will demonstrate below, the effect of the additional Markov noise is thus negligible compared to the martingale noise driven terms.

As the operator norm $\| \cdot \|$ is convex, applying Jensen's inequality yields
\[
\mtw_{k+1} \leq 2 \| \E [ \tw_{k+1}^{(1)} (\tw_{k+1}^{(1)})^\top ] \| + 2 \E  [ \tw_{k+1}^{(0)} (\tw_{k+1}^{(0)})^\top ] \| \leq 2 \E[ \| \tw^{(1)}_{k+1} \|^2 ] + 2 \| \E  [ \tw_{k+1}^{(0)} (\tw_{k+1}^{(0)})^\top ] \|
\]
Substituting \eqref{eq:mtw0bd_app} and \eqref{eq:mtw1bd_app} gives
\beq \notag
\begin{split}
\mtw_{k+1} & \leq 2 \Big\{ \tC_1 (G_{0:k}^{(2)})^2 + \tC_2 \gamma_k^2 ( \mtw_{k+1} + \mth_{k+1} ) + \tC_3 \gamma_{k+1} \sum_{j=0}^k \gamma_j^2 G_{j+1:k}^{(2)} (\mtw_j + \mth_j ) + \tC_4 \gamma_k^2 \Big\} \\
& + 2 \Big\{ \pw  ( G_{0:k}^{(2)} )^2 \mtw_0 + \tC_0 \CSEQ{a_{22}/2} \gamma_{k+1} + \tC_0 \sum_{j=0}^k \gamma_j^2 G_{j+1:k}^{(2)} (\mtw_j + \mth_j ) \Big\}
\end{split}
\eeq
The assumption on step size in \eqref{eq:gamma_markov} guarantees $2 \tC_2 \gamma_k^2 \leq (1/2)$, which further implies
\beq \notag
\begin{split}
\mtw_{k+1} & \leq 4 \Big\{ \tC_1 (G_{0:k}^{(2)})^2 + \tC_2 \gamma_k^2 \mth_{k+1} + \tC_3 \gamma_{k+1} \sum_{j=0}^k \gamma_j^2 G_{j+1:k}^{(2)} (\mtw_j + \mth_j ) + \tC_4 \gamma_k^2 \Big\} \\
& + 4 \Big\{ \pw  ( G_{0:k}^{(2)} )^2 \mtw_0 + \tC_0 \CSEQ{a_{22}/2} \gamma_{k+1} + \tC_0 \sum_{j=0}^k \gamma_j^2 G_{j+1:k}^{(2)} (\mtw_j + \mth_j ) \Big\}
\end{split}
\eeq
Like in the proof of Theorem~\ref{theo:preliminary-bound-martingale}, we set
\beq \notag
\opu_{k+1} =  G_{0:k}^{(2)} (\tC_1 + \pw \mtw_0) + \tC_0 \CSEQ{a_{22}/2} \gamma_{k+1} + \sum_{j=0}^k \gamma_j^2 G_{j+1:k}^{(2)} \big(\tC_3 \gamma_{j}  + \tC_0 \big) (\opu_j + \mth_j ) 
\eeq
with $\opu_0 = \tC_1 + \pw \mtw_0$. Through evaluating the recursion, we observe that for any $k \geq 0$, it holds
\beq \label{eq:mkopu_relate}
\begin{split}
\mtw_{k+1} & \leq 4 \Big\{ \opu_{k+1} +  \sum_{j=1}^{k+1} \gamma_{j-1}^2 G_{j:k}^{(2)} \big( \tC_2 \mth_{j} + \tC_4 ) \Big\} \\
& \leq 4 \Big\{ \opu_{k+1} + \gamma_k^2 (\tC_2 \mth_{k+1} + \tC_4) +  \sum_{j=1}^k \gamma_{j}^2 G_{j+1:k}^{(2)} \big( \tC_2 \mth_{j} + \tC_4 ) \Big\}
\end{split}
\eeq 
where the last inequality is due to A\ref{assum:stepsize}-2 which guarantees that $\gamma_{j-1}^2 G_{j:k}^{(2)} \leq \gamma_j^2 G_{j+1:k}$.
Moreover, the sequence $\opu_{k+1}$ can be expressed as follows:
\beq \notag
\begin{split}
\opu_{k+1} - (1 - \gamma_k a_{22}/2) \opu_k & = \tC_0 \CSEQ{a_{22}/2} ( \gamma_{k+1} - \gamma_k ( 1 - \gamma_k a_{22}/2)) + \gamma_k^2 \big(\tC_3 \gamma_k  + \tC_0 \big) (\opu_k + \mth_k) \\
& \leq \tC_0 \CSEQ{a_{22}/2} (a_{22}/2) \gamma_k^2 + \gamma_k^2 \big(\tC_3 \gamma_k  + \tC_0 \big) (\opu_k + \mth_k)
\end{split}
\eeq
As the step size satisfies $\gamma_k \big(\tC_3 \gamma_0  + \tC_0 \big) \leq \frac{a_{22}}{4}$,
we get
\beq
\begin{split}
& \opu_{k+1} \leq (1 - \gamma_k a_{22}/4) \opu_k + \gamma_k^2 \big(\tC_3 \gamma_0  + \tC_0 \big) \mth_k + \tC_0 \CSEQ{a_{22}/2} (a_{22}/2) \gamma_k^2 \\
& \Longrightarrow \opu_{k+1} \leq \widetilde{G}_{0:k}^{(2)} \opu_0 + \sum_{j=0}^k \gamma_j^2 \widetilde{G}_{j+1:k}^{(2)} \Big\{ \big(\tC_3 \gamma_0  + \tC_0 \big) \mth_j + (\tC_0 \CSEQ{a_{22}/2} (a_{22}/2)) \Big\}.
\end{split}
\eeq
Substituting the above into \eqref{eq:mkopu_relate} yields
\beq
\begin{split}
\mtw_{k+1} & \leq 4 \Big\{ \widetilde{G}_{0:k}^{(2)} \opu_0 + \sum_{j=0}^k \gamma_j^2 \widetilde{G}_{j+1:k}^{(2)} \big\{ \big(\tC_3 \gamma_0  + \tC_0 \big) \mth_j + (\tC_0 \CSEQ{a_{22}/2} (a_{22}/2)) \big\} \Big\} \\
& + 4 \Big\{ \gamma_k^2 (\tC_2 \mth_{k+1} + \tC_4) +  \sum_{j=1}^k \gamma_{j}^2 G_{j+1:k}^{(2)} \big( \tC_2 \mth_{j} + \tC_4 ) \Big\}
\end{split}
\eeq
Finally, using the fact that $\gamma_k^2 \leq \varsigma \gamma_{k+1}$ yields
\beq \label{eq:mwmarkovbd_app}
\boxed{\mtw_{k+1} \leq \widetilde{G}_{0:k}^{(2)} \Czerowtilde + \Conewtilde \gamma_{k+1} + \Ctwowtilde \sum_{j=0}^k \gamma_j^2 \widetilde{G}_{j+1:k}^{(2)} \mth_j +  \Cthreewtilde \gamma_k^2  \mth_{k+1}.}
\eeq
where
\beq \label{eq:mwmarkovbd_const}
\Czerowtilde := 4 (\tC_1 + \pw \mtw_0),~~\Conewtilde := 4 (\tC_4 ( \varsigma + \CSEQ{a_{22}/2}) + \tC_0 (\CSEQ{a_{22}/2})^2 (a_{22}/2))
\eeq
\[
\Ctwowtilde := 4 (\tC_3 \gamma_0 + \tC_2 + \tC_0),~~\Cthreewtilde := 4 \tC_2.
\]
This concludes the proof for Proposition~\ref{prop:mw_markov}.

Before we proceed, we need to bound $\| \E [ \tw^{(0)}_{k+1} (\tw^{(0)}_{k+1})^\top ]$ and $\E[ \| \tw^{(1)}_{k+1} \|^2 ]$ as well. 
Substituting \eqref{eq:mwmarkovbd_app} into \eqref{eq:mtw0bd_app} yields
\beq \label{eq:m0bound_first}
\begin{split}
& \| \E [ \tw^{(0)}_{k+1} (\tw^{(0)}_{k+1})^\top ] \| \leq  (G_{0:k}^{(2)})^2 \pw \mtw_0 + \tC_0 \sum_{j=0}^k \gamma_j^2 (G_{j+1:k}^{(2)})^2 (1+ \mth_j) \\
& \hspace{.5cm} + \tC_0 \sum_{j=0}^k \gamma_j^2 (G_{j+1:k}^{(2)})^2 \Big( \Czerowtilde \widetilde{G}_{0:j-1}^{(2)}   + \Conewtilde \gamma_j + \Ctwowtilde \sum_{i=0}^{j-1} \gamma_i^2 \widetilde{G}_{i+1:j-1}^{(2)} \mth_i + \Cthreewtilde \gamma_{j-1}^2 \mth_j\Big)  
\end{split}
\eeq
We observe 
\[
\begin{split}
& \sum_{j=0}^k \gamma_j^2 (G_{j+1:k}^{(2)})^2 \sum_{i=0}^{j-1} \gamma_i^2 \widetilde{G}_{i+1:j-1}^{(2)} \mth_i = \sum_{i=0}^{k-1} \gamma_i^2 \mth_i \sum_{j=i+1}^k \gamma_j^2  (G_{j+1:k}^{(2)})^2  \widetilde{G}_{i+1:j-1}^{(2)} \\
& \overset{(a)}{\leq} \frac{1}{1 - \gamma_0 a_{22}/4} \sum_{i=0}^{k-1} \gamma_i^2 \mth_i \widetilde{G}_{i+1:k}^{(2)} \sum_{j=i+1}^k \gamma_j^2 G_{j+1:k}^{(2)}
\overset{(b)}{\leq} \frac{\CSEQ{a_{22}/2} \gamma_{k+1} }{1 - \gamma_0 a_{22}/4} \sum_{i=0}^{k-1} \gamma_i^2 \mth_i \widetilde{G}_{i+1:k}^{(2)}
\end{split}
\]
where (a) is due to $G_{j+1:k}^{(2)} \leq \widetilde{G}_{j+1:k}^{(2)}$ and (b) is due to Corollary~\ref{cor:rate-of-convergence}. As such, combining terms in \eqref{eq:m0bound_first} yields:
\beq \label{eq:mw0markovbd}
\boxed{\| \E [ \tw^{(0)}_{k+1} (\tw^{(0)}_{k+1})^\top ] \| \leq \Czerowtildep \widetilde{G}_{0:k}^{(2)} + \Conewtildep \gamma_{k+1} + \Ctwowtildep \sum_{j=0}^k \gamma_j^2 \widetilde{G}_{j+1:k}^{(2)} \mth_j ,}
\eeq
where
\beq \label{eq:mw0markovbd_const}
\Czerowtildep := \pw \mth_0, \quad \Conewtildep := \tC_0 \CSEQ{a_{22}/2} \big(1 + \Czerowtilde + \Conewtilde \big)
\eeq
\[
\Ctwowtildep := \tC_0 \Big( 1 + \Cthreewtilde + \Ctwowtilde \CSEQ{a_{22}/2} \frac{ \gamma_0 }{1 - \gamma_0 a_{22}/4} \Big)
\]
Similarly, we can compute the bound for $\E[ \| \tw_{k+1}^{(1)} \|^2]$ as follows. Using \eqref{eq:mtw1bd_app}:
\beq \label{eq:mtw1bd_inter}
\begin{split}
\E[ \| \tw_{k+1}^{(1)} \|^2] & \leq \tC_1 (G_{0:k}^{(2)})^2 + \tC_2 \gamma_k^2 \mth_{k+1} + \tC_3 \gamma_{k+1} \sum_{j=0}^k \gamma_j^2 G_{j+1:k}^{(2)} \mth_j  \\
& + \tC_2 \gamma_k^2 \mtw_{k+1} + \tC_3 \gamma_{k+1} \sum_{j=0}^k \gamma_j^2 G_{j+1:k}^{(2)} \mtw_j 
\end{split}
\eeq
Notice that 
\beq \notag
\begin{split}
& \sum_{j=0}^k \gamma_j^2 G_{j+1:k}^{(2)} \mtw_j \leq \sum_{j=0}^k \gamma_j^2 G_{j+1:k}^{(2)} \Big( \Czerowtilde \widetilde{G}_{0:j-1}^{(2)} + \Conewtilde \gamma_j + \Ctwowtilde \sum_{i=0}^{j-1} \gamma_i^2 \widetilde{G}_{i+1:j-1}^{(2)} \mth_i + \Cthreewtilde \gamma_{j-1}^2 \mth_j  \Big) \\
& \leq \CSEQ{a_{22}/2} (\Czerowtilde + \Conewtilde \gamma_0 ) \gamma_{k+1} + \Cthreewtilde \sum_{j=0}^k \gamma_j^2 G_{j+1:k}^{(2)} \mth_j + \Ctwowtilde \sum_{i=0}^{k-1} \gamma_i^2 \mth_i \sum_{j=i+1}^k \gamma_j^2 G_{j+1:k}^{(2)} \widetilde{G}_{i+1:j-1}^{(2)} 
\end{split}
\eeq
Since $(1 - \gamma a_{22}/2) \leq (1 -\gamma a_{22}/4)^2$ for any $\gamma >0$, we have $G_{j+1:k}^{(2)} \leq (\widetilde{G}_{j+1:k}^{(2)})^2$, therefore together with Corollary~\ref{cor:rate-of-convergence} it yields
\beq
\sum_{i=0}^{k-1} \gamma_i^2 \mth_i \sum_{j=i+1}^k \gamma_j^2 G_{j+1:k}^{(2)} \widetilde{G}_{i+1:j-1}^{(2)} \leq \frac{\CSEQ{a_{22}/4} \gamma_{k+1}}{1 - \gamma_0 a_{22}/4} \sum_{i=0}^{k-1} \gamma_i^2 \mth_i \widetilde{G}_{i+1:k}^{(2)}.
\eeq
Collecting  terms and substituting them in \eqref{eq:mtw1bd_inter} yield
\beq \label{eq:mtw1bd_fin}
\boxed{\E[ \| \tw_{k+1}^{(1)} \|^2] \leq \Czerowtildeh \widetilde{G}_{0:k}^{(2)} + \Conewtildeh \gamma_{k+1}^2 + \Ctwowtildeh \gamma_{k+1} \sum_{j=0}^k \gamma_j^2 \widetilde{G}_{j+1:k}^{(2)} \mth_j + \Cthreewtildeh \gamma_k^2  \mth_{k+1},}
\eeq
where we use again the fact that $\gamma_k^2 \leq \varsigma \gamma_{k+1}$ and
\[
\Czerowtildeh := \tC_1 + \gamma_0^2 \tC_2 \Czerowtilde , ~~ \Conewtildeh := \tC_3 \CSEQ{a_{22}/2} ( \Czerowtilde + \Conewtilde \gamma_0 ) + \varsigma \tC_2 \Conewtilde \]
\[
\Ctwowtildeh := \tC_3 \Big( 1 + \Cthreewtilde + \Ctwowtilde\frac{ \CSEQ{a_{22}/4} \gamma_0 }{1 - \gamma_0 a_{22}/4} \Big) + \varsigma \tC_2 \Ctwowtilde, ~~ \Cthreewtildeh := \tC_2 (1 + \gamma_0^2 \Cthreewtilde ).
\]

\paragraph{Bounding the Cross Term (Proof of Lemma~\ref{lem:mtw0crossbd})}
Our next endeavor is to bound the cross variance between the \emph{martingale noise} driven terms $\tw^{(0)}_{k+1}$ and $\ttheta^{(0)}_{k+1}$. Here, the steps involved are similar to those in bounding $\mthw_k$ in the proof of Theorem~\ref{theo:preliminary-bound-martingale}. Particularly, in a similar vein as the derivation of \eqref{eq:mtwbound_immediate}, we obtain
\begin{align}
& \| \E[ \ttheta^{(0)}_{k+1} (\tw^{(0)}_{k+1})^\top ] \| \leq \pwth \Big\{ G_{0:k}^{(2)} \mthw_0 +  \normop{A_{12}}[] \sum_{j=0}^{k}\beta_j G_{j+1:k}^{(1)} G_{j:k}^{(2)}  \| \E[ \tw^{(0)}_{j} (\tw^{(0)}_{j})^\top ] \| \Big\} \notag \\
& + \pwth \Big\{ \sum_{j=0}^{k} \beta_j \gamma_j G_{j+1:k}^{(1)}G_{j+1:k}^{(2)} \normop{\E [ V_{j+1}^{(0)} (W_{j+1}^{(0)})^\top]} + \Cinfty \sum_{j=0}^{k}\beta_j^2 G_{j+1:k}^{(1)} G_{j+1:k}^{(2)} \normop{\E[ V_{j+1}^{(0)} (V_{j+1}^{(0)})^\top]}[] \Big\} \notag
\end{align}
By observing that $G_{j:k}^{(2)} \leq \widetilde{G}_{j:k}^{(2)}$, we have 
\beq \label{eq:mthwbd_markov_inter}
\begin{split}
\| \E[ \ttheta^{(0)}_{k+1} (\tw^{(0)}_{k+1})^\top ] \| & \leq \pwth \widetilde{G}_{0:k}^{(2)}  \mthw_0 + \pwth \| A_{12} \| \sum_{j=0}^k \beta_j G_{j+1:k}^{(1)} G_{j:k}^{(2)} \| \E[ \tw^{(0)}_{j} (\tw^{(0)}_{j})^\top ] \| \\
& + \pwth \sum_{j=0}^k \beta_j G_{j+1:k}^{(1)}G_{j+1:k}^{(2)} \big( \ttmVWzero \gamma_j + \ttmvzero \Cinfty \beta_j \big) \big(1 + \mth_j + \mtw_j \big) \\
\end{split}
\eeq
When combined with \eqref{eq:mwmarkovbd_app}, \eqref{eq:mw0markovbd}, it can be verified using similar steps as in deriving \eqref{eq:w_term} that:
\[
\sum_{j=0}^k \beta_j G_{j+1:k}^{(1)} G_{j:k}^{(2)} \| \E[ \tw^{(0)}_{j} (\tw^{(0)}_{j})^\top ] \|
\leq \frac{2 \Czerowtildep \widetilde{G}_{0:k}^{(2)}}{a_\Delta} + \Conewtildep \CSEQ{a_{22}/4} \beta_{k+1} + \frac{2 \Ctwowtildep}{a_\Delta} \sum_{i=0}^k \gamma_i^2 \widetilde{G}_{i+1:k}^{(2)} \mth_i,
\]
\[
\sum_{j=0}^k \beta_j G_{j+1:k}^{(1)} G_{j:k}^{(2)} \mtw_j 
\leq \frac{2 \Czerowtilde \widetilde{G}_{0:k}^{(2)}}{a_\Delta} + \Conewtilde \CSEQ{a_{22}/4} \beta_{k+1} + \Big( \frac{2 \Ctwowtilde}{a_\Delta} + \Cthreewtilde \Big) \sum_{i=0}^k \gamma_i^2 \widetilde{G}_{i+1:k}^{(2)} \mth_i,
\]
Substituting the above into \eqref{eq:mthwbd_markov_inter} gives:
\beq \label{eq:mtwmarkovbd}
\boxed{\| \E [ \ttheta^{(0)}_{k+1} (\tw^{(0)}_{k+1})^\top ] \| \leq \Czerotwtilde \widetilde{G}_{0:k}^{(2)} + \Conetwtilde \beta_{k+1} + \Ctwotwtilde \sum_{j=0}^k \gamma_j^2 \widetilde{G}_{j+1:k}^{(2)} \mth_j,}
\eeq
where
\beq \label{eq:mtwmarkovbd_const}
\begin{split}
& \Czerotwtilde := \pwth \Big( \mthw_0 + \frac{ \Czerowtildep \| A_{12} \| + \Czerowtilde \big( \ttmVWzero \gamma_0 + \ttmvzero \Cinfty \beta_0 \big) }{a_\Delta/2} \Big) \\
& \Conetwtilde := \pwth \CSEQ{a_{22}/4}  \Big( \Conewtildep \| A_{12} \| + \Conewtilde \big( \ttmVWzero \gamma_0 + \ttmvzero \Cinfty \beta_0 \big) \Big) \\
& \Ctwotwtilde := \pwth \Big\{ \frac{2 \Ctwowtildep}{a_\Delta} \| A_{12} \| + \Big( \frac{2 \Ctwowtilde}{a_\Delta} + \Cthreewtilde \Big) \big( \ttmVWzero \gamma_0 + \ttmvzero \Cinfty \beta_0 \big) \Big\}
\end{split}
\eeq
Notice that this concludes the proof of Lemma~\ref{lem:mtw0crossbd}.

\paragraph{Bounding $\mth_k$ (Proof of  Proposition~\ref{prop:mth_markov})} Like in the proof of Theorem~\ref{theo:preliminary-bound-martingale}, we begin by bounding $\| \E [ \tw^{(0)}_k ( \tw^{(0)}_k )^\top ] \|$ as follows. Evaluating the recursion in \eqref{eq:recur_wt} and following the derivations that lead to \eqref{eq:mth_immediate}, we obtain
\beq \notag
\begin{split}
\| \E \ttheta^{(0)}_{k+1} ( \ttheta^{(0)}_{k+1} )^\top ] \| & \leq \pth \Big\{ (G_{0:k}^{(1)})^2 \mth_0 + 2 \|A_{12}\| \sum_{j=0}^k \beta_j G_{j+1:k}^{(1)} G_{j:k}^{(1)} \| \E [ \ttheta^{(0)}_j (\tw^{(0)}_j)^\top ] \|  \Big\}  \\
& + \pth \Big\{ \sum_{j=0}^k \beta_j^2 (G_{j+1:k}^{(1)})^2 \big( \|A_{12}\|^2 \| \E [ \tw^{(0)}_j (\tw^{(0)}_j)^\top ] \| + \| \E [ V_{j+1}^{(0)} (V_{j+1}^{(0)})^\top ] \| \big) \Big\} 
\end{split}
\eeq
We apply \eqref{eq:mtwmarkovbd} and note that
\[
\begin{split}
& \sum_{j=0}^k \beta_j G_{j+1:k}^{(1)} G_{j:k}^{(1)} \| \E [ \ttheta^{(0)}_j (\tw^{(0)}_j)^\top ] \| \\
& \leq \sum_{j=0}^k \beta_j G_{j+1:k}^{(1)} G_{j:k}^{(1)} \Big( \Czerotwtilde \widetilde{G}_{0:j-1}^{(2)} + \Conetwtilde \beta_j + \Ctwotwtilde \sum_{i=0}^{j-1} \gamma_i^2 \widetilde{G}_{i+1:j-1}^{(2)} \mth_i \Big) \\
& \overset{(a)}{\leq} \Czerotwtilde \frac{ G_{0:k}^{(1)}}{a_\Delta/2} + \Conetwtilde \CSEQ{a_\Delta/2} \beta_{k+1} + \Ctwotwtilde \sum_{j=0}^k \beta_j G_{j+1:k}^{(1)} G_{j:k}^{(1)} \sum_{i=0}^{j-1} \gamma_i^2 \widetilde{G}_{i+1:j-1}^{(2)} \mth_i
\end{split} 
\]
where (a) is due to the observation that $1-\gamma_j a_{22}/4 \leq 1 - \beta_j a_\Delta/2$ and the application of Lemma~\ref{lem:bsum}. Moreover, by a slight modification of \eqref{eq:mthbound2}, we have 
\beq \label{eq:mthbound_mod}
\sum_{j=0}^k \beta_j G_{j+1:k}^{(1)} G_{j:k}^{(1)} \sum_{i=0}^{j-1} \gamma_i^2 \widetilde{G}_{i+1:j-1}^{(2)} \mth_i \leq \frac{16 \varsigma}{a_{22}} \sum_{i=0}^{k-1} \beta_i \gamma_i G_{i+1:k}^{(1)} \mth_i
\eeq
Therefore, 
\beq
\begin{split}
& \sum_{j=0}^k \beta_j G_{j+1:k}^{(1)} G_{j:k}^{(1)} \| \E [ \ttheta^{(0)}_j (\tw^{(0)}_j)^\top ] \| \\
& \leq \Czerotwtilde \frac{ G_{0:k}^{(1)}}{a_\Delta/2} + \Conetwtilde \CSEQ{a_\Delta/2} \beta_{k+1} + \Ctwotwtilde \frac{16 \varsigma}{a_{22}} \sum_{i=0}^{k-1} \beta_i \gamma_i G_{i+1:k}^{(1)} \mth_i
\end{split} 
\eeq
Similarly,  we apply \eqref{eq:mw0markovbd}, \eqref{eq:mthbound_mod} and note that
\beq \label{eq:mthbd_w0}
\begin{split}
& \sum_{j=0}^k \beta_j^2 (G_{j+1:k}^{(1)})^2 \| \E [ \tw^{(0)}_j (\tw^{(0)}_j)^\top ] \| \\
& \leq \sum_{j=0}^k \beta_j^2 (G_{j+1:k}^{(1)})^2 \Big\{ \Czerowtildep \widetilde{G}_{0:j-1}^{(2)} + \Conewtildep \gamma_j + \Ctwowtildep \sum_{i=0}^{j-1} \gamma_i^2 \widetilde{G}_{i+1:j-1}^{(2)} \mth_i \Big\} \\
& \leq ( \Czerowtildep + \Conewtildep \gamma_0 ) \CSEQ{a_\Delta/2} \beta_{k+1} + \Ctwowtildep \frac{\beta_0 (16 \varsigma/a_{22})}{1-\beta_0 a_\Delta/2} \sum_{i=0}^{k-1} \beta_i \gamma_i G_{i+1:k}^{(1)} \mth_i
\end{split}
\eeq
Finally,  we obtain that 
\beq
\begin{split}
& \sum_{j=0}^k \beta_j^2 (G_{j+1:k}^{(1)})^2  \| \E [ V_{j+1}^{(0)} (V_{j+1}^{(0)})^\top ] \|
\leq \ttmvzero \sum_{j=0}^k \beta_j^2 (G_{j+1:k}^{(1)})^2 \big(1 + \mth_j + \mtw_j \big) \\
& \leq \ttmvzero \Big\{ \CSEQ{a_\Delta/2} \beta_{k+1} + \sum_{j=0}^k \beta_j^2 G_{j+1:k}^{(1)} \mth_j + \sum_{j=0}^k \beta_j^2 (G_{j+1:k}^{(1)} )^2 \mtw_j \Big\}
\end{split}
\eeq
Using the bound in \eqref{eq:mwmarkovbd_app} and the derivations in \eqref{eq:mthbd_w0}, we have
\[
\begin{split}
\sum_{j=0}^k \beta_j^2 (G_{j+1:k}^{(1)} )^2 \mtw_j & \leq ( \Czerowtilde + \Conewtilde \gamma_0 ) \CSEQ{a_\Delta/2} \beta_{k+1} + \Ctwowtilde \frac{\beta_0 (16 \varsigma / a_{22})}{1-\beta_0 a_\Delta/2} \sum_{i=0}^{k-1} \beta_i \gamma_i G_{i+1:k}^{(1)} \mth_i \\
& + \Cthreewtilde \gamma_0\beta_0 \sum_{i=0}^k \beta_i \gamma_i G_{i+1:k}^{(1)} \mth_i
\end{split}
\]
Combining the above results, we obtain that
\beq \label{eq:mth0bd_fin}
\| \E [ \ttheta^{(0)}_{k+1} ( \ttheta^{(0)}_{k+1} )^\top ] \| \leq \tC_0^{(0)} G_{0:k}^{(1)} + \tC_1^{(0)} \beta_{k+1} + \tC_2^{(0)} \sum_{j=0}^k \beta_j \gamma_j G_{j+1:k}^{(1)} \mth_j ,
\eeq
where 
\beq \label{eq:mth0bd_const}
\tC_0^{(0)} = \pth \Big( \mth_0 + \Czerotwtilde \frac{ 4 }{a_\Delta/4} \| A_{12} \| \Big),
\eeq
\[
\tC_1^{(0)} = \pth \CSEQ{a_\Delta/2} \Big( 2\|A_{12}\| \Conetwtilde + \| A_{12}\|^2 (\Czerowtildep + \Conewtildep \gamma_0) + \ttmvzero ( \Czerowtilde + \Conewtilde \gamma_0) \Big),
\]
\[
\tC_2^{(0)} = \pth \Big\{ 2 \|A_{12}\| \frac{ \Ctwotwtilde }{a_{22}} + \| A_{12} \|^2 \frac{ \beta_0 (16 \varsigma / a_{22})}{1 - \beta_0 a_\Delta/2} + \ttmvzero \Big( \Ctwowtilde \frac{\beta_0 (16 \varsigma / a_{22})}{1-\beta_0 a_\Delta/2} + \Cthreewtilde \gamma_0 \beta_0 \Big) \Big\}
\]
To bound the term $\E[ \| \ttheta_{k+1}^{(1)} \|^2 ]$, we proceed by considering the following decomposition:
\beq
\ttheta^{(1)}_{k+1} = \underbrace{\ProdB_{0:k}^{(1)} \ttheta^{(1)}_0 + \sum_{j=0}^k \beta_j \ProdB_{j+1:k}^{(1)} A_{12} \tw^{(1)}_j}_{= \ttheta^{(1,0)}_{k+1}} + \underbrace{\sum_{j=0}^k \beta_j \ProdB_{j+1:k}^{(1)} V_{j+1}^{(1)}}_{= \ttheta^{(1,1)}_{k+1}}
\eeq
As $\ttheta^{(1)}_0 = 0$, we observe that
\beq
\| \ttheta^{(1,0)}_{k+1} \| \leq \sqrt{\pth} \| A_{12} \| \sum_{j=0}^k \beta_j G_{j+1:k}^{(1)} \| \tw_j^{(1)} \| 
\eeq
Taking square on both sides and applying the Jensen's inequality \eqref{eq:jensen} yields
\beq
\begin{split}
& \E[ \| \ttheta^{(1,0)}_{k+1} \|^2 ] \leq \frac{\pth \|A_{12}\|^2}{a_\Delta/2} \sum_{j=0}^k \beta_j G_{j+1:k}^{(1)} \E [\| \tw_j^{(1)} \|^2] 
\end{split}
\eeq
Applying \eqref{eq:mtw1bd_fin} gives
\beq \label{eq:mth10bd}
\begin{split}
\E[ \| \ttheta^{(1,0)}_{k+1} \|^2 ] & \leq \frac{\pth \|A_{12}\|^2}{a_\Delta/2} \sum_{j=0}^k \beta_j G_{j+1:k}^{(1)} \Big( \Czerowtildeh \widetilde{G}_{0:j-1}^{(2)} + \Cthreewtildeh \gamma_{j-1}^2 \mth_j + \Conewtildeh \gamma_j^2 \Big) \\
& + \frac{\pth \|A_{12}\|^2}{a_\Delta/2} \Ctwowtildeh \sum_{j=0}^k \beta_j \gamma_j G_{j+1:k}^{(1)} \sum_{i=0}^{j-1} \gamma_i^2 \widetilde{G}_{i+1:j-1}^{(2)} \mth_i 
\end{split}
\eeq
Let us bound the right hand side one by one, we observe 
\[
\sum_{j=0}^k \beta_j G_{j+1:k}^{(1)} \widetilde{G}_{0:j-1}^{(2)} \leq \sum_{j=0}^k \beta_j (\widetilde{G}_{j+1:k}^{(1)})^2 \widetilde{G}_{0:j-1}^{(2)} \leq \frac{G_{0:k}^{(1)}}{1 - \beta_0 a_\Delta/2} \sum_{j=0}^k \beta_j \widetilde{G}_{j+1:k}^{(1)} \leq \frac{(4/a_\Delta) G_{0:k}^{(1)}}{1 - \beta_0 a_\Delta/2} 
\]
\[
\sum_{j=0}^k \beta_j G_{j+1:k}^{(1)} \gamma_{j-1}^2 \mth_j 
\leq \rho_0 \sum_{j=0}^k \beta_j^2 G_{j+1:k}^{(1)} \mth_j,~~
\sum_{j=0}^k \beta_j G_{j+1:k}^{(1)} \gamma_{j}^2 \leq \rho_0 \CSEQ{a_\Delta/2} \beta_{k+1} 
\]
where the last two inequalities are due to $\gamma_{j-1}^2 \leq \rho_0 \beta_j$, see B\ref{assb:step}. In addition, using the fact $G_{m:n}^{(1)} \leq (\widetilde{G}_{m:n}^{(1)})^2$, we have
\[
\begin{split}
& \sum_{j=0}^k \beta_j \gamma_j G_{j+1:k}^{(1)} \sum_{i=0}^{j-1} \gamma_i^2 \widetilde{G}_{i+1:j-1}^{(2)} \mth_i = \sum_{i=0}^{k-1} \gamma_i^2 \mth_i \sum_{j=i+1}^k \beta_j \gamma_j G_{j+1:k}^{(1)} \widetilde{G}_{i+1:j-1}^{(2)} \\
& \leq \rho_0 \sum_{i=0}^{k-1} \beta_i \mth_i \sum_{j=i+1}^k \beta_j \gamma_j  ( \widetilde{G}_{j+1:k}^{(1)} )^2 \widetilde{G}_{i+1:j-1}^{(2)} \leq \frac{\rho_0}{1-\beta_0 a_\Delta/4} 
\sum_{i=0}^{k-1} \beta_i \mth_i \widetilde{G}_{i+1:k}^{(1)} \sum_{j=i+1}^k \beta_j \gamma_j  \widetilde{G}_{j+1:k}^{(1)} \\
& \leq \frac{\CSEQ{a_\Delta/4} \rho_0}{1-\beta_0 a_\Delta/4} 
\sum_{i=0}^{k-1} \beta_i^2  \mth_i \widetilde{G}_{i+1:k}^{(1)}
\end{split}
\]
Substituting these back into \eqref{eq:mth10bd} yields
\beq \label{eq:mth10bd_fin}
\begin{split}
& \E[ \| \ttheta^{(1,0)}_{k+1} \|^2 ] \leq \tC_0^{(1,0)} \widetilde{G}_{0:k}^{(1)} + \tC_1^{(1,0)} \beta_{k+1} + \tC_2^{(1,0)} \sum_{i=0}^{k} \beta_i^2  \mth_i \widetilde{G}_{i+1:k}^{(1)} ,
\end{split}
\eeq
where
\[
\tC_0^{(1,0)} = \Czerowtildeh \frac{8 \pth \|A_{12}\|^2}{a_\Delta^2 (1-\beta_0 a_\Delta/2)}, ~~\tC_1^{(1,0)} = \Conewtildeh \frac{\pth \|A_{12}\|^2}{a_\Delta/2} \rho_0 \CSEQ{a_\Delta/2},
\]
\[
\tC_2^{(1,0)} = \frac{\pth \|A_{12}\|^2}{a_\Delta/2} \rho_0 \Big( \Cthreewtildeh  + \Ctwowtildeh \frac{\CSEQ{a_\Delta/4} }{1-\beta_0 a_\Delta/4}  \Big).
\]
Next, we bound $\E[ \| \ttheta^{(1,1)}_{k}\|^2 ]$. Set $\markovtermt_j^{b_1} := \markovterm_j^{b_1} + \Markovterm_j^{A_{11}} \theta^\star + \Markovterm_j^{A_{12}} w^\star$, upon some algebraic manipulations (details in Appendix~\ref{subsec:aux_markov}) we observe the following decomposition
\beq \label{eq:Vdecompose}
\begin{split}
V_{j+1}^{(1)} & \equiv \markovtermt_j^{b_1}  - \markovtermt_{j+1}^{b_1}    + \big( \markovoneA_j \ttheta_j - \markovoneA_{j+1} \ttheta_{j+1}\big)  + 
\big(\markovoneB_j \tw_j - \markovoneB_{j+1} \tw_{j+1}\big) \\
& \quad  + \markovoneA_j \big( \ttheta_{j+1} - \ttheta_j \big) + \markovoneB_j \big( \tw_{j+1} - \tw_j \big) ,
\end{split}
\eeq
and from B\ref{assb:bdd} we have 
\beq
\| \markovtermt_j^{b_1} \| \vee \| \markovoneA_j \| \vee \| \markovoneB_j \| \leq \boundtsim := \overline{\rm A} \vee (\overline{\rm b} + \overline{\rm A}( \| \theta^\star \| + \| w^\star \|)).
\eeq
Applying Lemma~\ref{lem:addsubtract}, we can show
\beq
\begin{split}
& \sum_{j=0}^k \beta_j \ProdB_{j+1:k}^{(1)} \Big( \markovtermt_j^{b_1} - \markovtermt_{j+1}^{b_1}   + \big( \markovoneA_j \ttheta_j - \markovoneA_{j+1} \ttheta_{j+1}\big)  + 
\big(\markovoneB_j \tw_j - \markovoneB_{j+1} \tw_{j+1}\big) \Big) \\
& = \beta_0 \ProdB_{1:k}^{(1)} \big( \markovtermt_0^{b_1} + \markovoneA_0 \ttheta_0 + \markovoneB_0 \tw_0 \big)  - \beta_k \big( \markovtermt_{k+1}^{b_1} + \markovoneA_{k+1} \ttheta_{k+1} + \markovoneB_{k+1} \tw_{k+1} \big) \\
& \quad + \sum_{j=1}^k \big( \beta_j^2 B_{11}^k \ProdB_{j+1:k}^{(1)} + (\beta_j - \beta_{j-1}) \ProdB_{j:k}^{(1)} \big) \big( \markovtermt_j^{b_1} + \markovoneA_j \ttheta_j + \markovoneB_j \tw_j \big).
\end{split}
\eeq
Moreover, 
\beq
\sum_{j=0}^k \beta_j \ProdB_{j+1:k}^{(1)} \markovoneA_j \big( \ttheta_{j+1} - \ttheta_j \big) = -
\sum_{j=0}^k \beta_j^2 \ProdB_{j+1:k}^{(1)} \markovoneA_j ( A_{12} \tw_j + W_{j+1} )
\eeq
\beq
\sum_{j=0}^k \beta_j \ProdB_{j+1:k}^{(1)} \markovoneB_j \big( \tw_{j+1} - \tw_j \big) = -
\sum_{j=0}^k \beta_j \gamma_j \ProdB_{j+1:k}^{(1)} \markovoneB_j ( W_{j+1} + C_j V_{j+1} )
\eeq
The above inequalities allow us to upper bound $\| \ttheta^{(1,1)}_{k+1} \|$. Note that as $|\beta_j - \beta_{j-1}| \leq \frac{a_\Delta}{16} \beta_j^2$ [cf.~A\ref{assum:stepsize}], we have
\beq \notag
\begin{split}
\| \ttheta^{(1,1)}_{k+1} \| & \leq \sqrt{\pth} \boundtsim \Big\{ \frac{G_{0:k}^{(1)}}{1-\beta_0 a_\Delta/2} (1 + \| \tw_0 \| + \ttheta_0 \|) + \beta_k (1 + \| \ttheta_{k+1} \| + \| \tw_{k+1} \|) \Big\} \\
&  + \sqrt{\pth} \boundtsim (\Boneinfty + a_\Delta/16 ) \sum_{j=0}^k \beta_j^2 G_{j+1:k}^{(1)} \big(1 + \| \ttheta_j \| + \| \tw_j \| \big) \\
& + \sqrt{\pth}  \sum_{j=0}^k G_{j+1:k}^{(1)} \big( \beta_j^2 \| A_{12} \tw_j + W_{j+1} \| +  \beta_j \gamma_j \| W_{j+1} + C_j V_{j+1} \| \big),
\end{split}
\eeq
Applying the Jensen's inequality \eqref{eq:jensen} and taking square on both sides give
\beq \notag
\begin{split}
& \| \ttheta^{(1,1)}_{k+1} \|^2 \leq 7 \pth (\boundtsim)^2 \Big\{ (G_{0:k}^{(1)})^2 \Big( \frac{1+\|\tw_0\|+\|\ttheta_0\|}{1 - \beta_0 a_\Delta/2}\Big)^2 + \beta_k^2 (1 + \| \ttheta_{k+1} \|^2 + \| \tw_{k+1} \|^2) \Big\} \\
& + 7 \pth (\boundtsim)^2 (\Boneinfty + a_\Delta/16 )^2 \CSEQ{a_\Delta/2} \beta_{k+1} \sum_{j=0}^k \beta_j^2 G_{j+1:k}^{(1)} (1 + \| \ttheta_j \|^2 + \| \tw_j \|^2 ) \Big\} \\
& + 7 \pth  \CSEQ{a_\Delta/2} \Big\{ \beta_{k+1} \sum_{j=0}^k G_{j+1:k}^{(1)} \beta_j^2 \| A_{12} \tw_j + W_{j+1} \|^2 + \gamma_{k+1} \sum_{j=0}^k G_{j+1:k}^{(1)}  \beta_j \gamma_j \| W_{j+1} + C_j V_{j+1} \|^2 \Big\},
\end{split}
\eeq 
Note the subtle difference that the last term takes $\gamma_{k+1}$.
Taking expectation on both sides leads to
\beq \notag
\begin{split}
\E[ \| \ttheta^{(1,1)}_{k+1} \|^2] & \leq 7 \pth \Big\{ (G_{0:k}^{(1)})^2 \Big( \frac{1+\|\tw_0\|+\|\ttheta_0\|}{1 - \beta_0 a_\Delta/2}\Big)^2 + \beta_k^2 (1 + \dth \mth_{k+1} + \dw \mtw_{k+1} ) \Big\} \\
& + 7 \pth (\boundtsim)^2 (\Boneinfty + a_\Delta/16 )^2 \CSEQ{a_\Delta/2} \beta_{k+1} \sum_{j=0}^k \beta_j^2 G_{j+1:k}^{(1)} (1 + \dth \mth_j + \dw \mtw_j ) \Big\} \\
& + 7\pth \CSEQ{a_\Delta/2} \sum_{j=0}^k G_{j+1:k}^{(1)} ( \beta_0\beta_j^2 \ttmtcom + \beta_j \gamma_j^2 \ttmwcom)  (1 + \mth_{j} + \mtw_{j}),
\end{split}
\eeq 
where we have used $\beta_{k+1} \leq \beta_0$ and $\gamma_{k+1} \leq \gamma_j$.
Again, using the bound $\gamma_j^2 \leq \rho_0 \beta_j$ from B\ref{assb:step}, we can simplify the above inequality into
\beq \label{eq:mth11bd_pfin}
\begin{split}
& \E[ \| \ttheta^{(1,1)}_{k+1} \|^2] \leq 7 \pth \Big\{ (G_{0:k}^{(1)})^2 \Big( \frac{1+\|\tw_0\|+\|\ttheta_0\|}{1 - \beta_0 a_\Delta/2}\Big)^2 + \beta_k^2 (1 + \dth \mth_{k+1} + \dw \mtw_{k+1} ) \Big\} \\
& + 7 \pth (\boundtsim)^2 (\Boneinfty + a_\Delta/16 )^2 \CSEQ{a_\Delta/2} \beta_{k+1} \sum_{j=0}^k \beta_j^2 G_{j+1:k}^{(1)} (1 + \dth \mth_j + \dw \mtw_j ) \Big\} \\
& + 7\pth \CSEQ{a_\Delta/2} ( \beta_0\ttmtcom + \rho_0 \ttmwcom) \sum_{j=0}^k \beta_j^2 G_{j+1:k}^{(1)} (1 + \mth_{j} + \mtw_{j}) \\
& \leq \hC_0^{(1,1)} (G_{0:k}^{(1)})^2 + \hC_1^{(1,1)} \beta_k^2 (1 +  \mth_{k+1} + \mtw_{k+1} )  + \hC_2^{(1,1)} \sum_{j=0}^k \beta_j^2 G_{j+1:k}^{(1)} ( \mth_j + \mtw_j ) + \hC_3^{(1,1)}\beta_{k+1} ,
\end{split}
\eeq 
where
\[
\hC_0^{(1,1)} = 7 \pth \Big( \frac{1+\|\tw_0\|+\|\ttheta_0\|}{1 - \beta_0 a_\Delta/2}\Big)^2,~~\hC_1^{(1,1)} = 7 \pth (\dth \vee \dw)
\]
\[
\hC_2^{(1,1)} = 7 \pth \CSEQ{a_\Delta/2} \big\{ (\dth \vee \dw) (\boundtsim)^2  (\Boneinfty + a_\Delta/16 )^2 \beta_0 + ( \beta_0\ttmtcom + \rho_0 \ttmwcom) \big\}
\]
\[
\hC_3^{(1,1)} = 7 \pth (\CSEQ{a_\Delta/2})^2 \big( (\boundtsim)^2  (\Boneinfty + a_\Delta/16 )^2 + \beta_0\ttmtcom + \rho_0 \ttmwcom )
\]
Observe that
\beq \notag
\begin{split}
& \sum_{j=0}^k \beta_j^2 G_{j+1:k}^{(1)} \mtw_j \leq \sum_{j=0}^k \beta_j^2 G_{j+1:k}^{(1)}  \Big\{ \Czerowtilde \widetilde{G}_{0:j-1}^{(2)} + \Conewtilde \gamma_j + \Ctwowtilde \sum_{i=0}^{j-1} \gamma_i^2 \widetilde{G}_{i+1:j-1}^{(2)} \mth_i +  \Cthreewtilde \gamma_{j-1}^2 \mth_j \Big\} \\
& \leq \big( \Czerowtilde + \Conewtilde \gamma_0 \big) \CSEQ{a_\Delta/2} \beta_{k+1} + \Cthreewtilde \gamma_0^2 \sum_{j=0}^k \beta_j^2 G_{j+1:k}^{(1)} \mth_j + \Ctwowtilde \sum_{j=0}^k \beta_j^2 G_{j+1:k}^{(1)} \sum_{i=0}^{j-1} \gamma_i^2 \widetilde{G}_{i+1:j-1}^{(2)} \mth_i. 
\end{split}
\eeq
Furthermore, using $G_{j+1:k}^{(1)} \leq (\widetilde{G}_{j+1:k}^{(1)})^2$ and $\widetilde{G}_{i+1:j-1}^{(2)} \leq \widetilde{G}_{i+1:j-1}^{(1)}$, we have
\beq
\begin{split}
& \sum_{j=0}^k \beta_j^2 G_{j+1:k}^{(1)} \sum_{i=0}^{j-1}  \gamma_i^2 \widetilde{G}_{i+1:j-1}^{(2)} \mth_i = \sum_{i=0}^{k-1} \gamma_i^2 \mth_i \sum_{j=i+1}^k \beta_j^2 G_{j+1:k}^{(1)} \widetilde{G}_{i+1:j-1}^{(2)} \\
& \leq \sum_{i=0}^{k-1} \gamma_i^2 \mth_i \sum_{j=i+1}^k \beta_j^2 (\widetilde{G}_{j+1:k}^{(1)})^2 \widetilde{G}_{i+1:j-1}^{(1)} \leq \frac{\CSEQ{a_\Delta/4} \beta_{k+1}}{1 - \beta_0 a_\Delta/4} \sum_{i=0}^{k-1} \gamma_i^2 \mth_i \widetilde{G}_{i+1:k}^{(1)} 
\end{split}
\eeq
Moreover, through applying $\widetilde{G}_{i+1:j-1}^{(2)} \leq \widetilde{G}_{i+1:j-1}^{(1)}$ and $\beta_k \leq \beta_j$ for any $j \leq k$, we have
\[
\beta_k^2 \mtw_{k+1} \leq \beta_0^2 \Czerowtilde \widetilde{G}_{0:k}^{(1)} + \varsigma \gamma_0 \Conewtilde \beta_{k+1} + \Ctwowtilde \gamma_0^2 \sum_{j=0}^k \beta_j^2 \widetilde{G}_{j+1:k}^{(1)} \mth_j + \Cthreewtilde \gamma_0^2 \mth_{k+1},
\]
where we have used $\beta_k^2 \leq \varsigma \beta_{k+1}$. The above results simplify \eqref{eq:mth11bd_pfin} into
\beq \label{eq:mth11bd_fin}
\begin{split}
\E[ \| \ttheta^{(1,1)}_{k+1} \|^2] & \leq \tC_0^{(1,1)} \widetilde{G}_{0:k}^{(1)} + \tC_1^{(1,1)} \beta_{k+1} + \tC_2^{(1,1)} \sum_{j=0}^k \beta_j^2 \widetilde{G}_{j+1:k}^{(1)} \mth_j + \tC_3^{(1,1)} \beta_k^2 \mth_{k+1},
\end{split}
\eeq 
where 
\beq \label{eq:c11const}
\begin{split}
& \tC_0^{(1,1)} = \hC_0^{(1,1)} + \beta_0^2 \hC^{(1,1)}_1 \Czerowtilde, \\
& \tC_1^{(1,1)} = \hC_1^{(1,1)} (1 + \varsigma  \gamma_0 \Conewtilde) + \hC_3^{(1,1)} + \hC_2^{(1,1)} \big( \Czerowtilde + \Conewtilde \gamma_0 \big) \CSEQ{a_\Delta/2} \\
& \tC_2^{(1,1)} = \hC_2^{(1,1)} \Big(1 + \Cthreewtilde \gamma_0^2 + \frac{ \Ctwowtilde \CSEQ{a_\Delta/4} \beta_0 }{1 - \beta_0 a_\Delta/4} \Big) + \hC_1^{(1,1)} \Ctwowtilde \gamma_0^2,\\
& \tC_3^{(1,1)} = \hC_1^{(1,1)} (1 + \Cthreewtilde \gamma_0^2). 
\end{split}
\eeq
Finally, combining \eqref{eq:mth0bd_fin}, \eqref{eq:mth10bd_fin}, \eqref{eq:mth11bd_fin} gives
\beq
\begin{split}
\mth_{k+1} & \leq 3 \big( \| \E [ \ttheta^{(0)}_{k+1} ( \ttheta^{(0)}_{k+1} )^\top ] \| + \E [ \| \ttheta_{k+1}^{(1,0)} \|^2 ] + \E [ \| \ttheta_{k+1}^{(1,1)} \|^2 ] \big) \\
& \leq 3 \Big\{ \big( \tC_0^{(0)} + \tC_0^{(1,0)} + \tC_0^{(1,1)} \big)   \widetilde{G}_{0:k}^{(1)} + \big( \tC_1^{(0)} + \tC_1^{(1,0)} + \tC_1^{(1,1)} \big) \beta_{k+1} \Big\} \\
& + 3 \Big\{ \big( \tC_2^{(0)} + \tC_2^{(1,0)} + \tC_2^{(1,1)} \big) \sum_{i=0}^k \beta_i^2 \widetilde{G}_{i+1:k}^{(1)} \mth_i + \tC_3^{(1,1)} \beta_k^2 \mth_{k+1} \Big\} 
\end{split}
\eeq
As we have $3 \tC_3^{(1,1)} \beta_k^2 \leq 1/2$,
we have
\beq \label{eq:mthmarkovbd_app}
\boxed{\mth_{k+1} \leq \Czerottilde  \widetilde{G}_{0:k}^{(1)} + \Conettilde \beta_{k+1} + \Ctwottilde \sum_{i=0}^k \beta_i^2 \widetilde{G}_{i+1:k}^{(1)} \mth_i,}
\eeq
where
\beq \label{eq:ctheta_markov}
\begin{split}
& \Czerottilde := 6 \big( \tC_0^{(0)} + \tC_0^{(1,0)} + \tC_0^{(1,1)}),~~\Conettilde := 6 \big( \tC_1^{(0)} + \tC_1^{(1,0)} + \tC_1^{(1,1)} \big),\\
& \Ctwottilde := 6 \big( \tC_2^{(0)} + \tC_2^{(1,0)} + \tC_2^{(1,1)} \big).
\end{split}
\eeq
This concludes the proof of Proposition~\ref{prop:mth_markov}.

\paragraph{Completing the Proof of Theorem~\ref{theo:preliminary-bound-markov}}
From \eqref{eq:mthmarkovbd} we can derive a bound for $\mth_k$ as follows. Let $\topu_0 = \Czerottilde$, observe the following equivalent forms of the recursion
\[
\begin{split}
& \topu_{k+1} = \Czerowtilde \widetilde{G}_{0:k}^{(1)} + \Conettilde \beta_{k+1} + \Ctwottilde \sum_{i=0}^k \beta_i^2 \widetilde{G}_{i+1:k}^{(1)} \topu_i \\
\Longleftrightarrow~ & \topu_{k+1} = (1 - \beta_k a_\Delta/4) \topu_k + \Conettilde ( \beta_{k+1} - \beta_k + \beta_k^2 a_\Delta/4) + \Ctwottilde \beta_k^2 \topu_k\\
&  \leq (1 - \beta_k a_\Delta / 8) \topu_k + \Conettilde \beta_k^2 a_\Delta / 4
\end{split}
\]
where the last inequality is due to the fact $\beta_k \Ctwottilde \leq a_\Delta / 8$. 
Subsequently, we have
\[
\begin{split}
\topu_{k+1} & \leq \prod_{i=0}^k (1- \beta_i a_\Delta/8) \topu_0 + \frac{\Conettilde a_\Delta}{4} \sum_{j=0}^k \gamma_j^2 \prod_{i=j+1}^k (1 - \beta_i a_\Delta/8)\\
& \leq \prod_{i=0}^k (1- \beta_i a_\Delta/8) \topu_0 + \frac{\Conettilde a_\Delta}{4} \CSEQ{a_\Delta/8} \beta_{k+1},
\end{split}
\] 
Observing that $\mth_{k+1} \leq \topu_{k+1}$, we obtain
\beq \label{eq:mtmarkovbd_fin}
\boxed{\mth_{k+1} \leq \Czerottilde \prod_{i=0}^k (1- \beta_i a_\Delta/8)  + \Conettilde \frac{ a_\Delta}{4} \CSEQ{a_\Delta/8} \beta_{k+1},}
\eeq
We obtain \eqref{eq:convergence-slow-markov} by setting ${\rm C}_1^{\ttheta,\sf mark} = \frac{ a_\Delta}{4}  \Conettilde \CSEQ{a_\Delta/8}$ and observing $\Czerottilde \leq {\rm C}_0^{\ttheta,\sf mark} (1 + {\rm V}_0)$ for some constant ${\rm C}_0^{\ttheta,\sf mark}$.

Finally, we bound the tracking error $\hw_k := w_k - A_{22}^{-1} ( b_2 - A_{21} \theta_k )$ as follows. Similarly to the martingale noise case, we set $\mhr_k: = \normop{\E[\hw_k \hw_k^\top]}$ and observe:
\[
\mhr_{k+1} \le 2 \mtw_{k+1} + 2 \| L_{k+1} \|^2 \mth_{k+1} \leq 2 \mtw_{k+1} + 2  \Linfty^2  \frac{\lambda_{\sf max}( Q_{22} )}{\lambda_{\sf min}( Q_\Delta )} \mth_{k+1}
\]
Substituting \eqref{eq:mtmarkovbd_fin} into \eqref{eq:mwmarkovbd_app} gives
\[
\begin{split}
& \mtw_{k+1} \leq \widetilde{G}_{0:k}^{(2)} \Czerowtilde + \Conewtilde \gamma_{k+1} + \Cthreewtilde \gamma_k^2  \mth_{k+1} + \Ctwowtilde \sum_{j=0}^k \gamma_j^2 \widetilde{G}_{j+1:k}^{(2)} \Big\{ \Czerottilde \prod_{i=0}^{j-1} \Big(1- \beta_i \frac{a_\Delta}{8} \Big)  + \Conettilde \frac{ a_\Delta}{4} \CSEQ{a_\Delta/8} \beta_{j} \Big\} \\
& \overset{(a)}{\leq} \widetilde{G}_{0:k}^{(2)} \Czerowtilde + \Conewtilde \gamma_{k+1}  + \Cthreewtilde \gamma_k^2  \mth_{k+1} + \Ctwowtilde \gamma_{k+1}  \Big\{ \frac{ \Czerot \CSEQ{a_{22}/8}}{1- \beta_0 a_\Delta/8}  \prod_{i=0}^{k} \Big(1- \beta_i \frac{a_\Delta}{8} \Big) + \Conettilde \frac{ a_\Delta}{4} \CSEQ{a_\Delta/8} \CSEQ{a_{22}/4} \Big\} \\
& \leq \Big\{ \Conewtilde + \Ctwowtilde \Conettilde \frac{ a_\Delta}{4} \CSEQ{a_\Delta/8} \CSEQ{a_{22}/4} \Big\} \gamma_{k+1} + \Big\{ \Czerowtilde +  \frac{\Ctwowtilde \Czerot \CSEQ{a_{22}/8}}{1- \beta_0 a_\Delta/8} \Big\} \prod_{i=0}^{k} \Big(1- \beta_i \frac{a_\Delta}{8} \Big) + \Cthreewtilde \gamma_k^2  \mth_{k+1} 
\end{split}
\]
where we have used $\widetilde{G}_{j+1:k}^{(2)} \leq \big( \prod_{i=j+1}^k (1-\gamma_i a_{22}/8) \big)^2$ in (a). As such, together with \eqref{eq:mtmarkovbd_fin} this gives 
\beq \label{eq:mthwbound_markov_fin}
\boxed{
\mhr_{k+1} \leq \tC_0^{\hw} \prod_{\ell=0}^k \Big(1- \beta_\ell \frac{a_\Delta}{8} \Big) + {\rm C}_1^{\hw,\sf mark} \gamma_{k+1},
}
\eeq
where
\beq \label{eq:mthwbd_markov_const}
\begin{split}
& \tC_0^{\hw} := 2 \Bigg\{ \Czerowtilde +  \frac{\Ctwowtilde \Czerot \CSEQ{a_{22}/8}}{1- \beta_0 a_\Delta/8} + (1 + \Cthreewtilde) \Linfty^2  \frac{\lambda_{\sf max}( Q_{22} )}{\lambda_{\sf min}( Q_\Delta )} \Czerottilde \Bigg\}\\
& {\rm C}_1^{\hw,\sf mark} := 2 \Bigg\{ \Conewtilde + \Ctwowtilde \Conettilde \frac{ a_\Delta}{4} \CSEQ{a_\Delta/8} \CSEQ{a_{22}/4} + \kappa (1 + \Cthreewtilde) \Linfty^2  \frac{\lambda_{\sf max}( Q_{22} )}{\lambda_{\sf min}( Q_\Delta )} \Conettilde \frac{a_\Delta}{4} \CSEQ{a_\Delta/8}  \Bigg\}
\end{split}
\eeq
Similarly, as $ \tC_0^{\hw} \leq  {\rm C}_0^{\hw,\sf mark} (1 + {\rm V}_0 )$ for some constant $ {\rm C}_0^{\hw,\sf mark}$, the above yields \eqref{eq:tracking-fast-component-markov}. We conclude the proof of Theorem~\ref{theo:preliminary-bound-markov}.

\subsubsection{Auxiliary Results for the Markovian Noise Case} \label{subsec:aux_markov}
\begin{lemma}\label{lem:addsubtract}
Let $(a_j)_{j \geq 0}$ be a sequence of $\dth$-dimensional vectors. The following equality holds:
\beq \label{eq:addsubtractlem}
\begin{split}
& \textstyle \sum_{j=0}^k \beta_j \ProdB_{j+1:k}^{(1)} (a_j - a_{j+1}) \\
& \textstyle = \beta_0 \ProdB_{1:k}^{(1)} a_0 - \beta_k a_{k+1} + \sum_{j=1}^k \big( \beta_j^2 B_{11}^j \ProdB_{j+1:k}^{(1)} + (\beta_j - \beta_{j-1}) \ProdB_{j:k}^{(1)} \big) a_j
\end{split}
\eeq
Similarly, for $(b_j)_{j \geq 0}$ being a sequence of $\dw$-dimensional vectors, it holds:
\beq
\begin{split}
& \textstyle \sum_{j=0}^k \gamma_j \ProdB_{j+1:k}^{(2)} (b_j - b_{j+1}) \\
& \textstyle = \gamma_0 \ProdB_{1:k}^{(2)} b_0 - \gamma_k b_{k+1} + \sum_{j=1}^k \big( \gamma_j^2 B_{22}^j \ProdB_{j+1:k}^{(2)} + (\gamma_j - \gamma_{j-1}) \ProdB_{j:k}^{(1)} \big) b_j.
\end{split}
\eeq
\end{lemma}
\begin{proof}
We only prove \eqref{eq:addsubtractlem}. Observe the following chain
\beq
\begin{split}
& \sum_{j=0}^k \beta_j \ProdB_{j+1:k}^{(1)} (a_j - a_{j+1}) = \sum_{j=0}^k \beta_j \ProdB_{j+1:k}^{(1)} a_j - \sum_{j=0}^k \beta_j \ProdB_{j+1:k}^{(1)} a_{j+1} \\
& = \beta_0 \ProdB_{1:k}^{(1)} a_0 - \beta_k a_{k+1} + \sum_{j=1}^k \big( \beta_j \ProdB_{j+1:k}^{(1)} - \beta_{j-1} \ProdB_{j:k}^{(1)} \big) a_j 
\end{split}
\eeq
Using $\beta_j \ProdB_{j+1:k}^{(1)} - \beta_{j-1} \ProdB_{j:k}^{(1)} = \beta_j^2 B_{11}^j \ProdB_{j+1:k}^{(1)} + (\beta_j - \beta_{j-1}) \ProdB_{j:k}^{(1)}$ concludes the proof.
\end{proof}

\paragraph{Derivation of Eq.~\eqref{eq:WVdecompose}} The decomposition is obtained through repeatedly adding/subtracting terms. Particularly, we observe that the individual terms can be expressed as:
\[
\begin{array}{rl}
C_j \big(\markovonetb_j - \markovonetb_{j+1}\big) & = C_j \markovonetb_j - C_{j-1} \markovonetb_j + C_{j-1} \markovonetb_{j} - C_j \markovonetb_{j+1} \\[.1cm]
\big( \markovtwoB_j - \markovtwoB_{j+1} \big) \tw_j & = \markovtwoB_j \tw_j - \markovtwoB_{j+1} \tw_{j+1} + \markovtwoB_{j+1} ( \tw_{j+1} - \tw_j ) \\[.1cm]
C_j \big( \markovoneB_j - \markovoneB_{j+1} \big) \tw_j & = (C_j - C_{j-1}) \markovoneB_j \tw_j + C_{j-1} \markovoneB_j (\tw_j - \tw_{j-1}) \\[.2cm]
& \hspace{2cm} + C_{j-1} \markovoneB_j \tw_{j-1} - C_j \markovoneB_{j+1} \tw_j \\[.1cm]
\big( \markovtwoA_j - \markovtwoA_{j+1} \big) \ttheta_j & = \markovtwoA_j \ttheta_j - \markovtwoA_{j+1} \ttheta_{j+1} + \markovtwoA_{j+1} ( \ttheta_{j+1} - \ttheta_j ) \\[.1cm]
C_j \big( \markovoneA_j - \markovoneA_{j+1} \big) \ttheta_j & = (C_j - C_{j-1}) \markovoneA_j \ttheta_j + C_{j-1} \markovoneA_j (\ttheta_j - \ttheta_{j-1}) \\[.1cm]
& \hspace{2cm} + C_{j-1} \markovoneA_j \ttheta_{j-1} - C_j \markovoneA_{j+1} \ttheta_j \\[.1cm]
\big( \markovtwoB_j - \markovtwoB_{j+1} \big) C_{j-1} \ttheta_j & = \markovtwoB_j (C_{j-1} - C_{j-2}) \ttheta_j + \markovtwoB_j C_{j-2} (\ttheta_j - \ttheta_{j-1}) \\[.1cm]
& \hspace{2cm} + \markovtwoB_j C_{j-2} \ttheta_{j-1} - \markovtwoB_{j+1} C_{j-1} \ttheta_j \\[.1cm]
C_j ( \markovoneB_j - \markovoneB_{j+1}) C_{j-1} \ttheta_j & = (C_j - C_{j-1}) \markovoneB_j C_{j-1} \ttheta_j + C_{j-1} \markovoneB_j (C_{j-1} - C_{j-2}) \ttheta_j \\[.1cm]
& \hspace{2cm} + C_{j-1} \markovoneB_j C_{j-2} ( \ttheta_j - \ttheta_{j-1} ) \\[.1cm]
& \hspace{2cm} +  C_{j-1} \markovoneB_j C_{j-2} \ttheta_{j-1} - C_j \markovoneB_{j+1} C_{j-1} \ttheta_j.
\end{array}
\]
Collecting terms on the right hand side of the above equations yields \eqref{eq:WVdecompose}. Moreover, we the vectors/matrices that appear in \eqref{eq:WVdecompose} can be bounded as
\[
\| \markovterm_j^{WV} \| \leq  \overline{\rm b} (1 + \Cinfty),~~
\| \simplewtdiff_j \| \leq  \overline{\rm A} (1 + 2 \Cinfty + \Cinfty^2),~~\| \simplewwdiff_j \| \leq \overline{\rm A}(1 + \Cinfty)
\]
\[
\| \simplewtiter \| \leq \overline{\rm A} (1 + 2 \Cinfty + \Cinfty^2),~~\| \simplewwiter \| \leq \overline{\rm A} (1 + \Cinfty )
\]
\[
\| \simplewtA_j \| \leq \overline{\rm A} C_2^U \CSEQ{a_{22}/2} ( 1 + \Cinfty ) (1 +  \varsigma ) \gamma_j, ~~\| \simplewwA_j \| \leq \overline{\rm A} C_2^U \CSEQ{a_{22}/2} \gamma_j .
\]
where the last inequality is due to Lemma~\ref{lem:Lkgamkbound} and we have used $\gamma_{j-1} \leq \varsigma \gamma_j$ [cf.~A\ref{assum:stepsize}-1].
Consequently, we can establish the bounds on the matrix/vector norms by setting
\beq \label{eq:markovWV_const}
\boundwsim := \max\{ \overline{\rm b} (1 + \Cinfty), \overline{\rm A}(1 + 2 \Cinfty + \Cinfty^2), \overline{\rm A} C_2^U \CSEQ{a_{22}/2} ( 1 + \Cinfty ) (1 + \varsigma ) \}.
\eeq

\paragraph{Derivation of Eq.~\eqref{eq:Vdecompose}} Setting $\markovtermt_j^{b_1} := \markovterm_j^{b_1} + \Markovterm_j^{A_{11}} \theta^\star + \Markovterm_j^{A_{12}} w^\star$, we observe
\[
\begin{split}
V_{j+1}^{(1)} & = (\markovoneb_j - \markovoneb_{j+1}) + (\markovoneA_j - \markovoneA_{j+1}) \theta_j + (\markovoneB_j - \markovoneB_{j+1}) w_j \\
& = \markovtermt_j^{b_1} - \markovtermt_{j+1}^{b_1} + (\markovoneA_j - \markovoneA_{j+1}) \ttheta_j + (\markovoneB_j - \markovoneB_{j+1}) \tw_j
\end{split}
\]
 Similar to the previous paragraph, the decomposition is obtained through repeatedly adding/subtracting terms. We observe
\[
\begin{array}{rl}
(\markovoneA_j - \markovoneA_{j+1}) \ttheta_j & = \markovoneA_j \ttheta_j - \markovoneA_{j+1} \ttheta_{j+1} + \markovoneA_{j+1} ( \ttheta_{j+1} - \ttheta_j ) \\[.2cm]
(\markovoneB_j - \markovoneB_{j+1}) \tw_j & =  \markovoneB_j \tw_j - \markovoneB_{j+1} \tw_{j+1} + \markovoneB_{j+1} ( \tw_{j+1} - \tw_j )
\end{array}
\]


\section{Detailed Proof of Theorem~\ref{th: expansion}} \label{app:ext}
Throughout this section we will use additional notations. We denote 
\[
\kappa_\ell: = \frac{\beta_\ell}{\gamma_\ell}. 
\]
Let
\beq
\label{eq:condition-kappa expansion}
\kappa_\infty^{\sf exp} : = \kappa_\infty \wedge (1/2) (\normop{A_{12}} \Cinfty \CSEQ{a_{22}})^{-1}
\eeq
and
\beq
\label{eq:condition-beta expansion}
\beta_\infty^{\sf exp} : = \beta_\infty^{\sf mtg} \wedge 1/(4 \normop{\Delta})
\eeq
We assume $\beta_k \le \beta_\infty^{\sf exp}, \gamma_k \le \gamma_\infty^{\sf mtg}, \kappa_k \le \kappa \le \kappa_\infty^{\sf exp}$. 
Furthermore, let us define 
$$
\ProdBB_{m,n}^{(1)}:= \prod_{j=m}^n (\Id - \beta_j \Delta), \, \ProdBB_{m,n}^{(2)}:= \prod_{j=m}^n (\Id - \gamma_j A_{22}).
$$
Using standard arguments we may bound operator norm of these matrices
\beq
\label{eq: gamma 1 estimate}
\| \ProdBB_{m:n}^{(1)} \|  
\leq \sqrt{\pth} \prod_{j=m}^n(1 - a_{\Delta}\beta_j), \, \| \ProdBB_{m:n}^{(2)} \| \leq \sqrt{\pw}  \prod_{j=m}^n(1 - a_{22}\gamma_j)
\eeq
We set quantities $\mw, \mv$ from the assumption~\ref{assum:bound-conditional-variance} to be equal to 
$$
\mw := \mv: = \max(\EmVW, \normop{\Sigma^{11}}, \normop{\Sigma^{12}}, \normop{\Sigma^{22}})
$$
All conditions of Theorem~\ref{theo:preliminary-bound-martingale} are satisfied. We will use this theorem in the following form
\begin{align}
\label{eq: theta crude bound}
\mth_k &\le \EConst{0, \theta}  \prod_{\ell = 0}^{k-1}(1 - (a_\Delta/4) \beta_{\ell}) {\rm V}_0 + \EConst{1,\theta} \beta_{k}, \\
\label{eq: w crude bound}
\mtw_k &\le  \EConst{0, w} \prod_{\ell = 0}^{k-1}(1 - (a_\Delta/4) \beta_{\ell}) {\rm V}_0 + \EConst{1,w} \gamma_k,
\end{align}
where $\EConst{0, \theta}, \EConst{1, \theta}, \EConst{0,w}$ and $\EConst{1,w}$ denote corresponding constants from Theorem~\ref{theo:preliminary-bound-martingale}. 
Similarly to~\eqref{eq:noise:constants} and~\eqref{eq:noise:constantsVW} we can define $\EtmVW$. Hence, the following inequality holds 
$$
\normop{\E[V_jV_j^T]} \vee \normop{ \E[W_j W_j^T]}  \vee \normop{\E[V_j W_j^T]}  \le \EtmVW (1 + \mth_k + \mtw_k)
$$
Applying~\eqref{eq:2ts1-1} and~\eqref{eq:2ts2-1} (compare with~\cite{konda:tsitsiklis:2004}[Formula 4.4]) we may write down the following expansion for $\ttheta_{k+1}$:
\begin{align} \nonumber 
    \ttheta_{k+1} &= S_{k+1}^{(0)} + \ldots +   S_{k+1}^{(6)},
\end{align}    
where 
\begin{align}   
    S_{k+1}^{(0)} &:= \ProdBB_{0:k}^{(1)} \ttheta_0; \nonumber \\
    S_{k+1}^{(1)} &:= \sum_{j=0}^{k} \beta_j \ProdBB_{j+1:k}^{(1)} A_{12} \ProdBB_{0:j}^{(2)} \tw_0; \nonumber \\
    S_{k+1}^{(2)} &:= \sum_{j=0}^{k} \beta_j \ProdBB_{j+1:k}^{(1)} \delta_j^{(1)}; \nonumber \\ 
    S_{k+1}^{(3)}&:=\sum_{j=0}^{k} \beta_j\ProdBB_{j+1:k}^{(1)} (V_{j+1} + A_{12} A_{22}^{-1} W_{j+1}); \nonumber\\ 
    S_{k+1}^{(4)} &:= \sum_{j=0}^{k} \beta_j \ProdBB_{j+1:k}^{(1)} A_{12} \bigg(\sum_{\ell=0}^{j-1} \beta_\ell \ProdBB_{\ell+1:j-1}^{(2)} \delta_\ell^{(2)}\bigg); \nonumber \\
     S_{k+1}^{(5)} &:= \sum_{j=0}^{k} \beta_j \ProdBB_{j+1:k}^{(1)} A_{12} \bigg(\sum_{\ell=0}^{j-1} \beta_\ell \ProdBB_{\ell+1:j}^{(2)} C_\ell V_{\ell+1}\bigg); \nonumber \\
     S_{k+1}^{(6)} &:= \sum_{j=0}^{k} \beta_j \ProdBB_{j+1:k}^{(1)} A_{12} \sum_{\ell=0}^{j-1} \gamma_\ell \ProdBB_{\ell+1:j-1}^{(2)} W_{\ell+1} - \sum_{\ell=0}^{k} \beta_\ell\ProdBB_{j+1:k}^{(1)} A_{12} A_{22}^{-1} W_{\ell+1}, \nonumber
    \end{align}
where     
$$
\delta_\ell^{(1)} := A_{12} L_\ell \ttheta_\ell, \, \delta_\ell^{(2)} := - C_k A_{12} \tw_\ell.
$$
We will group all terms in the expansion into $5$ blocks, $S_{k+1}^{(0)} + S_{k+1}^{(1)}, S_{k+1}^{(2)}  + S_{k+1}^{(5)}, S_{k+1}^{(3)}, S_{k+1}^{(4)}$ and $S_{k+1}^{(6)}$. It is easy to see that 
$S_{k+1}^{(0)} + S_{k+1}^{(1)}$ is uncorrelated with $S_{k+1}^{(3)}, S_{k+1}^{(6)}$  (moreover it is uncorrelated with $S_{k+1}^{(5)}$, but we ignore this fact). Since $\E[\norm{\ttheta_{k+1}}^2] = \E[\Tr(\ttheta_{k+1} \ttheta_{k+1}^\top)]$ and by the linearity of trace using expansion we show 
\begin{align}
\label{eq:expansion}
\E[\norm{\ttheta_{k+1}}^2] &=  \E\big[\Tr(S_{k+1}^{(3)} (S_{k+1}^{(3)})^\top)\big] + J_{k+1}', 
\end{align}
where for $J_{k+1}'$ we will use the following crude estimate
\begin{align}
|J_{k+1}'| &\le
  3\E\big[\Tr(\{S_{k+1}^{(0)} + S_{k+1}^{(1)}\} \{S_{k+1}^{(0)} + S_{k+1}^{(1)}\}^\top) \big] \nonumber \\
  & + 5\E\big[\Tr(\{S_{k+1}^{(2)} +  S_{k+1}^{(5)}\} \{S_{k+1}^{(2)} +  S_{k+1}^{(5)}\}^\top) + 5\E\big[\Tr(S_{k+1}^{(6)}(S_{k+1}^{(6)})^\top)\big] \big] \nonumber \\
  &+ 2\E\big[\Tr(S_{k+1}^{(3)} (S_{k+1}^{(6)})^\top)\big] + 2 \E\big[\Tr(S_{k+1}^{(3)} \{S_{k+1}^{(2)} +  S_{k+1}^{(5)}\}^\top) \big] \nonumber \\
  &+ 5\E\big[\Tr(S_{k+1}^{(4)} (S_{k+1}^{(4)})^\top)\big] + 2 \E\big[\Tr(S_{k+1}^{(3)} (S_{k+1}^{(4)})^\top) \big] \nonumber
\end{align}
Using martingale property and definition of $\Sigma$ we rewrite the term $\E[\Tr(S_{k+1}^{(3)} (S_{k+1}^{(3)})^\top)]$ as follows
\beq
 \label{eq: leading term expansion}
\begin{split}
 \Tr ( \E[S_{k+1}^{(3)} (S_{k+1}^{(3)})^\top]) &= \sum_{j=0}^{k} \beta_j^2 \Tr(\ProdBB_{j+1:k}^{(1)} \Sigma [\ProdBB_{j+1:k}^{(1)}]^\top) \\
 &+ \sum_{j=0}^{k} \beta_j^2 \Tr(\ProdBB_{j+1:k}^{(1)} (\Sigma_j - \Sigma) [\ProdBB_{j+1:k}^{(1)}]^\top)
\end{split}
\eeq
where
$$
\Sigma_j := \E[V_j V_j^\top] + A_{12} A_{22}^{-1} \E[W_j W_j^\top]  A_{22}^{-\top} A_{12}^\top + \E[V_j W_j^\top]  A_{22}^{-\top} A_{12}^\top + A_{12} A_{22}^{-1} \E[W_j V_j^\top]
$$
\paragraph{Leading term in~\eqref{eq: leading term expansion}}

For lower bound of the first term in~\eqref{eq: leading term expansion} we  will use the following fact. Since for any $s \in [j+1,k]$
$$
(\Id - \beta_s \Delta)^\top (\Id - \beta_s \Delta) = \Id - \beta_s (\Delta + \Delta^\top) + \beta_s^2 \Delta^\top \Delta \succeq (1 - 2\beta_s \|\Delta\|) \Id, 
$$
we obtain using Lemma~\ref{lem:bsum} (and remark after this lemma)
\beq
\label{eq:lower bound}
\sum_{j=0}^{k} \beta_j^2 \Tr(\ProdBB_{j+1:k}^{(1)} \Sigma [\ProdBB_{j+1:k}^{(1)}]^\top) \geq \beta_{k+1} \Tr \Sigma \sum_{j=0}^{k} \beta_j \prod_{\ell=j+1}^{k}(1-2 \beta_\ell \normop{\Delta})  \geq \EConst{3} \beta_{k+1} \Tr \Sigma
\eeq
where
\beq
\label{eq: exp const 3}
\EConst{3}: =  1/(8 \normop{\Delta})
\eeq
and we used $\beta_\infty^{\sf exp} \le 1/(4 \normop{\Delta})$ and $k \geq k_0^{\sf \exp}$. To obtain upper bound we
apply von Neumann trace inequality 
(i.e. $\Tr(A B) \le \sum_{j=1}^n a_j b_j$, where $\{a_j\}$ and $\{b_j\}$ are non-increasing sequences of eigenvalues of Hermitian matrices $A$ and $B$ resp.) and Lemma~\ref{cor:rate-of-convergence}
\beq
\label{eq:upper bound}
 \sum_{j=0}^{k} \beta_j^2 \Tr(\ProdBB_{j+1:k}^{(1)} \Sigma [\ProdBB_{j+1:k}^{(1)}]^\top) \le \pth \Tr(\Sigma) \sum_{j=0}^{k} \beta_j^2 \prod_{\ell=j+1}^{k}(1 - a_\Delta \beta_\ell)^2  \le \EConst{4}  \Tr(\Sigma) \beta_{k+1}  
\eeq
where
\beq
\label{eq: exp const 4}
\EConst{4}: = \pth \CSEQ{a_\Delta}
\eeq
Inequalities~\eqref{eq:lower bound} and~\eqref{eq:upper bound} together imply~\eqref{eq:leading term of exp}. 

\paragraph{Remainder term in~\eqref{eq: leading term expansion}}
The second term in~\eqref{eq: leading term expansion}, which we denote by $R_{k+1}$ may be estimated as follows
\begin{align}
    |R_{k+1}| \le \pth \dth \EtmVW (1 +  \normop{A_{12} A_{22}^{-1}})^2 \sum_{j=0}^{k} \beta_j^2 \prod_{\ell=j+1}^k (1 - a_\Delta \beta_\ell)^2 (\mth_{j} + \mtw_j) \nonumber
\end{align}
Applying~\eqref{eq: theta crude bound} and~\eqref{eq: w crude bound} and Lemma~\ref{cor:rate-of-convergence} 
\begin{align}\label{eq: second term of leading term exp}
\boxed{    |R_{k+1} | \le \EConst{3,0}\prod_{\ell=0}^k (1 - (a_\Delta/4) \beta_\ell) {\rm V}_0 \beta_{k+1} + \EConst{3,1} \beta_{k+1} \gamma_{k+1},  }
\end{align}
where 
\begin{align}
    \EConst{3,0}&:= \pth \dth \EtmVW (1 +  \normop{A_{12} A_{22}^{-1}})^2 \CSEQ{a_\Delta} (\EConst{0, \theta} + \EConst{0,w})/(1 - a_\Delta \beta_\infty^{\sf exp}), \nonumber \\
    \EConst{3,1}&:= \pth \dth \EtmVW (1 +  \normop{A_{12} A_{22}^{-1}})^2 \CSEQ{a_\Delta} (\kappa_\infty^{\sf exp} \EConst{1, \theta} + \EConst{1, w}) \nonumber
\end{align}
\paragraph{Estimation of $J_{k+1}'$}
To finish the proof of the theorem it remains to estimate $J_{k+1}'$.  Applying~\eqref{eq: gamma 1 estimate} it is easy to check that
$$
\Tr(\E [S_{k+1}^{(0)} (S_{k+1}^{(0)})^\top]) = \Tr(\ProdBB_{0:k}^{(1)}  \E[\ttheta_0 \ttheta_0^\top] [\ProdBB_{0:k}^{(1)} ]^\top) \le  \pth \prod_{\ell=0}^k (1 - a_\Delta \beta_\ell)^2  \E[\norm{\ttheta_0}^2]
$$
Similarly, recalling that $\kappa_\infty^{\sf exp} \le (1/4) a_{22}/a_{\Delta}$ and using $\prod_{s = 0}^j (1 - a_{22}\gamma_s)(1 - a_\Delta \beta_s)^{-1} \le \prod_{s = 0}^j (1 - (a_{22}/2)\gamma_s)$ we obtain 
\begin{align}
\Tr(\E [S_{k+1}^{(1)} (S_{k+1}^{(1)})^\top]) &= \sum_{j=0}^{k} \sum_{\ell=0}^{k}  \beta_j \beta_\ell \Tr\big(\ProdBB_{j+1:k}^{(1)} A_{12} \ProdBB_{0:j}^{(2)}\E[\tw_0 \tw_0^\top] [\ProdBB_{0:\ell}^{(2)}]^\top A_{12}^\top [\ProdBB_{\ell+1:k}^{(1)}]^\top  \big) \nonumber \\
& \le \EConst{1} \prod_{\ell=0}^{k} (1 - a_\Delta \beta_\ell)^2 \E[\norm{\tw_0}^2], \nonumber
\end{align}
where
$$
\EConst{1} : = (4/a_{22}^2)  \pw \pth \normop{A_{12}}^2 (\kappa_\infty^{\sf exp})^2
$$
Hence, we may conclude from the previous two inequalities that
\begin{align}\label{eq: term 0 + 1}
\boxed{\E[\Tr(\{S_{k+1}^{(0)} + S_{k+1}^{(1)}\}\{S_{k+1}^{(0)} + S_{k+1}^{(1)}\}^\top)]  \le \EConst{0+1}\prod_{\ell=0}^{k} (1 - a_\Delta \beta_\ell)^2 {\rm V}_0}, 
\end{align}
where
$$
\EConst{0+1} := 2\pth + 4 \EConst{1} (1 + \Cinfty^2) 
$$
For the next term in the expansion we have
\begin{align}
\Tr(\E [S_{k+1}^{(2)} (S_{k+1}^{(2)})^\top]) &\le  \sum_{j=0}^{k} \sum_{\ell=0}^{k} \beta_j \beta_\ell \E \big[\big|\Tr \big( \ProdBB_{j+1:k}^{(1)}  A_{12} L_j \ttheta_j \ttheta_l^\top  L_l^\top A_{12}^T [\ProdB_{\ell+1:k}^{(1)} ]^\top \big) \big|\big] \nonumber
\end{align}
We apply Cauchy-Schwartz inequality twice, first $|\Tr(A B^\top)| \le \Tr^{1/2}(A A^\top) \Tr^{1/2} (B B^\top)$ and then for expectation. We obtain 
\begin{align}
\Tr(\E [S_{k+1}^{(2)} (S_{k+1}^{(2)})^\top]) 
 \le \bigg(\sum_{j=0}^{k}  \beta_j  \big(\Tr \big( \ProdBB_{j+1:k}^{(1)} A_{12} L_j  \E[\ttheta_j \ttheta_j^\top]  L_j^\top A_{12}^\top [\ProdBB_{j+1:k}^{(1)}]^\top \big) \big)^{1/2} \bigg)^2 \nonumber
\end{align}
From Lemma~\ref{lem: L k estimate sharp} we conclude that $\normop{L_j}[Q_\Delta, Q_{22}] \le \EConst{L} \beta_j \gamma_j^{-1}$, where $\EConst{L} : = C_D(L_\infty) \CSEQ{a_{22}}$. This inequality and Jensen's inequality imply
\begin{align}
\Tr(\E [S_{k+1}^{(2)} (S_{k+1}^{(2)})^\top]) 
& \le (\EConst{L})^2 \pth  \normop{A_{12}}[Q_{22}, Q_\Delta]^2 \bigg(\sum_{j=0}^{k}  \beta_j \kappa_j \prod_{\ell=j+1}^k (1 - a_\Delta \beta_{\ell})  \{\E[\norm{\ttheta_j}^2]\}^{1/2}\bigg)^2 \nonumber \\
&\le 
\dth (\EConst{L})^2 \pth a_\Delta^{-1} \normop{A_{12}}[Q_{22}, Q_\Delta]^2  \sum_{j=0}^{k}  \beta_j \kappa_j^2 \prod_{\ell=j+1}^k (1 - a_\Delta \beta_{\ell}) \mth_j \nonumber
\end{align}
Applying~\eqref{eq: theta crude bound} and Lemma~\ref{cor:rate-of-convergence} we get
\begin{align}
\Tr(\E [S_{k+1}^{(2)} (S_{k+1}^{(2)})^\top]) 
\le \EConst{2,0} \prod_{\ell = 0}^k (1 - (a_\Delta/4) \beta_\ell) {\rm V}_0 \kappa_{k+1}^2 + \EConst{2,1} \beta_{k+1} \kappa_{k+1}^2\} \nonumber
\end{align}
where 
\begin{align}
    \EConst{2,0}&: = \dth a_\Delta^{-1} (\EConst{L})^2 \pth  \normop{A_{12}}[Q_{22}, Q_\Delta]^2 \EConst{0,\theta} \CSEQ{a_\Delta/2}/(1 - a_\Delta \beta_\infty^{\sf exp}) , \nonumber \\
    \EConst{2,1}&: = \dth a_\Delta^{-1} (\EConst{L})^2 \pth  \normop{A_{12}}[Q_{22}, Q_\Delta]^2 \EConst{1,\theta} \CSEQ{a_\Delta} \nonumber 
\end{align}
To estimate the next term  we rewrite it as follows
\begin{align} \nonumber
S_{k+1}^{(4)} = \sum_{\ell=0}^{k} \beta_\ell N_{\ell+1,k} \delta_\ell^{(2)},
\end{align}
where 
$$
N_{\ell+1,k} := \sum_{j=\ell+1}^{k} \beta_j \ProdBB_{j+1:k}^{(1)} A_{12}  \ProdBB_{\ell+1:j-1}^{(2)}. 
$$
It is straightforward to check
\beq
\label{eq: N bound}
\begin{split}
\normop{N_{\ell+1,k}} &\le \sqrt{\pw \pth}  \normop{A_{12}} \kappa_{\ell} \prod_{s = \ell+1}^{k} (1 - a_\Delta \beta_s) \sum_{j = \ell + 1}^{k} \gamma_j \prod_{s = \ell+1}^{j-1} (1 - (a_{22}/2) \gamma_s)  \\
&\le \EConst{N} \kappa_\ell (1 - a_\Delta \beta_\infty^{\sf exp})^{-1} \prod_{s = \ell+1}^{k-1} (1 - a_\Delta \beta_s) ,
\end{split}
\eeq
where $\EConst{N}: = (2/a_{22}) \sqrt{\pw \pth}  \normop{A_{12}} (1 - a_\Delta \beta_\infty^{\sf exp})^{-1}$ and we used  (compare with Lemma~\ref{lem:bsum})
\beq \label{eq: finite sum 0}
\sum_{j = \ell +1}^{k} \gamma_j \prod_{s = \ell+1}^{j-1} (1 - (a_{22}/2)  \gamma_s) = \frac{1}{a_{22}} \bigg\{1  - \prod_{s = \ell+1}^{k} (1 - (a_{22}/2) \gamma_s) \bigg\} \le 2/a_{22}. 
\eeq
Applying~\eqref{crude bounds w}, Jensen's inequality and observation
\begin{align}
\label{crude bounds w}
\E[\norm{\delta_\ell^{(2)}}^2] \le \EConst{22} \E[\norm{\tw_\ell}^2],  
\end{align}
where 
$\EConst{22}  :=\Cinfty^2  \normop{A_{12}}^2$, we obtain     
\begin{align}
\Tr (\E[S_{k+1}^{(4)} (S_{k+1}^{(4)})^\top]) &\le \dw a_\Delta^{-1} \EConst{22} (\EConst{N})^2    \sum_{\ell=0}^{k} \beta_\ell \kappa_\ell^2  \prod_{s = \ell+1}^k (1 - a_\Delta \beta_s) \mtw_\ell. \nonumber
\end{align}
Applying~\eqref{eq: w crude bound} and Lemma~\ref{cor:rate-of-convergence} we get
\begin{align}\label{eq: term 4 4}
\boxed{\Tr (\E[S_{k+1}^{(4)} (S_{k+1}^{(4)})^\top]) \le      \EConst{4,0} \prod_{s = 0}^k (1 - (a_\Delta/4) \beta_s) {\rm V}_0   + \EConst{4,0} \beta_{k+1} \kappa_{k+1},} 
\end{align}
where
\begin{align}
    \EConst{4,0}&: = \dw \EConst{0,w} \EConst{22} (\EConst{N})^2 (a_\Delta)^{-1} \CSEQ{a_\Delta/2} (1 - a_\Delta \beta_\infty^{\sf exp})^{-1} \kappa_{\infty}^2 , \nonumber \\
     \EConst{4,1}&: = \dw \EConst{1,w} \EConst{22} (\EConst{N})^2 a_\Delta^{-1} \CSEQ{a_\Delta}. \nonumber
\end{align}
To estimate the next term we proceed similarly. Using martingale property we obtain
\begin{align} \nonumber
\Tr (\E[S_{k+1}^{(5)} (S_{k+1}^{(5)})^\top]) \le \dth \EtmVW \Cinfty^2 \normop{A_{12}}^2  (\EConst{N})^2 \sum_{\ell=0}^{k} \beta_\ell^2 \kappa_\ell^2  \prod_{s = \ell+1}^k (1 - a_\Delta \beta_s)^2(\beta_\ell \gamma_\ell^{-1})^2  (1 +  \mth_\ell +  \mtw_\ell)
\end{align}
Hence, due to Corollary
\begin{align} \nonumber
\Tr (\E[S_{k+1}^{(5)} (S_{k+1}^{(5)})^\top]) \le \EConst{5,0} \prod_{\ell=0}^k (1 - (a_\Delta/4) \beta_\ell) {\rm V}_0\beta_{k+1} \kappa_{k+1}^2 + \EConst{5,1} \beta_{k+1} \kappa_{k+1}^2
\end{align}
where
\begin{align}
\EConst{5,0}&: =  \dth \EtmVW \Cinfty^2 \normop{A_{12}}^2  (\EConst{N})^2 \CSEQ{a_\Delta} (\EConst{0,w} + \EConst{0,\theta})/(1 - a_\Delta \beta_\infty^{\sf exp}), \nonumber \\
\EConst{5,1}&: = \dth \EtmVW \Cinfty^2 \normop{A_{12}}^2  (\EConst{N})^2 (1 + \EConst{1, w} \gamma_\infty^{\sf mtg} + \EConst{ 1   , \theta} \beta_\infty^{\sf exp}) \nonumber
\end{align}
It follows from the previous inequalities that
\begin{align}\label{eq: term 2 + 5}
\boxed{\E\big[\Tr(\{S_{k+1}^{(2)}  + S_{k+1}^{(5)}\} \{S_{k+1}^{(2)} +  S_{k+1}^{(5)}\}^\top)   \le \EConst{2+5, 0 } \prod_{\ell=0}^k (1 - (a_\Delta/4) \beta_\ell) {\rm V}_0  \kappa_{k+1}^2 + \EConst{2+5, 1 }\beta_{k+1} \kappa_{k+1}^2}
\end{align}
where 
\begin{align}
\EConst{2+5, 0}&:= 2 \EConst{5,0} \beta_\infty^{\sf exp} + 2 \EConst{2,0}, \nonumber \\
\EConst{2+5, 1}&:= 2 \EConst{5,1} + 2 \EConst{2,1} \nonumber
\end{align}
For the term $\E\big[\Tr(S_{k+1}^{(3)} \{S_{k+1}^{(2)} +  S_{k+1}^{(5)}\}^\top) \big]$ we write
\begin{align}
&\E\big[\Tr(S_{k+1}^{(3)} \{S_{k+1}^{(2)} +  S_{k+1}^{(5)}\}^\top) \big] \le \E^{1/2}\big[\Tr(S_{k+1}^{(3)} (S_{k+1}^{(3)})^\top \big]\E^{1/2}\big[\Tr(\{S_{k+1}^{(2)} +  S_{k+1}^{(5)}\} \{S_{k+1}^{(2)} +  S_{k+1}^{(5)}\}^\top \big] \nonumber  \\
&\qquad\qquad \le \big\{\EConst{3,0}\prod_{\ell=0}^k (1 - (a_\Delta/4) \beta_\ell) {\rm V}_0 \beta_{k+1} + \EConst{3,1} \beta_{k+1} \gamma_{k+1} + \pth \CSEQ{a_\Delta}   \Tr(\Sigma) \beta_{k+1} \big\}^{1/2} \nonumber \\
&\qquad\qquad \times \big\{\EConst{2+5, 0 } \prod_{\ell=0}^k (1 - (a_\Delta/4) \beta_\ell) {\rm V}_0  \kappa_{k+1}^2 + \EConst{2+5, 1 }\beta_{k+1} \kappa_{k+1}^2 \big\}^{1/2} \nonumber
\end{align}
We obtain 
\begin{align}\label{eq: term 3 2+5}
\boxed{\E\big[\Tr(S_{k+1}^{(3)} \{S_{k+1}^{(2)} +  S_{k+1}^{(5)}\}^\top) \big] \le \EConst{3/2+5,0} \prod_{\ell=0}^k (1 - (a_\Delta/4) \beta_\ell) {\rm V}_0  + \EConst{3/2+5,1} \beta_{k+1} \kappa_{k+1},} 
\end{align}
where 
\begin{align}
    \EConst{3/2+5,0}&:= (\EConst{2+5, 0} +  \EConst{3, 0} + \EConst{3, 0} \EConst{2+5, 0})/2, \nonumber \\
    \EConst{3/2+5,1}&:= \kappa_\infty^{\sf exp}/2 +  \EConst{2+5, 1} \beta_\infty^{\sf exp} \kappa_\infty^{\sf exp}/2 + (\EConst{3,1} \gamma_\infty^{\sf mtg} + \pth \CSEQ{a_\Delta}   \Tr(\Sigma) ) \kappa_\infty^{\sf exp}/2 \nonumber \\
    &+ (\EConst{2+5, 1} (\EConst{3,1} \gamma_\infty^{\sf mtg} + \pth \CSEQ{a_\Delta}   \Tr(\Sigma) ))^{1/2} \nonumber
\end{align}
Let us consider the term
\begin{align}
\label{eq: 3/4 term start}
\E[\Tr(S_{k+1}^{(3)} S_{k+1}^{(4)})^\top] = \sum_{j=0}^{k} \beta_j \sum_{\ell=j+1}^{k} \beta_\ell \Tr(\ProdBB_{j+1:k}^{(1)} Z_{j+1} \tw_\ell^\top A_{12}^\top C_\ell^\top N_{\ell+1, k}^\top),
\end{align}
where $Z_{j+1}: = V_{j+1} + A_{12}A_{22}^{-1} W_{j+1}$. 
For $\tw_\ell$ we can use the following expansion
$$
\tw_\ell = \ProdBB_{0:\ell-1}^{(2)} \tw_0 + \sum_{i=0}^{\ell-1} \gamma_i \ProdBB_{i+1:\ell-1}^{(2)} \widetilde Z_{i+1} + \sum_{i=0}^{\ell-1} \beta_i \ProdBB_{i+1:\ell-1}^{(2)} C_{i} A_{12} \tw_i,  
$$
where $\widetilde Z_{i+1} := W_{i+1} + \kappa_i C_i V_{i+1}$
Substituting this expansion into r.h.s of~\eqref{eq: 3/4 term start} and repeating this procedure until  $\E[Z_{j+1}\tw_l^\top] = \gamma_j\E[Z_{j+1} \widetilde Z_{j+1}^\top (\ProdBB_{j+1:\ell_1-1}^{(2)})^\top$ we come to the following expansion of~\eqref{eq: 3/4 term start}
\begin{align}
& \sum_{\ell=j+1}^{k} \beta_\ell \Tr(\ProdBB_{j+1:k}^{(1)} Z_{j+1} \tw_\ell^\top A_{12}^\top C_\ell^\top N_{\ell+1, k}^\top) \nonumber  \\
 & = \sum_{\ell_1=j+1}^{k} \beta_{\ell_1} \sum_{\ell_2 = j+1}^{\ell_1 - 1} \beta_{\ell_2}\Tr(\ProdBB_{j+1:k}^{(1)} Z_{j+1} \tw_{\ell_2}^\top A_{12}^\top C_{\ell_2}^\top (\ProdBB_{\ell_2+1:\ell_1-1}^{(2)})^\top A_{12}^\top C_{\ell_1}^\top N_{\ell_1+1, k}^\top) \nonumber  \\
 & + \gamma_j \sum_{\ell_1=j+1}^{k} \beta_{\ell_1} \Tr(\ProdBB_{j+1:k}^{(1)} Z_{j+1} \widetilde Z_{j+1}^\top (\ProdBB_{j+1:\ell_1-1}^{(2)})^\top A_{12}^\top C_{\ell_1}^\top N_{\ell_1+1, k}^\top) \nonumber \\
 & = \gamma_j \sum_{s = 1}^{k-j}   \sum_{\ell_1 = j+1}^{k} \beta_{\ell_1}  \sum_{\ell_2 = 1}^{ \ell_1-1} \beta_{\ell_2} \ldots  \sum_{\ell_s = j+1}^{l_{s-1} - 1} \beta_{\ell_{s}} \nonumber \\
 &\qquad\qquad \times \Tr ( \ProdBB_{j+1:k}^{(1)} Z_{j+1} \widetilde Z_{j+1}^\top ( \ProdBB_{j+1:\ell_s-1}^{(2)})^\top  A_{12}^\top C_{\ell_{s}}^\top ( \ProdBB_{\ell_{s}+1:\ell_{s-1}-1}^{(2)})^\top  A_{12}^\top C_{\ell_{s}}^\top \ldots  (\ProdBB_{\ell_2+1:\ell_1-1}^{(2)})^\top A_{12}^\top C_{\ell_1}^\top N_{\ell_1+1, k}^\top) \nonumber
\end{align}
where $\ell_0 := k+1$. Using iteratively Corollary  and estimate~\eqref{eq: N bound} for $N_{\ell_1+1, k}$ we obtain the following bound
\begin{align}
    &\bigg \| \sum_{\ell=j+1}^{k} \beta_\ell \Tr(\ProdB_{j:k-1}^{(1)} Z_{j+1} \tw_\ell^\top A_{12}^\top C_\ell^\top N_{\ell+1, k}^\top) \bigg  \| \nonumber \\
    &\qquad\qquad\le \Cinfty \CSEQ{a_{22}}\EConst{N} \beta_j \kappa_j \prod_{\ell = j+1}^k (1 - a_\Delta \beta_\ell) \E^{1/2}[\norm{Z_{j+1}}^2] \E^{1/2}[\norm{\widetilde Z_{j+1}}^2] \sum_{s = 1}^{k-j} (\kappa_{j} \normop{A_{12}} \Cinfty \CSEQ{a_{22}})^{s-1} \nonumber
\end{align}
Since $\kappa_\infty^{\sf exp} \le (1/2) (\normop{A_{12}} \Cinfty \CSEQ{a_{22}})^{-1}$ and
$$
\E^{1/2}[\norm{Z_{j+1}}^2] \E^{1/2}[\norm{\widetilde Z_{j+1}}^2] \le 2 \EtmVW (1 + \normop{A_{12}A_{22}^{-1}})(1 + \sqrt{\kappa_\infty^{\sf exp} \Cinfty}) (1 + \mth_j + \mtw_j)
$$
we obtain that 
\begin{align}
    &\bigg \| \sum_{\ell=j+1}^{k} \beta_\ell \Tr(\ProdBB_{j+1:k}^{(1)} Z_{j+1} \tw_\ell^\top A_{12}^\top C_\ell^\top N_{\ell+1, k-1}^\top) \bigg  \| \nonumber \\
    &\qquad\qquad \le \Cinfty \EConst{N} 2 \EtmVW (1 + \normop{A_{12}A_{22}^{-1}})(1 + \sqrt{\kappa_\infty^{\sf exp} \Cinfty}) \beta_j \kappa_j \prod_{s = j+1}^{k} (1 - a_\Delta \beta_s) (1 + \mth_j + \mtw_j) \nonumber
\end{align}
This inequality and~\eqref{eq: 3/4 term start} together imply
\begin{align}
&\big|\E[\Tr(S_{k+1}^{(3)} S_{k+1}^{(4)})^\top] \big| \nonumber \\
&\qquad\qquad \le \Cinfty \EConst{N} 2 \EtmVW (1 + \normop{A_{12}A_{22}^{-1}})(1 + \sqrt{\kappa_\infty^{\sf exp} \Cinfty}) \sum_{j=0}^{k} \beta_j^2 \kappa_j \prod_{s = j+1}^k (1 - a_\Delta \beta_s) (1 + \mth_{j} + \mtw_j) \nonumber
\end{align}
Finally, the standard arguments will lead to
\begin{align}\label{eq: term 3 4}
\boxed{\big|\E[\Tr(S_{k+1}^{(3)} S_{k+1}^{(4)})^\top] \big|  \le \EConst{3/4,0} \prod_{\ell=0}^k (1 - (a_\Delta/4) \beta_\ell) V_0  + \EConst{3/4,0}\beta_{k+1} \kappa_{k+1},} 
\end{align}
where
\begin{align}
    \EConst{3/4,0}&: = \Cinfty \EConst{N} 2 \EtmVW (1 + \normop{A_{12}A_{22}^{-1}})(1 + \sqrt{\kappa_\infty^{\sf exp} \Cinfty}) \CSEQ{a_\Delta/2} (\EConst{0,w} + \EConst{0,\theta})/(1 - a_\Delta \beta_\infty^{\sf exp}) \beta_\infty^{\sf exp} \kappa_\infty^{\sf exp}, \nonumber \\
     \EConst{3/4,1}&: = \Cinfty \EConst{N} 2 \EtmVW (1 + \normop{A_{12}A_{22}^{-1}})(1 + \sqrt{\kappa_\infty^{\sf exp} \Cinfty}) \CSEQ{a_\Delta} (1 + \EConst{1,w} \gamma_\infty^{\sf mtg} + \EConst{1,\theta} \beta_\infty^{\sf exp}) \nonumber
\end{align}
Finally, we estimate all terms involving $S_{k+1}^{(6)}$. We rewrite $S_{k+1}^{(6)}$ as follows
$$
S_{k+1}^{(6)} = \sum_{\ell=0}^{k} \beta_l  M_{\ell,k} W_{\ell+1}
$$
where we defined
$$
M_{\ell,k} :=    \gamma_\ell \beta_\ell^{-1} \sum_{j=\ell+1}^{k} \beta_j \ProdBB_{j+1:k}^{(1)}  A_{12}   \ProdBB_{\ell+1:j-1}^{(2)}  -   \ProdBB_{\ell+1:k}^{(1)} A_{12} A_{22}^{-1}.
$$
Using martingale property we obtain 
\begin{align}
\Tr (\E[S_{k+1}^{(6)} (S_{k+1}^{(6)})^\top]) \le \sum_{\ell=0}^{k} \beta_\ell^2 \normop{M_{\ell,k}}^2 \E[\norm{W_{\ell+1}}^2] \nonumber
\end{align}
We rewrite $M_{\ell,k}$ as follows
\begin{align}
M_{\ell,k} &=\sum_{j=\ell+1}^{k} \gamma_j \bigg[ \frac{\beta_j \gamma_l}{\beta_l \gamma_j} \Id - \ProdBB_{\ell+1:j}^{(1)}  \bigg] \ProdBB_{j+1:k}^{(1)}  A_{12}   \ProdBB_{\ell+1:j-1}^{(2)}  + \ProdBB_{\ell+1:k}^{(1)} A_{12} \bigg(\sum_{j=l+1}^{k} \gamma_j \ProdBB_{\ell+1:j-1}^{(2)} - A_{22}^{-1}\bigg) \nonumber
\end{align}
Since,
$$
\sum_{j=l+1}^{k} \gamma_j \ProdB_{\ell+1:j-1}^{(2)} = A_{22}^{-1} \bigg ( \Id  - \ProdBB_{\ell+1:k}^{(2)} \bigg)
$$
this equation leads to  
\begin{align}
M_{\ell,k} &=\sum_{j=\ell+1}^{k} \gamma_j \bigg[ \frac{\beta_j \gamma_l}{\beta_l \gamma_j} \Id - \ProdBB_{\ell+1:j}^{(1)}  \bigg] \ProdBB_{j+1:k}^{(1)}  A_{12}   \ProdBB_{\ell+1:j-1}^{(2)}    - \ProdBB_{\ell+1:k}^{(1)} A_{12} A_{22}^{-1} \ProdBB_{\ell+1:k}^{(2)} \nonumber
\end{align}
We rewrite the term in the square brackets as follows 
$$
 \frac{\beta_j \gamma_l}{\beta_l \gamma_j} \Id - \ProdBB_{\ell+1:j}^{(1)} = \prod_{s=\ell+1}^{j} \frac{\kappa_{s}}{\kappa_{s-1}}\Id -  \prod_{s=\ell+1}^{j} (\Id - \beta_s \Delta) = \sum_{t = \ell+1}^{j} \frac{\kappa_{t-1}}{\kappa_l}\Id \big\{\beta_t \Delta + (\kappa_{t}/\kappa_{t-1} -1)  \Id \big\} \ProdBB_{t+1:j}^{(1)}
$$
Using assumption~\ref{assum:stepsize} we may show that 
$$
|\kappa_{t}/\kappa_{t-1} -1| \le  (a_\Delta/16) \beta_t
$$
Taking norm of the both sides of the previous equation we obtain
$$
\bigg \|\frac{\beta_j \gamma_l}{\beta_l \gamma_j} \Id - \ProdBB_{\ell+1:j}^{(1)} \bigg\| \le \sqrt{p_\Delta} (\normop{\Delta} + (a_\Delta/16))\kappa_l^{-1} \sum_{t = \ell+1}^{j} \beta_t \kappa_{t-1} \prod_{s = t+1}^{j}(1 - a_\Delta \beta_s)
$$
Finally we arrive at the following bound for $M_{\ell,k}$
\beq
\label{eq: M l k bound}
\begin{split}
&\normop{M_{\ell,k}} \le  \sqrt{\pw \pth} \normop{A_{12}A_{22}^{-1}} \prod_{s = \ell+1}^k (1 - a_\Delta \beta_s) \prod_{s = \ell+1}^{k} (1 - a_{22} \gamma_s) \\
& + \sqrt{\pth \pw} \bigg(\normop{\Delta} +\frac{a_\Delta}{16}\bigg)\kappa_l^{-1} \sum_{j = \ell+1}^{k} \gamma_j  \sum_{t = \ell+1}^{j} \beta_t \kappa_{t-1} \prod_{s = t+1}^{k}(1 - a_\Delta \beta_s) \prod_{s = \ell+1}^{j-1}(1 - a_{22} \gamma_s) 
\end{split}
\eeq
This bound will yield
\begin{align}
\Tr (\E[S_{k+1}^{(6)} (S_{k+1}^{(6)})^\top]) &\le 2  \pw \pth \normop{A_{12}A_{22}^{-1}}^2 \sum_{\ell=0}^{k}  \beta_{\ell}^2 \prod_{s = \ell+1}^k (1 - a_\Delta \beta_s) \prod_{s = \ell+1}^{k} (1 - a_{22} \gamma_s)    \E[\norm{W_{\ell+1}}^2] \nonumber \\
& + 2\pth \pw (\normop{\Delta} + (a_\Delta/16))^2
\sum_{\ell=0}^{k} \gamma_{\ell}^2 
\E[\norm{W_{\ell+1}}^2] \prod_{s = \ell+1}^{k}(1 - a_\Delta \beta_s) \nonumber \\
&\times \bigg \{ \sum_{j = \ell+1}^{k} \gamma_j \prod_{s = \ell+1}^{j-1}(1 - (a_{22}/2) \gamma_s)  \sum_{t = \ell+1}^{j} \beta_t \kappa_{t-1} \prod_{s = t+1}^{j}(1 - a_\Delta \beta_s)  \bigg \}^2 \nonumber \\
& =: A_1 + A_2
\end{align}
The estimate of $A_1$ follows from Lemma~\ref{cor:rate-of-convergence} 
\begin{align}\label{eq: term 6 1 estimate}
    A_1 \le \EConst{6,0,2} \prod_{\ell=0}^k (1 - (a_\Delta/4) \beta_\ell){\rm V}_0 
    + \EConst{6,1,1} \beta_{k+1} \kappa_{k+1}. 
\end{align}
where 
\begin{align}
    \EConst{6,0,1}&:= 2 \dw \EtmVW  \pw \pth \normop{A_{12}A_{22}^{-1}}^2 \CSEQ{a_{22}}(\EConst{0,w} + \EConst{0,\theta}) \beta_\infty^{\sf exp} \kappa_\infty^{\sf exp} /(1 - a_\Delta \beta_\infty^{\sf exp})  , \nonumber \\
    \EConst{6,1,1}&:= 2 \dw \EtmVW  \pw \pth \normop{A_{12}A_{22}^{-1}}^2 (1 + \EConst{1, w}\gamma_\infty^{\sf mtg} + \EConst{1, \theta} \beta_\infty^{\sf exp}) \CSEQ{a_{22}} \nonumber
\end{align}
Inequality~\eqref{eq: finite sum 0} and 
Jensen's inequality together imply
\begin{align}
A_2 &\le  2(2/a_{22})^2 \pth \pw (\normop{\Delta} + (a_\Delta/16))^2
\sum_{\ell=0}^{k} \gamma_\ell^2
\E[\norm{W_{\ell+1}}^2] \prod_{s = \ell+1}^{k}(1 - a_\Delta \beta_s) \nonumber \\
&\times  \sum_{j = \ell+1}^{k} \gamma_j \bigg \{ \sum_{t = \ell+1}^{j} \beta_t \kappa_{t-1} \prod_{s = t+1}^{j}(1 - a_\Delta \beta_s)  \bigg \}^2 \prod_{s = \ell+1}^{j-1}(1 - (a_{22}/2) \gamma_s) \nonumber
\end{align} 
Similarly, applying Jensen's inequality for the second time we come to the following inequality
\begin{align}
A_2 &\le  2(2/a_{22})^2 (1/a_\Delta)^2 \pth \pw (\normop{\Delta} + (a_\Delta/16))^2
\sum_{\ell=0}^{k} \gamma_\ell^2
\E[\norm{W_{\ell+1}}^2] \prod_{s = \ell+1}^{k}(1 - a_\Delta \beta_s) \nonumber  \\
&\times  \sum_{j = \ell+1}^{k} \gamma_j \sum_{t = \ell+1}^{j} \beta_t \kappa_{t-1}^2 \prod_{s = t+1}^{j}(1 - a_\Delta \beta_s)  \prod_{s = \ell+1}^{j-1}(1 - (a_{22}/2) \gamma_s) \nonumber
\end{align} 
Changing the order of summation we obtain
\begin{align}
A_2 &\le  2(2/a_{22})^2 (1/a_\Delta)^2 \pth \pw (\normop{\Delta} + (a_\Delta/16))^2
\sum_{j = 0}^{k} \gamma_j \prod_{s = j+1}^{k}(1 - a_\Delta \beta_s) \sum_{t = 0}^{j} \beta_t \kappa_{t-1}^2  \prod_{s = t+1}^{j}(1 - (a_{22}/2) \gamma_s) \nonumber \\
& \times \sum_{\ell=0}^{t-1} \gamma_\ell^2 \prod_{s = \ell+1}^{t-1}(1 - (a_{22}/2) \gamma_s)
   \E[\norm{W_{\ell+1}}^2] \nonumber
   \end{align} 
Finally, estimating $\E[\norm{W_{\ell+1}}^2]$ by \eqref{eq: theta crude bound} and~\eqref{eq: w crude bound} and applying Lemma~\ref{cor:rate-of-convergence} we obtain 
\begin{align}\label{eq: term 6 2 estimate}
A_2 &\le \EConst{6,0,2} \prod_{\ell=0}^k (1 - (a_\Delta/4) \beta_\ell) {\rm V}_0  +  \EConst{6,1,2} \beta_{k+1} \kappa_{k+1}, 
\end{align} 
where 
\begin{align}
  \EConst{6,0,2} &: = \dw \EtmVW (\CSEQ{a_{22}/4})^2 \CSEQ{a_\Delta/2} 2(2/a_{22})^2 (1/a_\Delta)^2 \pth \pw (\normop{\Delta} + (a_\Delta/16))^2 \nonumber \\
  & \qquad\qquad\qquad \times (\EConst{0,w} + \EConst{0,\theta}) (\kappa_\infty^{\sf exp})^2 \gamma_\infty^{\sf mtg}/(1 - a_\Delta \beta_\infty^{\sf exp})^2, \nonumber \\
 \EConst{6,1,2} &: =  \dw \EtmVW (\CSEQ{a_{22}/2})^2 \CSEQ{a_\Delta} 2(2/a_{22})^2 (1/a_\Delta)^2 \pth \pw (\normop{\Delta} + (a_\Delta/16))^2   (1 + \EConst{1, w} \gamma_\infty^{\sf mtg} + \EConst{1,\theta} \beta_\infty^{\sf exp}) \nonumber
\end{align}
We conclude from~\eqref{eq: term 6 1 estimate} and~\eqref{eq: term 6 2 estimate} that
\begin{align}\label{eq: term 6 6}
\boxed{\Tr (\E[S_{k+1}^{(6)} (S_{k+1}^{(6)})^\top]) \le \EConst{6,0} \prod_{\ell=0}^k (1 - (a_\Delta/4) \beta_\ell) {\rm V}_0 +  \EConst{6,1} \beta_{k+1} \kappa_{k+1},} 
\end{align}
where
\begin{align} \nonumber
  \EConst{6,0} &: = \EConst{6,0, 1} + \EConst{6,0, 2},  \\
 \EConst{6,1} &: = \EConst{6,1,1} + \EConst{6,1,2}
\end{align}
It remains to consider $\E\big[\Tr(S_{k+1}^{(3)} (S_{k+1}^{(6)})^\top)\big]$. We proceed similarly and use~\eqref{eq: M l k bound} to get
\begin{align} \nonumber
\big|\E\big[\Tr(S_{k+1}^{(3)} (S_{k+1}^{(6)})^\top)\big] \big| &\le \sum_{\ell=0}^{k} \beta_\ell^2 \normop{M_{\ell,k}} \prod_{s=\ell+1}^k (1 - a_\Delta \beta_s) \E^{1/2} [\norm{Z_{\ell+1}}^2]\E^{1/2} [\norm{W_{\ell+1}}^2] \\
&= A_1' + A_2'\nonumber
\end{align}
The following estimate holds for $A_1'$
\begin{align}
|A_1'| &\le \sqrt{\pw \pth} \normop{A_{12}A_{22}^{-1}} \max(\dth, \dw) ( 2+2 \normop{A_{12} A_{22}^{-1}})^{1/2} \sum_{\ell=0}^{k} \beta_\ell^2  \prod_{s = \ell+1}^k(1 - a_\Delta \beta_s) \nonumber \\
&\qquad\qquad\qquad\qquad\times\prod_{s = \ell+1}^k (1 - a_{22} \gamma_s) (1 + \mth_\ell + \mtw_\ell) \nonumber
\end{align}
Applying standard arguments we get
\begin{align}\label{eq: term 3 6 1 estimate}
A_{1}' \le \EConst{3/6,0,1} \prod_{\ell=0}^k (1 - (a_\Delta/4) \beta_\ell) {\rm V}_0 +  \EConst{3/6,1,1} \beta_{k+1} \kappa_{k+1}, 
\end{align}
where
\begin{align}
   \EConst{3/6,0,1}&:= \EtmVW\sqrt{\pw \pth} \normop{A_{12}A_{22}^{-1}} \max(\dth, \dw) ( 2+2 \normop{A_{12} A_{22}^{-1}})^{1/2} \nonumber \\
   &\qquad\qquad\qquad\qquad \times (\EConst{0,w} + \EConst{0,\theta}) \CSEQ{a_{22}} \beta_\infty^{\sf exp} \kappa_\infty^{\sf exp}/(1 - a_\Delta \beta_\infty^{\sf exp}), \nonumber \\ 
   \EConst{3/6,1,1}&:= 2\EtmVW\sqrt{\pw \pth} \normop{A_{12}A_{22}^{-1}} \max(\dth, \dw) ( 2+2 \normop{A_{12} A_{22}^{-1}})^{1/2} (1 + \EConst{1,w} \gamma_\infty^{\sf mtg} + \EConst{1,\theta} \beta_\infty^{\sf exp}) \CSEQ{a_{22}} \nonumber
\end{align}
For $A_2'$ we write the following bound
\begin{align}
A_2' &\le \sqrt{\pth \pw} (\normop{\Delta} + (a_\Delta/16))\max(\dth, \dw) ( 2+2 \normop{A_{12} A_{22}^{-1}})^{1/2}
\sum_{\ell=0}^{k} \gamma_{\ell}^2 
(1 + \mth_\ell + \mtw_\ell) \prod_{s = \ell+1}^{k}(1 - a_\Delta \beta_s) \nonumber \\
&\times \bigg \{ \sum_{j = \ell+1}^{k} \gamma_j \prod_{s = \ell+1}^{j-1}(1 - a_{22} \gamma_s)  \sum_{t = \ell+1}^{j} \beta_t \kappa_{t-1} \prod_{s = t+1}^{j}(1 - a_\Delta \beta_s)  \bigg \} \nonumber
\end{align}
Changing the order of summation and applying arguments from the estimation of $A_2$ we come to the following bound
\begin{align}\label{eq: term 3 6 2 estimate}
A_2' &\le \EConst{3/6,0,2} \prod_{\ell=0}^k (1 - (a_\Delta/4) \beta_\ell) {\rm V}_0  +  \EConst{3/6,1,2} \beta_{k+1} \kappa_{k+1}, 
\end{align} 
where 
\begin{align}
  \EConst{3/6,0,2} &: = \sqrt{\pth \pw} (\normop{\Delta} + (a_\Delta/16))\max(\dth, \dw) ( 2+2 \normop{A_{12} A_{22}^{-1}})^{1/2} \nonumber \\
  &\qquad\qquad\qquad\qquad \times (\EConst{0,w} + \EConst{0,\theta}) (\CSEQ{a_{22}/4})^2 \CSEQ{a_{\Delta}/2}  \gamma_\infty^{\sf mtg} (\kappa_\infty^{\sf exp})^2 /(1 - a_\Delta \beta_\infty^{\sf exp})^2 , \nonumber \\
 \EConst{3/6,1,2} &: =  \sqrt{\pth \pw} (\normop{\Delta} + (a_\Delta/16))\max(\dth, \dw) ( 2+2 \normop{A_{12} A_{22}^{-1}})^{1/2}   (1 + \EConst{1,w} \gamma_\infty^{\sf mtg} + \EConst{1,\theta} \beta_\infty^{\sf exp}) (\CSEQ{a_{22}/2})^2 \CSEQ{a_{\Delta}} \nonumber
\end{align}
We conclude from~\eqref{eq: term 3 6 1 estimate} and~\eqref{eq: term 3 6 2 estimate} 
\begin{align}\label{eq: term 3 6}
\boxed{\big|\Tr (\E[S_{k+1}^{(3)} (S_{k+1}^{(6)})^\top]) \big| \le \EConst{3/6,0} \prod_{\ell=0}^k (1 - (a_\Delta/4) \beta_\ell) {\rm V}_0  +  \EConst{3/6,1} \beta_{k+1} \kappa_{k+1},} 
\end{align}
where
\begin{align}
  \EConst{3/6,0} &: = \EConst{3/6,0, 1} + \EConst{3/6,0, 2}, \nonumber  \\
 \EConst{3/6,1} &: = \EConst{3/6,1,1} + \EConst{3/6,1,2} \nonumber
\end{align}
\paragraph{Final estimate of the remainder term $J_{k+1}$}
Collecting \eqref{eq: second term of leading term exp}, \eqref{eq: term 0 + 1},\eqref{eq: term 4 4}, \eqref{eq: term 2 + 5}, \eqref{eq: term 3 2+5}, \eqref{eq: term 3 4}, \eqref{eq: term 6 6}, \eqref{eq: term 3 6}  we obtain the estimate of $J_{k+1}: = R_{k+1} + J_{k+1}'$
\begin{align}
\boxed{|J_{k+1}| \le \EConst{0} \prod_{\ell=0}^k (1 - (a_\Delta/4) \beta_\ell){\rm V}_0 + \EConst{0} \beta_{k+1} (\gamma_{k+1} +  \kappa_{k+1}),} \nonumber
\end{align}
where
\beq
 \label{eq: exp const 0 1}
\begin{split}
  \EConst{0}&:= \EConst{3,0} \beta_\infty^{\sf exp}  + 3\EConst{0/1} + 5\EConst{4,0}  + 5\EConst{2+5,0} ( \kappa_\infty^{\sf exp})^2  + 2\EConst{3/2+5,0} + 2 \EConst{3/4,0} \\
  &+ 5\EConst{6,0}  + 2\EConst{3/6,0} , \\
  \EConst{1}&:= \EConst{3,1} + 5\EConst{4,1} + 5\EConst{2+5,1} + 2\EConst{3/2+5,1} + 2\EConst{3/4,1} + 5\EConst{6,1 } + 2\EConst{3/6,1}
\end{split}
\eeq
Hence, we obtained~\eqref{eq:remainder term of exp}.


\section{Auxiliary Lemmas}\label{app:aux}

\begin{lemma} \label{lem:bsum}
Let $a > 0$ and $(\gamma_k)_{k \geq 0}$ be a nonincreasing sequence such that $\gamma_0 < 1/a$. Then, for any integer $k \geq 1$, 
\[
\sum_{j=0}^{k-1} \gamma_j \prod_{l=j+1}^{k-1} (1 - \gamma_l a) = \frac{1}{a} \left\{1  - \prod_{l=0}^{k-1} (1 - \gamma_l a) \right\}
\]
\end{lemma}
\begin{remark}
If $k_0$ is such that $\sum_{l = 0}^{k_0-1} \gamma_l \geq \log(2)/a $ then the r.h.s. of the previous equation is lower bounded by $1/(2a)$ for any $k \geq k_0$.
\end{remark}
\begin{proof}
Let us denote $u_{j:k-1} = \prod_{l = j}^{k-1} (1 - \gamma_l a)$. Then, for $j \in\{0,\dots,k-1\}$,  
$u_{j+1:k-1} - u_{j:k-1} = a \gamma_j u_{j+1}$. Hence, 
\[
\sum_{j=0}^{k-1} \gamma_j \prod_{l=j+1}^{k-1} (1 - \gamma_l a) = \frac{1}{a} \sum_{j=0}^{k-1} (u_{j+1:k-1} - u_{j:k-1}) = a^{-1} ( 1 - u_{0:k-1} ) \,.
\]
\end{proof}


\begin{lemma}
\label{cor:rate-of-convergence}
Assume A\ref{assum:stepsize} and set
\begin{equation}
\label{eq:definition-Cseq}
\CSEQ{a} = \frac{2}{a} \varsigma \max\{1, a_{22}/(4a_\Delta)\} \vee \frac{4}{a}  \, \varsigma^3.
\end{equation}
The following holds 
\begin{enumerate}[label=(\roman*)]
    \item For any $a \in [a_{22}/4,\gamma_0^{-1}]$ and $k \in \nset$, if in addition, we have $\kappa \leq a_{22} / (4a_\Delta)$, then
\[
\sum_{j=0}^{k-1} \gamma^2_j \prod_{l=j+1}^{k-1} (1 - \gamma_l a)  \leq \CSEQ{a} \gamma_k,~~ \sum_{j=0}^{k-1} \beta_j \gamma_j \prod_{l=j+1}^{k-1} (1 - \gamma_l a)  \leq \CSEQ{a} \beta_k,~~
\sum_{j=0}^{k-1} \beta_j^2  \prod_{l=j+1}^{k-1} (1 - \gamma_l a)  \leq \CSEQ{a} \beta_k
\]
\item For any $a \in [a_\Delta/8, \beta_0^{-1}]$ and $k \in \nset$,
\[
\sum_{j=0}^{k-1} \beta_j \gamma_j \prod_{\ell=j+1}^{k-1} (1 - a \beta_\ell ) \leq \CSEQ{a} \gamma_{k}, \quad \sum_{j=0}^{k-1} \beta_j^2 \prod_{\ell=j+1}^{k-1} (1 - a \beta_\ell ) \leq \CSEQ{a} \beta_{k}
\]
\item For any $a \in [a_\Delta/4, \beta_0^{-1}]$ and $k \in \nset$,
\[
\sum_{j=0}^{k-1} \beta_j^3/\gamma_j \prod_{\ell=j+1}^{k-1} (1 - a \beta_\ell ) \leq \CSEQ{a} \beta_{k}^2/\gamma_k, \quad \sum_{j=0}^{k-1} \beta_j^4/\gamma_j^2 \prod_{\ell=j+1}^{k-1} (1 - a \beta_\ell ) \leq \CSEQ{a} \beta_{k}^3/\gamma_k^2,
\]
\[
\sum_{j=0}^{k-1} \beta_j^2 \gamma_j \prod_{\ell=j+1}^{k-1} (1 - a \beta_\ell ) \leq \CSEQ{a} \beta_{k} \gamma_k
\]
\item For any $a \in [a_{22}/4, \gamma_0^{-1}]$ and $k \in \nset$,
\[
\sum_{j=0}^{k-1} \beta_j \prod_{l=j+1}^{k-1} (1 - \gamma_l a)  \leq \CSEQ{a} \beta_k/\gamma_k, \quad \sum_{j=0}^{k-1} \beta_j^2 \prod_{l=j+1}^{k-1} (1 - \gamma_l a)  \leq \CSEQ{a} \beta_k^2/\gamma_k, \quad \sum_{j=0}^{k-1} \beta_j^3/\gamma_j \prod_{l=j+1}^{k-1} (1 - \gamma_l a)  \leq \CSEQ{a} \beta_k^3/\gamma_k^2
\]
\end{enumerate}
\end{lemma}

\begin{proof}
Part i) of the corollary, consider the first inequality and observe that
\[
\begin{split}
& \sum_{j=0}^{k-1} \gamma^2_j \prod_{l=j+1}^{k-1} (1 - \gamma_l a) = \gamma_k \sum_{j=0}^{k-1} \frac{\gamma_{k-1}}{\gamma_k} \frac{\gamma_j}{\gamma_{k-1}} \gamma_j \prod_{l=j+1}^{k-1} (1 - \gamma_l a) \leq \varsigma \gamma_k \sum_{j=0}^{k-1} \gamma_j \prod_{l=j+1}^{k-1} \frac{\gamma_{l-1}}{\gamma_l} (1 - \gamma_l a)
\end{split}
\]
Note that as $a \geq a_{22}/4$, we have
\[
\frac{\gamma_{l-1}}{\gamma_l} (1 - \gamma_l a) \leq (1 + \frac{a_{22}}{8} \gamma_l ) ( 1 - \gamma_l a ) \leq 1 - a \gamma_l / 2
\]
Substituting into the above inequality yields
\[
\sum_{j=0}^{k-1} \gamma^2_j \prod_{l=j+1}^{k-1} (1 - \gamma_l a) \leq 
\varsigma \gamma_k \sum_{j=0}^{k-1} \gamma_j \prod_{l=j+1}^{k-1} (1 - \gamma_l a / 2) \leq \varsigma \frac{2}{a} \gamma_k
\]
where we have applied Lemma~\ref{lem:bsum} in the last inequality.
Next and applying similar steps as before, we observe that 
\[
\sum_{j=0}^{k-1} \beta_j \gamma_j \prod_{l=j+1}^{k-1} (1 - \gamma_l a) \leq \varsigma \beta_k \sum_{j=0}^{k-1} \gamma_j \prod_{l=j+1}^{k-1} \frac{ \beta_{l-1} }{ \beta_l } (1 - \gamma_l a)  
\]
As we have 
\[
\frac{ \beta_{l-1} }{ \beta_l } (1 - \gamma_l a) \leq 1 - \gamma_l ( a - \kappa a_\Delta / 16 ) \leq 1 - \gamma_l ( a -  a_{22} / 64 ) \leq 1 - \gamma_l a / 2
\]
we obtain 
\[
\sum_{j=0}^{k-1} \beta_j \gamma_j \prod_{l=j+1}^{k-1} (1 - \gamma_l a) \leq \varsigma \frac{2}{a} \beta_k 
\]
Similarly, using $\beta_j \leq \kappa \gamma_j \leq a_{22}/(4 a_\Delta) \gamma_j$, we obtain
\[
\sum_{j=0}^{k-1} \beta_j^2 \prod_{l=j+1}^{k-1} (1 - \gamma_l a) \leq \varsigma \frac{a_{22}}{2 a \, a_\Delta} \beta_k 
\]
For part ii) of the corollary, we observe that the first inequality can be proven by:
\[
\begin{split}
& \sum_{j=0}^{k-1} \beta_j \gamma_j \prod_{\ell=j+1}^{k-1} (1 - a \beta_\ell ) = \gamma_{k} \sum_{j=0}^{k-1} \beta_j \frac{\gamma_{k-1}}{\gamma_k} \frac{\gamma_j}{\gamma_{k-1}} \prod_{\ell=j+1}^{k-1} (1 - a \beta_\ell ) \\
& \leq \varsigma \gamma_{k} \sum_{j=0}^{k-1} \beta_j \frac{\gamma_j}{\gamma_{k-1}} \prod_{\ell=j+1}^{k-1} (1 - a \beta_\ell ) \leq \varsigma \gamma_{k} \sum_{j=0}^{k-1} \beta_j  \prod_{\ell=j+1}^{k-1} \frac{\gamma_{\ell-1}}{\gamma_\ell} (1 - a \beta_\ell )
\end{split}
\]
Note that as $a \geq a_\Delta/8$, we have
\[
\frac{\gamma_{\ell-1}}{\gamma_\ell} (1 - a \beta_\ell ) \leq (1 + \epsilon_\beta \beta_\ell) ( 1 - a \beta_\ell ) \leq 1 - a \beta_\ell / 2
\]
Using Lemma~\ref{lem:bsum}, this yields
\[
\begin{split}
& \sum_{j=0}^{k-1} \beta_j \gamma_j \prod_{\ell=j+1}^{k-1} (1 - a \beta_\ell )
\leq \varsigma \gamma_{k} \sum_{j=0}^{k-1} \beta_j  \prod_{\ell=j+1}^{k-1} (1 - a \beta_\ell / 2) \leq \frac{2}{a} \varsigma \, \gamma_k 
\end{split}
\]
Similarly, the second inequality is:
\[
\begin{split}
& \sum_{j=0}^{k-1} \beta_j^2 \prod_{\ell=j+1}^{k-1} (1 - a \beta_\ell )
\leq \varsigma \beta_{k} \sum_{j=0}^{k-1} \beta_j \frac{\beta_j}{\beta_{k-1}} \prod_{\ell=j+1}^{k-1} (1 - a \beta_\ell ) \\
& \leq \varsigma \beta_{k} \sum_{j=0}^{k-1} \beta_j  \prod_{\ell=j+1}^{k-1} \frac{\beta_{\ell-1}}{\beta_\ell} (1 - a \beta_\ell ) \overset{(a)}{\leq} 
\varsigma \beta_{k} \sum_{j=0}^{k-1} \beta_j  \prod_{\ell=j+1}^{k-1} (1 - a \beta_\ell/2 ) \overset{(b)}{\leq} \frac{2}{a} \varsigma \beta_k
\end{split}
\]
where (a) is due to the fact that we have 
$\frac{\beta_{\ell-1}}{\beta_\ell} (1 - a \beta_\ell ) \leq 1 - a \beta_\ell / 2$, and (b) is obtained by applying Lemma~\ref{lem:bsum}.

For the proof of part iii) we proceed similarly. We prove the second inequality only. The proof of the remaining follows the same lines. Denote $\kappa_\ell: = \beta_\ell/\gamma_\ell$. Clearly, $\kappa_{\ell-1}/\kappa_{\ell} \le \beta_{\ell-1}/\beta_\ell \le \varsigma$. 
Then
\[
\begin{split}
& \sum_{j=0}^{k-1} \beta_j^2 \kappa_j^2 \prod_{\ell=j+1}^{k-1} (1 - a \beta_\ell )
\leq \varsigma^3 \beta_{k} \kappa_k^2 \sum_{j=0}^{k-1} \beta_j \frac{\beta_j }{\beta_{k-1}} \bigg(\frac{\kappa_j}{\kappa_{k-1}}\bigg)^2 \prod_{\ell=j+1}^{k-1} (1 - a \beta_\ell ) \\
& \leq \varsigma^3 \beta_{k} \kappa_k^2\sum_{j=0}^{k-1} \beta_j  \prod_{\ell=j+1}^{k-1} \bigg(\frac{\beta_{\ell-1}}{\beta_\ell}\bigg)^3 (1 - a \beta_\ell ) \overset{(a)}{\leq} 
\varsigma^3 \beta_{k} \kappa_k^2 \sum_{j=0}^{k-1} \beta_j  \prod_{\ell=j+1}^{k-1} (1 - a \beta_\ell/2 ) \overset{(b)}{\leq} \frac{4}{a} \varsigma^3 \beta_{k} \kappa_k^2
\end{split}
\]
where (a) is due to the fact that we have 
$(\frac{\beta_{\ell-1}}{\beta_\ell})^3 (1 - a \beta_\ell ) \leq 1 - a \beta_\ell / 4$, and (b) is obtained by applying Lemma~\ref{lem:bsum}.
Part iv) may be proved in the similar way. 
\end{proof}

\begin{lemma}
\label{lem:summation-lemma}
Let $a > 0$, $p \geq 0$, $(\gamma_j)_{j \geq 0}$, $(\kappa_j)_{j \geq 0}$ and $(u_j)_{j \geq 0}$ be nonnegative sequences. Then, for all integers $k$, 
\[
\sum_{j=0}^{k-1} \kappa_j \prod_{\ell=j}^{k-1} (1 -a \gamma_{\ell}) \sum_{i=0}^{j-1} \gamma_i^p \prod_{n=i+1}^{j-1} (1 - a \gamma_n) u_i 
= \sum_{i=0}^{k-1} \gamma_i^p u_i \left( \sum_{j=i+1}^{k-1} \kappa_j \right) \prod_{\ell=i+1}^{k-1} (1 - a \gamma_\ell)
\]
\end{lemma}

\begin{lemma}[Lyapunov Lemma]
\label{lem:lyapunov}
A matrix $A$ is Hurwitz if and only if for any positive symmetric matrix $P=P^\top \succ 0$ there is  $Q =Q^\top \succ 0$ that satisfies the Lyapunov equation
\[
A^\top Q + Q A = -P \,.
\]
In addition, $Q$ is unique.
\end{lemma}
\begin{proof}
See \cite[Lemma 9.1, p. 140]{poznyak:control}.
\end{proof}

\begin{lemma}
\label{lem:stability}
Assume that $-A$ is a Hurwitz matrix. Let $Q$ be the unique solution of the Lyapunov equation
\[
A^\top Q + Q A =  \Id \,.
\]
Then, for any $\zeta \in [0, \zeta_A]$, where 
\beq
\label{def:kappa_A}
\zeta_A:=(1/2) \normop{A}[Q]^{-2} \normop{Q}^{-2},
\eeq
we get
\[
\normop{\Id - \zeta A}[Q]^2 \leq (1 - a \zeta) \quad \text{with} \quad a= (1/2) \normop{Q}^{-2} \eqsp.
\]
If in addition $\zeta \le \normop{Q}^2$ then
\[
1 - a \zeta \geq 1/2.
\]
\end{lemma}
\begin{proof}
For any  $x \in \rset^d$, we get
\[
\frac{x^\top (\Id - \gamma A)^\top Q (\Id - \gamma A) x}{x^\top Q x}
=1 - \gamma \frac{\norm{x}^2}{x^\top Q x} + \gamma^2 \frac{x^\top A^\top Q A x}{x^\top Q x}
\]
Hence, we get that for all $\gamma \in [0, (1/2) \normop{A}[Q]^{-2} \normop{Q}^{-2}]$,
\begin{align}
1 - \gamma \frac{\norm{x}^2}{x^\top Q x} + \gamma^2 
\frac{x^\top A^\top Q A x}{x^\top Qx}  
&\leq 1 - \gamma  \normop{Q}^{-2} + \gamma^2 \normop{A}[Q]^2 \nonumber \\
&\leq 1 - (1/2) \normop{Q}^{-2} \gamma \,. \nonumber
\end{align}
The proof follows.
\end{proof}

\begin{lemma} 
\label{lem:LkBound} 
Assume that $\normop{L}[Q_{\Delta},Q_{22}] \leq \varepsilon$ for some $\varepsilon > 0$  and 
\begin{align}
\label{eq:condition-beta}
& 0   \leq \beta \leq (1/2) \{ \normop{\Delta}[Q_{\Delta}] + \varepsilon \normop{A_{12}}[Q_{22},Q_{\Delta}] \}^{-1} \\
\label{eq:condition-gamma}
& 0 \leq \gamma \leq (1/2) \normop{Q_{22}}[\Id]^{-2} \normop{A}[Q_{22}]^{-2} \eqsp.
\end{align}
Set $B_{11}(L)= \Delta - A_{12} L$. Then, the equation 
\begin{equation}
\label{eq:equation-def-L}
   L' \{\Id - \beta B_{11}(L) \} = (\Id - \gamma A_{22}) L + \beta A_{22}^{-1}A_{21} B_{11}(L)  
\end{equation}
has a unique solution satisfying
\[
\normop{L'}[Q_{\Delta},Q_{22}] \leq (1- \gamma a_{22}) \normop{L}[Q_{\Delta},Q_{22}] + \beta \CD(\varepsilon) 
\]
where 
\begin{equation}
\label{eq:definition-CD-varepsilon}
\CD(\varepsilon)= 2 \{ \normop{A_{22}^{-1} A_{21}}[Q_\Delta,Q_{22}] + \varepsilon\}\{  \normop{\Delta}[Q_{\Delta}] + \varepsilon \normop{A_{12}}[Q_{22},Q_{\Delta}] \} \eqsp.
\end{equation}
If $\beta/\gamma \leq \varepsilon a_{22}/\CD(\varepsilon)$, then $\normop{L'}[Q_{\Delta},Q_{22}] \leq \varepsilon$.
\end{lemma}
\begin{proof}
Since $\normop{L}[Q_{\Delta},Q_{22}] \leq \varepsilon$, we get that $\normop{B_{11}(L)}[Q_{\Delta}] \leq \normop{\Delta}[Q_{\Delta}] + \varepsilon \normop{A_{12}}[Q_{22},Q_{\Delta}] $. Hence, using \eqref{eq:condition-beta} and the triangular inequality, we get that $\beta \normop{B_{11}(L)}[Q_{\Delta}] \leq 1/2$ and thus
\begin{equation}
\label{eq:bound-I-beta-B}
    \normop{\Id - \beta B_{11}(L)}[Q_{\Delta}] \geq 1/2 \eqsp.
\end{equation}
Hence, $\Id- \beta B_{11}(L)$ is invertible and \eqref{eq:equation-def-L} has a unique solution given by 
\begin{align}
    L' 
    &= \left\{ (\Id - \gamma A_{22}) L + \beta A_{22}^{-1}A_{21} B_{11}(L) \right\} \left\{  \Id - \beta B_{11}(L) \right\}^{-1} \nonumber \\
    &= (\Id - \gamma A_{22}) L + \beta D(L) \nonumber
\end{align}
where 
\[
D(L)= \{ A_{22}^{-1} A_{21} + (\Id - \gamma A_{22}) L \}B_{11}(L) \{\Id - \beta B_{11}(L) \}^{-1} \eqsp.
\]
Using \eqref{eq:bound-I-beta-B} and $\normop{L}[Q_{\Delta},Q_{22}] \leq \varepsilon$, we get that 
$\normop{D(L)}[Q_{\Delta},Q_{22}] \leq \CD(\varepsilon)$. Hence, for $\gamma$ satisfying \eqref{eq:condition-gamma}, we get that 
\[
\normop{L'}[Q_{\Delta},Q_ {22}] \leq (1- \gamma a_{22}) \normop{L}[Q_{\Delta},Q_{22}] + \beta \CD(\varepsilon) \leq \varepsilon + \gamma \big( \frac{\beta}{\gamma} \CD(\varepsilon) - a_{22} \varepsilon \big) \leq \varepsilon \eqsp,
\]
where the last inequality is due to $\frac{\beta}{\gamma} \leq \varepsilon a_{22} / \CD(\varepsilon) $. 
\end{proof}

\begin{lemma}\label{lem: L k estimate sharp}
Let $L_0 = 0$. Assume that $\normop{L_k}[Q_{\Delta},Q_{22}] \leq L_\infty$ and  
\begin{align}
& 0   \leq \beta_0 \leq (1/2) \{ \normop{\Delta}[Q_{\Delta}] + L_\infty \normop{A_{12}}[Q_{22},Q_{\Delta}] \}^{-1} \nonumber \\
& 0 \leq \gamma_0 \leq (1/2) \normop{Q_{22}}[\Id]^{-2} \normop{A}[Q_{22}]^{-2} \nonumber
\end{align}
Then for any $k \in \nset$
$$
\normop{L_k}[Q_{\Delta},Q_{22}] \le \CD(L_\infty) \CSEQ{a_{22}}\beta_k/\gamma_k,
$$
where 
$$
\CD(L_\infty): = 2 \{ \normop{A_{22}^{-1} A_{21}}[Q_\Delta,Q_{22}] + L_\infty\}\{  \normop{\Delta}[Q_{\Delta}] + L_\infty \normop{A_{12}}[Q_{22},Q_{\Delta}] \}
$$
\end{lemma}
\begin{proof}
Similarly to Lemma~\ref{lem:LkBound}  we may show that 
$$
L_{k+1} = (\Id - \gamma A_{22}) L_k + \beta_k D(L_k)  
$$
where $\normop{D(L_k)}[Q_{\Delta},Q_{22}] \leq \CD(L_\infty)$. Hence,
$$
\normop{L_k}[Q_{\Delta},Q_{22}] \le \CD(L_\infty) \sum_{j=0}^k \beta_j \prod_{s = j+1}^k (1 - a_{22} \gamma_s) 
$$
Application of Lemma~\ref{cor:rate-of-convergence} to the right hand side of the above completes the proof.
\end{proof}
\begin{lemma} \label{lem:Lkgamkbound}
Let $L_1 := L_0 := 0$. Assume that $\normop{L_k}[Q_{\Delta},Q_{22}] \leq L_\infty$   and 
\begin{align}
& 0   \leq \beta_0 \leq (1/2) \{ \normop{\Delta}[Q_{\Delta}] + L_\infty \normop{A_{12}}[Q_{22},Q_{\Delta}] \}^{-1} \nonumber \\
& 0 \leq \gamma_0 \leq (1/2) \normop{Q_{22}}[\Id]^{-2} \normop{A}[Q_{22}]^{-2} \nonumber \\
&\beta_{k-1} - \beta_k \le \rho_\beta \beta_k^2, \gamma_{k-1} - \gamma_k \le \rho_\gamma \gamma_k^2 \nonumber \\ 
& \beta_k/\gamma_k \le (1/(2C_1^U)) a_{22} \nonumber
\eqsp
\end{align}
with 
$$
C_1^U := 2 (\normop{\Delta}[Q_{\Delta}] + \normop{A_{22}^{-1} A_{21}}[Q_\Delta, Q_{22}]\normop{A_{12}}[Q_{22}, Q_\Delta] + 2 L_\infty \normop{A_{12}}[Q_{22},Q_{\Delta}]).
$$
Then
$$
\normop{L_{k+1} - L_k}[Q_\Delta, Q_{22}] \le C_2^U \CSEQ{a_{22}/2}\gamma_{k+1}, 
$$
where 
$$
C_2^U:=2\rho_\gamma L_\infty\normop{A_{22}}[Q_{22}] + 2 \rho_{\beta} (L_\infty + \normop{A_{22}^{-1} A_{21}}[Q_\Delta, Q_{22}])   (\normop{\Delta}[Q_{\Delta}] +  L_\infty \normop{A_{12}}[Q_{22},Q_{\Delta}])
$$
\end{lemma}
\begin{proof} Recall that $B_{11}(L)= \Delta - A_{12} L$. It follows from Lemma~\ref{lem:LkBound} that $\Id - \beta_k B_{11}(L_k)$ is invertible matrix with bounded norm. Equation
$$
L_{k}(\Id - \beta_{k-1} B_{11}(L_{k-1})) = \bigg \{ (\Id - \gamma_{k-1} A_{22}) L_{k-1} + \beta_{k-1} A_{22}^{-1} A_{21} B_{11}(L_{k-1})\bigg\}
$$
may be rewritten as follows
\begin{align}
L_{k}(\Id - \beta_{k} B_{11}(L_{k})) = (\Id - \gamma_k A_{22}) L_{k-1} + \beta_{k} B_{11}(L_{k}) + E_k, 
\end{align}
where $E_k := (\gamma_k - \gamma_{k-1}) A_{22} L_{k-1} +  (L_k + A_{22}^{-1} A_{21}) D_k, D_k: = - \beta_k A_{12}(L_k - L_{k-1}) + (\beta_k - \beta_{k-1}) B_{11}(L_{k-1})$. Let $U_k = L_{k} - L_{k-1}$. Then
$$
U_{k+1} (\Id - \beta_{k} B_{11}(L_{k})) = (\Id - \gamma_k A_{22}) U_k - E_k. 
$$
Then
$$
U_{k+1} = (\Id - \gamma_k A_{22}) U_k + \beta_k (\Id - \gamma_k A_{22}) U_k B_{11}(L_k)(\Id - \beta_{k} B_{11}(L_{k}))^{-1} - E_k (\Id - \beta_{k} B_{11}(L_{k}))^{-1}  
$$
It is easy to check that
$$
\normop{(\Id - \gamma_k A_{22}) U_k B_{11}(L_k)(\Id - \beta_{k} B_{11}(L_{k}))^{-1}}[Q_\Delta, Q_{22}] \le 2 \normop{U_k}[Q_\Delta, Q_{22}] \{\normop{\Delta}[Q_{\Delta}] + L_\infty \normop{A_{12}}[Q_{22},Q_{\Delta}]\}
$$
Moreover,
\begin{align}
&\normop{E_k (\Id - \beta_{k} B_{11}(L_{k}))^{-1}}[Q_\Delta, Q_{22}] \le 2\rho_\gamma \gamma_k^2 L_\infty\normop{A_{22}}[Q_{22}] \nonumber \\
& + 2(L_\infty + \normop{A_{22}^{-1} A_{21}}[Q_\Delta, Q_{22}]) \{ \rho_{\beta} \beta_k^2 (\normop{\Delta}[Q_{\Delta}] +  L_\infty \normop{A_{12}}[Q_{22},Q_{\Delta}]) +\beta_k \normop{A_{12}}[Q_{22}, Q_\Delta] \normop{U_k}[Q_\Delta, Q_{22}]  \} \nonumber
\end{align}
Applying previous inequalities we obtain
\begin{align}
\normop{U_{k+1}}[Q_\Delta, Q_{22}] \le (1 - \gamma_k a_{22} + C_1^U \beta_k) \normop{U_{k}}[Q_\Delta, Q_{22}] + C_2^U \gamma_k^2 \nonumber
\end{align}
Since $\beta_k/\gamma_k \le (1/(2C_1^U)) a_{22} $ we obtain 
$$
\normop{U_{k+1}}[Q_\Delta, Q_{22}] \le C_2^U \CSEQ{a_{22}/2}\gamma_{k+1} 
$$
\end{proof}

\begin{lemma}
\label{lem:matrix-trace-normop}
Let $Q$ be a symmetric definite positive $n \times n$ matrix and $\Sigma$ be a $n  \times n$ matrix. Then
\[
\Tr(Q \Sigma) \leq \normop{\Sigma}[Q] \Tr(Q) \eqsp.
\]
\end{lemma}
\begin{proof}
Denote by $(e_i)_{i=1}^n$ an orthonomal basis of eigenvectors of $Q$, $Q e_i= \lambda_i(Q) e_i$, $i=1,\dots,n$, $\ps{e_i}{e_j}= \delta_{i,j}$, where $\delta_{i,j}$ is the Kronecker symbol. We get that
\begin{align}
\Tr(Q \Sigma)
&= \sum_{i=1}^n \ps{e_i}{Q \Sigma e_i}= \sum_{i=1}^n \ps[Q]{e_i}{\Sigma e_i} \nonumber \\
&\leq \normop{\Sigma}[Q] \sum_{i=1}^n \norm{e_i}[Q]^2 = \normop{\Sigma}[Q] \Tr{Q} \nonumber
\end{align}
where we have used $\norm{e_i}[Q]=\lambda_i$ and $\Tr{Q}= \sum_{i=1}^n \lambda_i(Q)$ \eqsp.
\end{proof}

\begin{corollary}
\label{coro:matrix-trace-normop}
If $X$ is a $n \times 1$ random vector such that $\PE[\norm{X}[2]] < \infty$. Then, 
\[
\E[\norm{X}[Q]^2] \leq \Tr(Q) \normop{\E[XX^\top]}[Q] \eqsp.
\]
\end{corollary}
\begin{proof}
Note that $\E[\norm{X}[Q]^2] = \Tr(Q\E[X X^\top]) \leq  \normop{\E[XX^\top]}[Q] \Tr{Q}$
\end{proof}

\begin{lemma}
\label{lem:key-inequality}
Let $m$ and $n$ be two integers, $P$ and $Q$ be $m \times m$ and $n \times n$ symmetric positive definite matrices. Let $X$ and $Y$ be  $m \times 1$ and $n \times 1$ random vectors such that $\PE[\norm{X}^2] < \infty$ and $\PE[\norm{Y}^2] < \infty$. Then,
\[
\normop{\PE[X Y^\top]}[Q,P] \leq \lambda_{\min}(Q)^{-1} \{ \Tr(Q) \}^{1/2} \{\Tr(P) \}^{1/2} \normop{\PE[X X^\top]}[P]^{1/2} \normop{\PE[Y Y^\top]}[Q]^{1/2}
\]
\end{lemma}
\begin{proof}
Note that $\normop{\PE[X Y^\top]}[Q,P] \leq \PE[ \normop{X Y^\top}[Q,P]]$ and 
\begin{align}
\normop{X Y^\top}[Q,P] 
&= \sup_{\norm{y}[Q]=1} \norm{X \ps[Q]{Y}{y}}= \norm{X}[P] \sup_{\norm{y}[Q]=1} \ps[Q]{Q^{-1}Y}{y} \nonumber \\
&= \norm{X}[P] \norm{Q^{-1} Y}[Q]= \norm{X}[P] \norm{Y}[Q^{-1}] \leq \lambda_{\min}^{-1}(Q) \norm{X}[P] \norm{Y}[Q] \eqsp. \nonumber
\end{align}
By applying the Cauchy-Schwarz inequality, we obtain
\begin{equation}
\normop{\PE[X Y^\top]}[Q,P] \leq \lambda_{\min}^{-1}(Q) \left\{ \PE[\norm{X}[P]^2]  \right\}^{1/2} 
\left\{ \PE[\norm{Y}[Q]^2]  \right\}^{1/2} \eqsp. \nonumber
\end{equation}
The proof follows from Corollary~\ref{coro:matrix-trace-normop}.
\end{proof}


\section{Details on Numerical Experiments} \label{app:num}
This section provides details about the numerical experiments and verification that the convergence conditions are satisfied. 

\subsection{Toy Example}
In this toy example, we consider randomly generated instances of linear two timescale SA in the form \eqref{eq:tts-gen1}, \eqref{eq:tts-gen2} with i.i.d.~samples (and thus the martingale noise setting). In particular, we let the iterates $\theta_k, w_k \in \rset^{d}$ be $d$-dimensional and construct a problem instance as follows:
\begin{enumerate}
    \item Sample a random matrix $T$ whose entries are drawn i.i.d.~from the uniform distribution $U[-1,1]$; Compute the $QR$-decomposition as $T=QR$. 
    \item Set $A_{12}=Q$ and $A_{22}=Q^\top \Lambda_0 Q$, where $\Lambda_0$ is a diagonal matrix with i.i.d. entries from $U[-1,1]$.
    \item Sample a random matrix $R$ whose entries are drawn i.i.d.~from the uniform distribution $U[-1,1]$.
    \item Set $A_{11}=R R^\top + I$ and $A_{21}=Q^\top \Lambda_1$, where $\Lambda_1$ is a diagonal matrix with i.i.d. entries from $U[-1,1]$.
    \item Sample a stationary solution pair $\theta^*,w^*$ with i.i.d.~entries from $U[-1,1]$.
    \item Compute $b_1,b_2$ using the generated matrices and stationary points, i.e.,
    \[
    b_1 = A_{11} \theta^\star + A_{12} w^\star, \quad b_2 = A_{21} \theta^\star + A_{22} w^\star.
    \]
\end{enumerate}
During the linear two timescale SA iteration, the noise terms are generated as
\begin{align}
&V_{k+1} = F_V^k + A_{V,\theta}^{k} \theta_k + A_{V,w}^{k} w_k, \quad W_{k+1} = F_W^k + A_{W,\theta}^{k} \theta_k + A_{W,w}^{k} w_k \nonumber
\end{align}
where $F_V^k,A_{V,\theta}^{k}, A_{V,w}^{k}$ are vectors/matrices with entries drawn i.i.d.~from the standard normal distribution ${\cal N}(0,0.1)$, and $F_W^k, A_{W,\theta}^{k},A_{W,w}^{k}$ are vectors/matrices with entries drawn i.i.d.~from the standard normal distribution ${\cal N}(0,0.5)$. With the above constructions, it can be verified that the required assumptions A\ref{assum:hurwitz}, A\ref{assum:zero-mean}, A\ref{assum:bound-conditional-variance} of the martingale noise setting hold. It remains to verify that the step sizes chosen satisfy A\ref{assum:stepsize}. 

Below, we show the plots of deviations in $\theta_k$ and $w_k$ without normalization by the step sizes (see Fig. \ref{fig:ratesMartingaleMarkovNotNormalized}).

\begin{figure}[h!]
\centering
    \begin{tabular}{c c r c}
    \addvbuffer[0cm 5.25cm]{\sf (a)}\hspace{-.45cm} & \includegraphics[width=.375\linewidth]{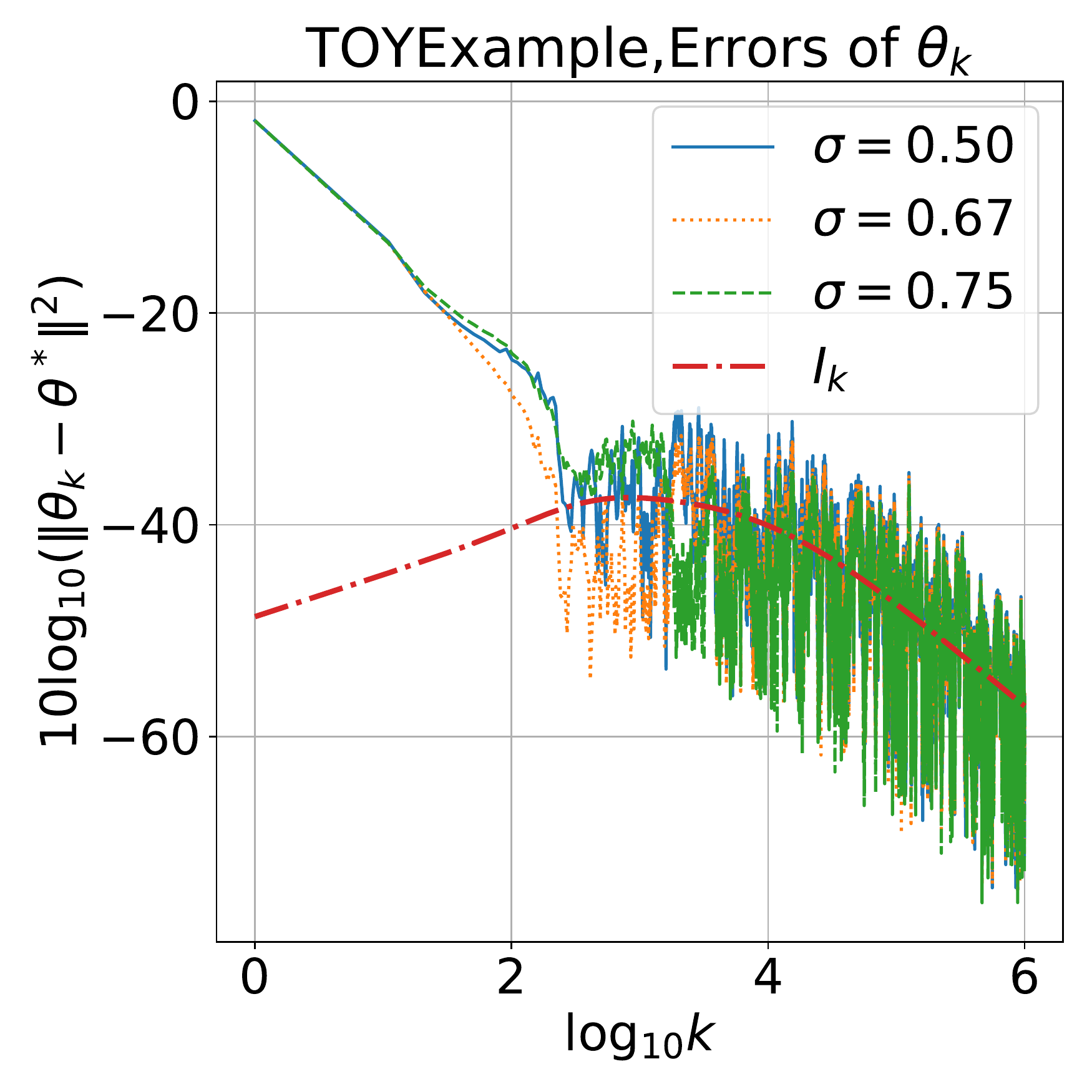} & \addvbuffer[0cm 5.25cm]{\sf (b)}\hspace{-.45cm} & \includegraphics[width=.375\linewidth]{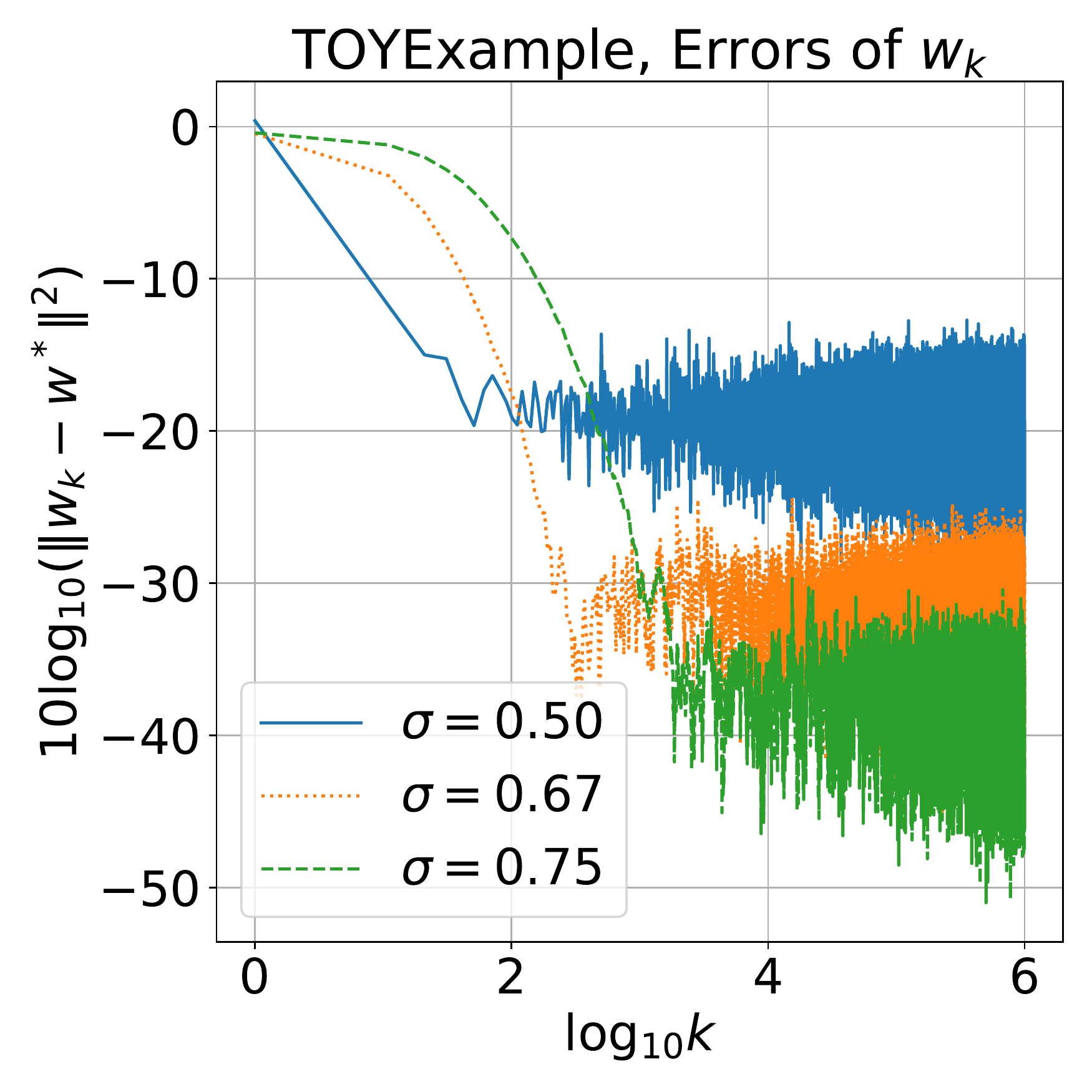} 
    \end{tabular}
    \caption{Unnormalized deviations from stationary point $(\theta^\star, w^\star)$ and term $I_k$ : the toy example.\vspace{-.5cm}}
    \label{fig:ratesMartingaleMarkovNotNormalized}
\end{figure}

\subsection{Garnet Problems}
\paragraph{GTD Algorithm and Policy Evaluation Problem}
The specific form of linear two timescale SA used in this example follows from that of the classical GTD algorithm \citep{sutton:gtd:2009,sutton:gtd2:2009}, which is described below for completeness. 
Let $\mathcal{S},\mathcal{A}$ be some discrete state and action spaces (for clarity we bound ourselves by discrete setting, but one could formulate it in more general way), $\gamma\in(0,1)$ and $\pi:\mathcal{S} \to \mathcal{P}(\mathcal{A})$ be a \textit{stochastic policy}, i.e. mapping from states to probability measures over actions. When in state $s$ the agent performs action $a$ (distributed according to its policy $\pi$), it transitions randomly to state $s'$ with probability $p(s' \vert s,a)$ and obtains reward $r(s,a)$. This induces a Markov chain with transition probabilities $p_\pi ( s' \vert s) := \sum_{a \in \mathcal{A}} \pi(a \vert s) p(s'\vert s,a)$.


The goal of policy evaluation is to estimate the average discounted cumulative reward obtained with the policy $\pi$. In detail, we evaluate the value function $V_\pi(s):=\mathds{E}\left[r(s,a)+\sum_{k=1}\rho^k r(s_k,a_k) \right]$ with $\rho$ being the \textit{discounting factor}. As the state space $|{\cal S}|$ is often large, we use the linear approximation $V_\pi(s) \approx V_\theta(s) := \langle \theta,  \phi(s) \rangle$, where  $\phi: \mathcal{S} \rightarrow \rset^{d}$ is a pre-defined feature map. Define also \textit{temporal difference} at iteration  $k\in\mathds{Z}_+$ for transition $s_k \to s_{k+1}$ as $\delta_k := r(s_k,a_k) + \gamma V_{\theta_k}(s_{k+1}) - V_{\theta_k}(s_k)$. For brevity, denote the observation at iteration $k\in \mathds{Z}_+$, namely, $\phi(s_k)$, $\phi(s_{k+1})$, $r(s_k,a_k)$ as $\phi_k$, $\phi_{k+1}$, $r_k$ respectively. The GTD algorithm iterations are described as:
\begin{align}
    &\theta_{k+1} = \theta_k + \beta_k \left[\phi_k - \rho \phi_{k+1} \right]\langle \phi_k,  w_k \rangle, \quad w_{k+1} = w_k + \gamma_k \left[\phi_k \delta_k - w_k \right].
    \label{eqs:gtd}
\end{align}
The above is a special case of our linear two timescale SA in \eqref{eq:2ts1}, \eqref{eq:2ts2} with the notations:
\begin{align}
    & b_1 =0,~ A_{11}=0,~ A_{12}= - \E[(\phi_k-\rho \phi_{k+1})\phi_k^\top],\\
    &b_2=\E[\phi_k r_k],~ A_{21} = - \E[\phi_k(\rho \phi_{k+1}-\phi_k)^\top],~ A_{22}= \Id_d,\\ 
    & V_{k+1}=\left((\phi_k- \rho \phi_{k+1})\phi_k^\top - \E[(\phi_k- \rho \phi_k')\phi_k^\top] \right) w_k,\\
    & W_{k+1}=\phi_k r_k - \E[\phi_k r_k] + \left( (\phi_k- \rho \phi_{k+1})\phi_k^\top - \E[(\phi_k- \rho \phi_{k+1})\phi_k^\top] \right) \theta_k,
    \label{eqs:gtdnoises}
\end{align}
where the expectations above are taken with respect to the stationary distribution of the MDP under policy $\pi$.
Particularly, the noise terms $V_{k+1}, W_{k+1}$ follow the Markovian noise setting. 

\paragraph{Garnet Problem} The Garnet problem refers to a set of policy evaluation problems with randomly generated problem instances, originally proposed in \citep{arch:garnetInitial:1995}. Here, we consider a simpler version of Garnet problems described in \citep{geist:offpolicy:2014}.
Particularly, we consider a finite-state MDP with the parameters $n_S$ as the number of states, $n_A$ as the number of possible actions in each state, $b$ as the branching factor, i.e., the number of transitions from each state-action pair to a new state, $p$ as the number of features in the linear function approximation applied. For each $(s,a)\in \mathcal{S} \times \mathcal{A}$ the next transitions $s' \in \mathcal{S}' \subset \mathcal{S}$ is chosen uniformly from the set of all combinations from $\mathcal{S}$ consisting of $b$ items. For all $s' \in \mathcal{S}'$ the transition probabilities $p(s'\vert s,a)$ are generated from $U[0,1]$ and then normalized by their sum. For the features, for each state $s \in \mathcal{S}$ the corresponding feature vector $\phi(s)$ is generated from $(U[0,1])^p$. In our numerical example, we consider a particular problem from the family $n_S=30,n_A=8,b=2,p=8$.

By the above constructions, we observe that the assumptions A\ref{assum:hurwitz}, B\ref{assb:mc1}--B\ref{assb:bdd} are all satisfied. It remains to verify that the step sizes chosen satisfy A\ref{assum:stepsize}, B\ref{assb:step}.
  
\subsection{Step Size Parameters}
We consider the family of step size schedules:
 \beq \label{eq:step_app}
 \beta_k = {c^\beta} / (k_0^\beta + k), ~~ \gamma_k = {c^\gamma} / { (k_0^\gamma + k)^{\sigma} },
 \eeq
with $\sigma \in [0.5,1]$ and the parameters $c^\beta, c^\gamma, k^\beta_0, k_0^\gamma$. Note that
\[
\frac{\beta_k}{\gamma_k} = \frac{c^\beta (k_0^\gamma + k)^\sigma } {c^\gamma (k_0^\beta + k) } \leq \frac{c^\beta}{c^\gamma} \left( \frac{ k_0^\gamma } {k_0^\beta } \right)^\sigma =: \kappa
\]
since we have $\sigma \leq 1$. This ensures A\ref{assum:stepsize}-1.
Furthermore, we observe that
\[
\frac{\gamma_{k-1}}{\gamma_k} = \left( 1 + \frac{1}{k_0^\gamma+k-1} \right)^\sigma \leq 1 + \frac{\sigma}{k_0^\gamma+k-1} \leq 1 + \frac{\sigma k_0^\gamma }{c^\gamma (k_0^\gamma-1) } \frac{c^\gamma}{(k_0^\gamma+k)^\sigma} = 1 + \frac{\sigma k_0^\gamma }{c^\gamma (k_0^\gamma-1) } \gamma_k,
\]
On the other hand, we also have 
\[
\frac{\gamma_{k-1}}{\gamma_k} \leq 1 + \frac{\sigma k_0^\beta }{c^\beta (k_0^\gamma-1) } \frac{c^\beta}{k_0^\beta+k} = 1 + \frac{\sigma k_0^\beta }{c^\beta (k_0^\gamma-1) } \beta_k
\]
Similar upper bound can be derived for $\beta_{k-1}/\beta_k$. Setting $c^\gamma, c^\beta$ large enough ensures A\ref{assum:stepsize}-2. Lastly, B\ref{assb:step} can be guaranteed by observing that $\sigma \geq 0.5$. 

The above discussions illustrate that the satisfaction of A\ref{assum:stepsize} hinge on setting a large $c^\gamma, c^\beta$. However, this requirement can be hard to satisfy since we also have requirements such as $\gamma_k \leq \gamma_\infty^{\sf mark}$, $\beta_k \leq \beta_\infty^{\sf mark}$. To this end, we have to set a large $k_0^\beta, k_0^\gamma$. As a result, there are four inter-related hyper parameters to be tuned in order to ensure the desired convergence of linear two timescale SA. 
We remark that tuning the step size parameters for SA scheme is generally difficult.


\end{document}